\pdfoutput =1
\newif\ifdraft
\drafttrue  
\draftfalse 

\documentclass[letterpaper]{article}

\usepackage{chngcntr} 
\usepackage[numbers]{natbib} 
\renewcommand{\citep}[1]{\cite{#1}}
\renewcommand{\citet}[1]{\cite{#1}}

\usepackage{subfigure} 
\usepackage[english]{babel}
\usepackage[utf8x]{inputenc}
\usepackage[T1]{fontenc} 

\usepackage[top=3cm,bottom=2cm,left=3cm,right=3cm,marginparwidth=1.75cm]{geometry}

\usepackage[pagebackref=false,breaklinks=true,letterpaper=true,colorlinks=true,bookmarks=false, allcolors=blue, hyperindex,pdftex]{hyperref}

\usepackage[colorinlistoftodos]{todonotes}

\usepackage{url}            
\usepackage{booktabs}       
\usepackage{amsfonts}       
\usepackage{nicefrac}       
\usepackage{microtype}      
\usepackage{lipsum}
\usepackage{xcolor}
\usepackage{enumitem}
\usepackage[english]{babel}
\usepackage{epsfig}
\usepackage{graphicx}
\usepackage{amsmath,amsthm}
\usepackage{amssymb}
\usepackage{algorithm}
\usepackage{algorithmicx}
\usepackage{algpseudocode}
\usepackage{enumitem}
\usepackage{xspace}
\usepackage{thmtools,thm-restate} 
\usepackage{comment}
\usepackage{appendix}
\usepackage{bm}
\usepackage{bbm}
\usepackage{dsfont}
\allowdisplaybreaks

\newtheorem{theorem}{Theorem}
\newtheorem*{theorem*}{Theorem}
\newtheorem{lemma}{Lemma}

\newtheorem{proposition}{Proposition}
\newtheorem{definition}{Definition}

\newtheorem{remark}{Remark}
\newtheorem{example}{Example}

\title{Structures of Spurious Local Minima in $k$-means}

\author{Wei Qian$^\dagger$, Yuqian Zhang$^\sharp$, Yudong Chen$^\dagger$ \footnote{Authors' emails: \texttt{wq34@cornell.edu, yqz.zhang@rutgers.edu, yudong.chen@cornell.edu}}\\ ~\\
\large $^\dagger$School of Operations Research and Information Engineering, Cornell University\\
\large $^\sharp$Department of Electrical and Computer Engineering, Rutgers University
}

\date{}

\begin{document}
\newif\ifdraft
\drafttrue   
\newcommand{\blue}[1]{\textcolor{blue}{#1}}
\newcommand{\red}[1]{\textcolor{red}{#1}}
\newcommand{\green}[1]{\textcolor{green}{#1}}
\newcommand{\magenta}[1]{\textcolor{magenta}{#1}}
\newcommand{\orange}[1]{\textcolor{orange}{#1}}
\newcommand{\wq}[1]{\ifdraft {\bf{{\magenta{{Wei:  #1}}}}}\else\fi}
\newcommand{\yz}[1]{\ifdraft {\bf{{\blue{{Yuqian:  #1}}}}}\else\fi}
\newcommand{\yc}[1]{\ifdraft {\bf{{\red{{Yudong:  #1}}}}}\else\fi}

\newcommand{\ddup}{\textup{d}}%
\newcommand\real{\mathbb{R}}%
\newcommand{\mP}{\mathbb{P}}%
\newcommand{\E}{\mathbb{E}}%
\newcommand{\mB}{\mathbb{B}}%
\newcommand{\mS}{\mathbb{S}}%
\newcommand{\ba}{\bm{a}}%
\newcommand{\bb}{\bm{b}}%
\newcommand{\bc}{\bm{c}}%
\newcommand{\bs}{\bm{s}}%
\newcommand{\bu}{\bm{u}}%
\newcommand{\bv}{\bm{v}}%
\newcommand{\w}{\bm{w}}%
\newcommand{\x}{\bm{x}}%
\newcommand{\y}{\bm{y}}%
\newcommand{\z}{\bm{z}}%
\newcommand{\X}{\bm{X}}%
\newcommand{\bbeta}{\bm{\beta}}%
\newcommand{\balpha}{\bm{\alpha}}%
\newcommand{\bmu}{\bm{\mu}}%
\newcommand{\bnu}{\bm{\nu}}%
\newcommand{\bzero}{\bm{0}}%
\newcommand{\cV}{\mathcal{V}}%
\newcommand{\dist}{\textup{dist}}%
\newcommand{\revol}{\textup{ReVol}}%
\newcommand{\vol}{\textup{Vol}}%
\newcommand{\indic}{\mathds{1}}%
\newcommand{\hess}{\mathbb{H}}%
\newcommand{\linspace}{\mathcal{L}}%
\newcommand{\partition}{\bm{S}}%

\newcommand{\vor}{\cV}%
\newcommand{\ball}{\mB}%
\newcommand{\sphere}{\mS}%
\newcommand{\deltamax}{\Delta_{\max}}%
\newcommand{\deltamin}{\Delta_{\min}}%
\newcommand{\snrmax}{\eta_{\max}}%
\newcommand{\snrmin}{\eta_{\min}}%
\newcommand{\rad}{r}%
\newcommand{\std}{\sigma}%

\newcommand{\Wint}{A}
\newcommand{\Wext}{B}
\newcommand{\tail}{\varphi}

\maketitle

\begin{abstract}

$k$-means clustering is a fundamental problem in unsupervised learning. The problem concerns finding a partition of the data points into $k$ clusters such that the within-cluster variation is minimized. Despite its importance and wide applicability, a theoretical understanding of the $k$-means problem has not been completely satisfactory. Existing algorithms with theoretical performance guarantees often rely on sophisticated (sometimes artificial) algorithmic techniques and restricted assumptions on the data. The main challenge lies in the non-convex nature of the problem; in particular, there exist additional local solutions other than the global optimum. Moreover, the simplest and most popular algorithm for $k$-means, namely Lloyd's algorithm, generally converges to such spurious local solutions both in theory and in practice.

In this paper, we approach the $k$-means problem from a new perspective, by investigating the \emph{structures} of these spurious local solutions under a probabilistic generative model with $k$ ground truth clusters. As soon as $k=3$, spurious local minima provably exist, even for well-separated and balanced clusters. One such local minimum puts two centers at one true cluster, and the third center in the middle of the other two true clusters. For general $k$, one local minimum puts multiple centers at a true cluster, and one center in the middle of multiple true clusters. Perhaps surprisingly, we prove that this is essentially the \emph{only} type of spurious local minima under a separation condition. Our results pertain to the $k$-means formulation for mixtures of Gaussians or bounded distributions. Our theoretical results corroborate existing empirical observations and provide justification for several improved algorithms for $k$-means clustering.
\end{abstract}


\section{Introduction\label{sec:intro}}

$k$-means clustering is one of the most fundamental problems in unsupervised learning, with a wide range of applications in multiple fields including machine learning, image analysis, computer graphics and beyond; see \citet{jain2010data} and the references therein. The $k$-means problem can be formulated as follows: given $n$ data points $\x_{1},\ldots,\x_{n}\in\real^{d}$, find $k$ centers $\bbeta=(\bbeta_{1},\ldots,\bbeta_{k})\in\real^{d\times k}$ such that the following sum of squared distances is minimized:\footnote{Another common way of formulating the $k$-means problem involves finding a partition of the data points into $k$ clusters such that the within-cluster sum of squared distance is minimized. This partition-based formulation is equivalent to the center-based formulation~(\ref{eq:k-means-formulation}) used in this paper, as we show in Appendix~\ref{sec:Equivalence-to-standard-k-means}.}
\begin{equation}
G_n(\bbeta):=\sum_{i=1}^{n}\min_{j\in[k]}\|\x_{i}-\bbeta_{j}\|^{2}.\label{eq:k-means-formulation}
\end{equation}
The $k$-means objective function~\eqref{eq:k-means-formulation} is non-convex: it involves the minimization of quadratic functions and is symmetric with respect to permutation of the indices of components of $\bbeta$. This optimization problem is known to be NP-hard in general~\citep{dasgupta2008hardness,mahajan2009planar,awasthi2015hardness}. It has been observed that standard algorithms for $k$-means often converge to spurious local solutions of~\eqref{eq:k-means-formulation} that are not globally optimal~\citep{macqueen1967some, hartigan1979algorithm}. Moreover, these local minima of $k$-means are prevalent in practice~\citep{steinley2003local,steinley2006profiling}.

Recent theoretical work has made progress in understanding the $k$-means and related clustering problems with \emph{two} clusters. In particular, if the data is generated from a balanced mixture of two identical and spherical Gaussians, the work in \citep{chaudhuri2009learning,xu2016global,daskalakis2016ten} effectively shows that there is no spurious local minima, and that greedy algorithms such as the Lloyd's algorithm and Expectation-Maximization~(EM) are guaranteed to converge to a global minimizer from a random initialization. However, as soon as there are more than two clusters, non-trivial spurious local solutions do exist, even when the ground truth clusters are well-separated and balanced. Worst yet, these spurious local solutions may have objective values arbitrarily worse than the global optimum, and randomly-initialized greedy algorithms may provably converge to these local solutions with high probability~\citep{jin2016local}.

Despite above negative results, not all hope is lost. In this paper, we show that even with a general number of clusters, a lot can be said about the structural properties of these spurious local minima. In particular, under certain mixture models, we prove that \emph{all} spurious local minima of $k$-means are well-behaved, in the sense that they possess the same type of structure that partially recover the global minimum. We elaborate below.

\subsection{Main Contributions}

Consider the $k$-means problem under the following probabilistic generative model. Let $\bbeta_{1}^{*},\ldots,\bbeta_{k}^{*}\in\real^{d}$ be $k$ distinct unknown true cluster centers. For each $s\in[k]$, let $f_{s}$ be the density of a distribution with mean $\bbeta_{s}^{*}$. Each data point $\x\in\real^{d}$ is sampled independently from a mixture $f$ of these distributions $\{f_{s}\}_{s\in[k]}$, with the density 
\begin{equation}
f(\x)=\frac{1}{k}\sum_{s=1}^{k}f_{s}(\x).\label{eq:ground_truth_density}
\end{equation}
Note that if each $f_{s}$ is a Gaussian distribution centered at $\bbeta_{s}^{*}$, the above distribution reduces to the (balanced/equally-weighted) Gaussian Mixture Model (GMM). Under the generative model~(\ref{eq:ground_truth_density}), we consider the following population version of the $k$-means objective function:
\begin{align}
\label{eq:pop-kmean}
G(\bbeta) & =\int\min_{j\in[k]}\|\x-\bbeta_{j}\|^{2}f(\x)\ddup\x
  =\frac{1}{k}\sum_{s=1}^{k}\int\min_{j\in[k]}\|\x-\bbeta_{j}\|^{2}f_{s}(\x)\ddup\x.
\end{align}
The objective function above can be viewed as the infinite-sample ($n\to\infty$) limit of the empirical objective function in equation~\eqref{eq:k-means-formulation}. Note that this population objective is still non-convex.

\paragraph{Existence of spurious local minima.}

Under general conditions, the ground truth centers $\bbeta^{*}=(\bbeta_{1}^{*},\ldots,\bbeta_{k}^{*})\in\real^{d\times k}$ and any permutation thereof are (close to) a global minimum of $G$; see Proposition~\ref{prop:truth_is_global_opt}. However, there exist additional spurious local minima, even in the simple one-dimensional setting with $k=3$ clusters and when the densities $\left\{ f_{s}\right\}_{s\in[k]} $ have bounded and disjoint supports. In particular, we show that one spurious local minimum $\bbeta=(\bbeta_1, \bbeta_2, \bbeta_3)$ has the following configuration:
\begin{equation}\label{eq:bad_local_min}
\bbeta_{1}\approx\bbeta_{2}\approx\bbeta_{1}^{*}\qquad\text{and}\qquad\bbeta_{3}\approx\frac{\bbeta_{2}^{*}+\bbeta_{3}^{*}}{2}.
\end{equation}
In words, this local solution uses two centers to fit one true cluster, and the third center to fit the other two true clusters. See Proposition~\ref{prop:existence} for details. A similar observation was made in \cite{jin2016local} for the log-likelihood objective function of Gaussian mixtures.

\paragraph{Structures of spurious local minima.}

The above local solution involves disjoint many-fit-one and one-fit-many associations. Perhaps surprisingly, we show that this is essentially the \emph{only} type of spurious local minima for $k$-means with a general $k$ under a separation condition:
\begin{theorem*}[Informal]
\label{thm:informal}For well-separated mixture models, all non-degenerate local minima $\bbeta = (\bbeta_{1},\ldots,\bbeta_{k})$ of $G$ must have the following form: (i) multiple centers $\{\bbeta_{j}\}$ lie near a true cluster $\bbeta_{s}^{*}$, or (ii) one center $\bbeta_{j}$ lies near the mean of multiple true clusters $\{\bbeta_{s}^{*}\}$. Moreover, the configurations (i) and (ii) involve disjoint sets of $\bbeta_{j}$'s and $\bbeta_{s}^{*}$'s.
\end{theorem*}
See Theorems~\ref{thm:main_ball} and \ref{thm:main_gaussian} for the precise statement of this result. In words, viewing a solution $\bbeta$ as an assignment of the centers $\{\bbeta_{j}\}$ for fitting the ground truth clusters, we show that a local minimum $\bbeta$ can only involve many-fit-one associations (case (i) above) and one-fit-many associations (case (ii) above), and any fitted center $\bbeta_{s}$ or true cluster $\bbeta_{s}^{*}$ only participates in one of these associations. Any other solution $\bbeta$ with many-fit-many associations \emph{cannot} be a local minimum.


\begin{figure}
\begin{centering}
\begin{tabular}{|c|c|c|c|}
\hline 
 {\scriptsize 1a} \includegraphics[viewport=110bp 200bp 540bp 650bp,clip,scale=0.2]{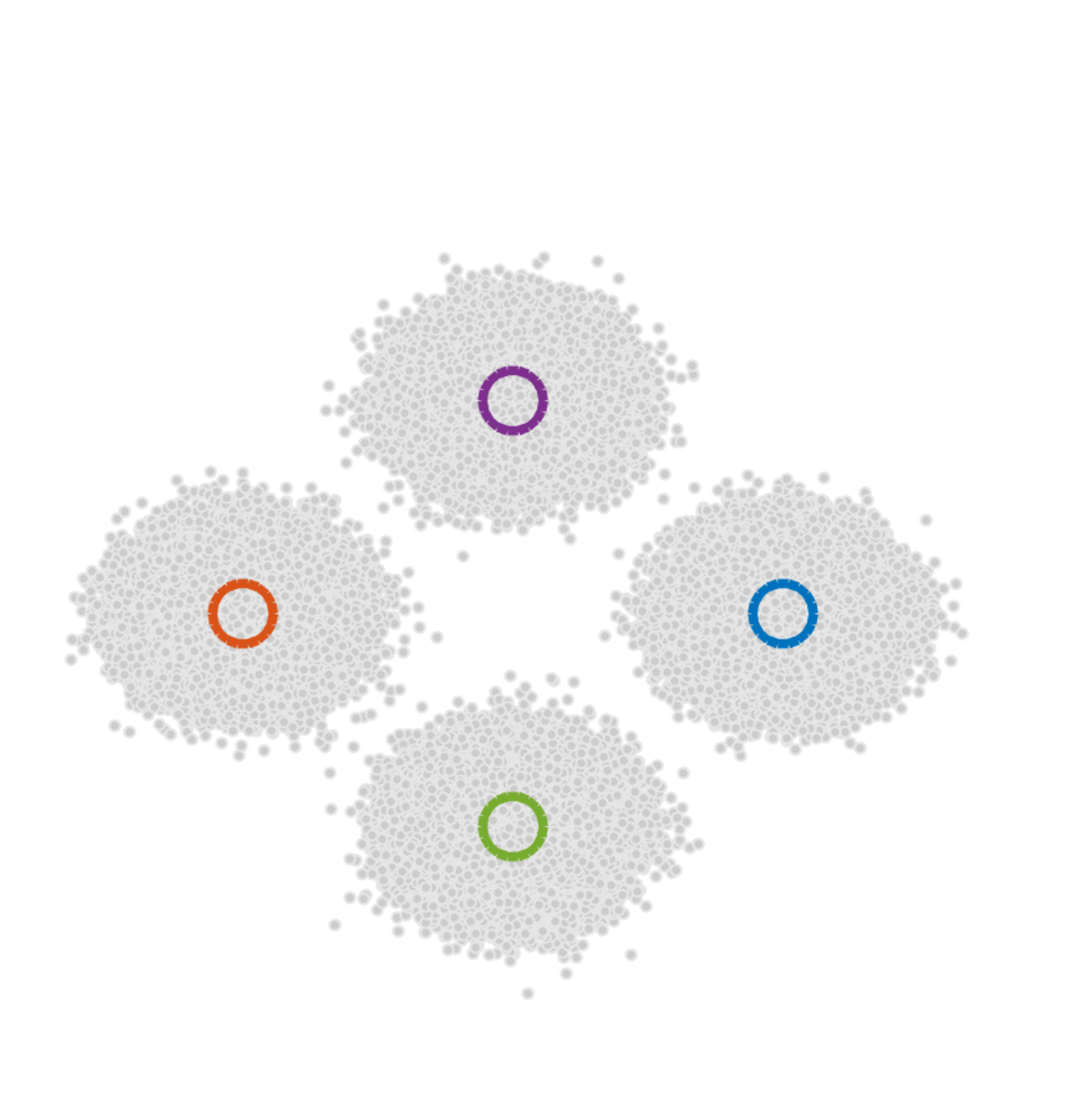} & {\scriptsize 1b} \includegraphics[viewport=110bp 200bp 540bp 650bp,clip,scale=0.2]{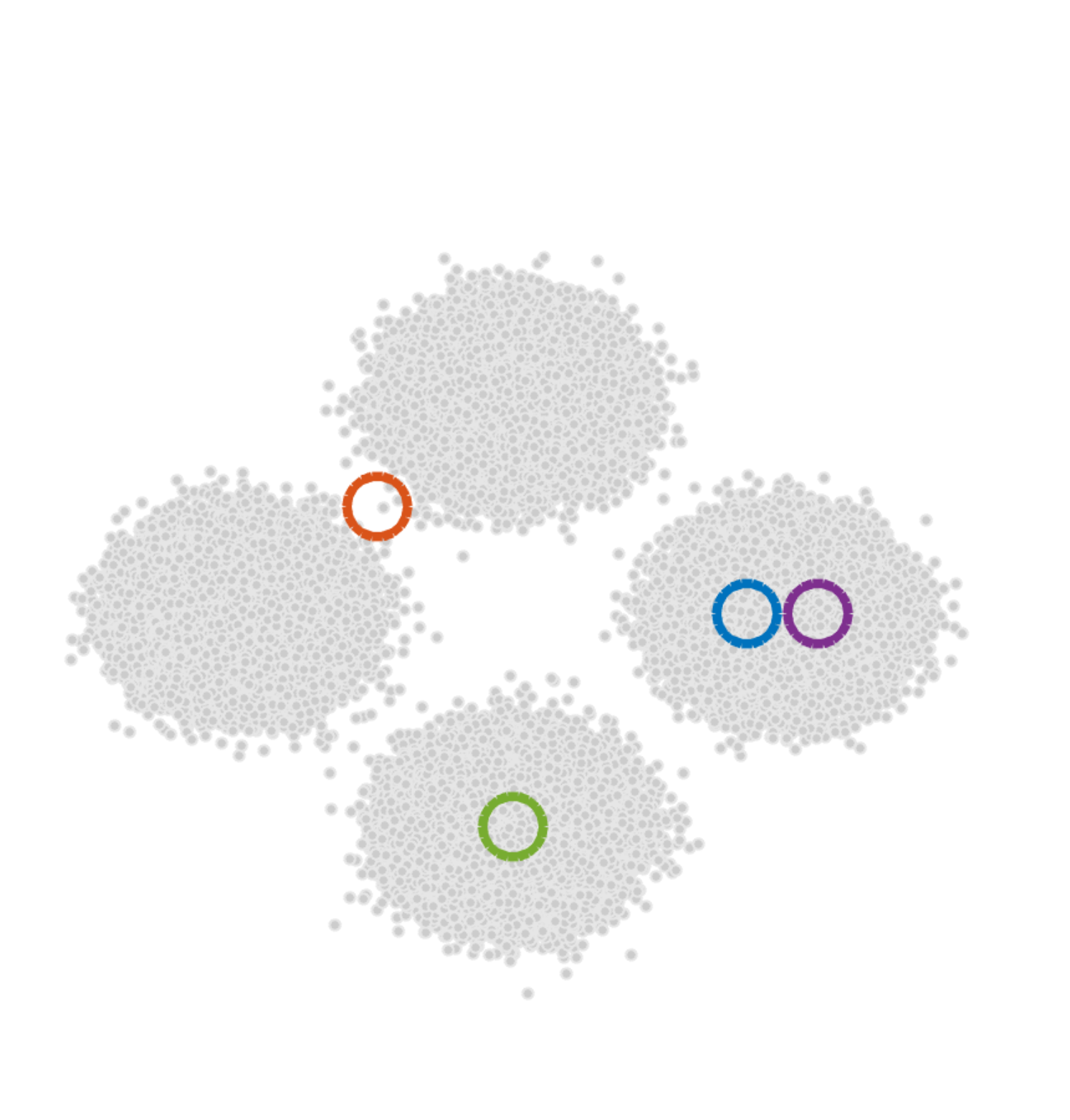} & {\scriptsize 1c} \includegraphics[viewport=110bp 200bp 540bp 650bp,clip,scale=0.2]{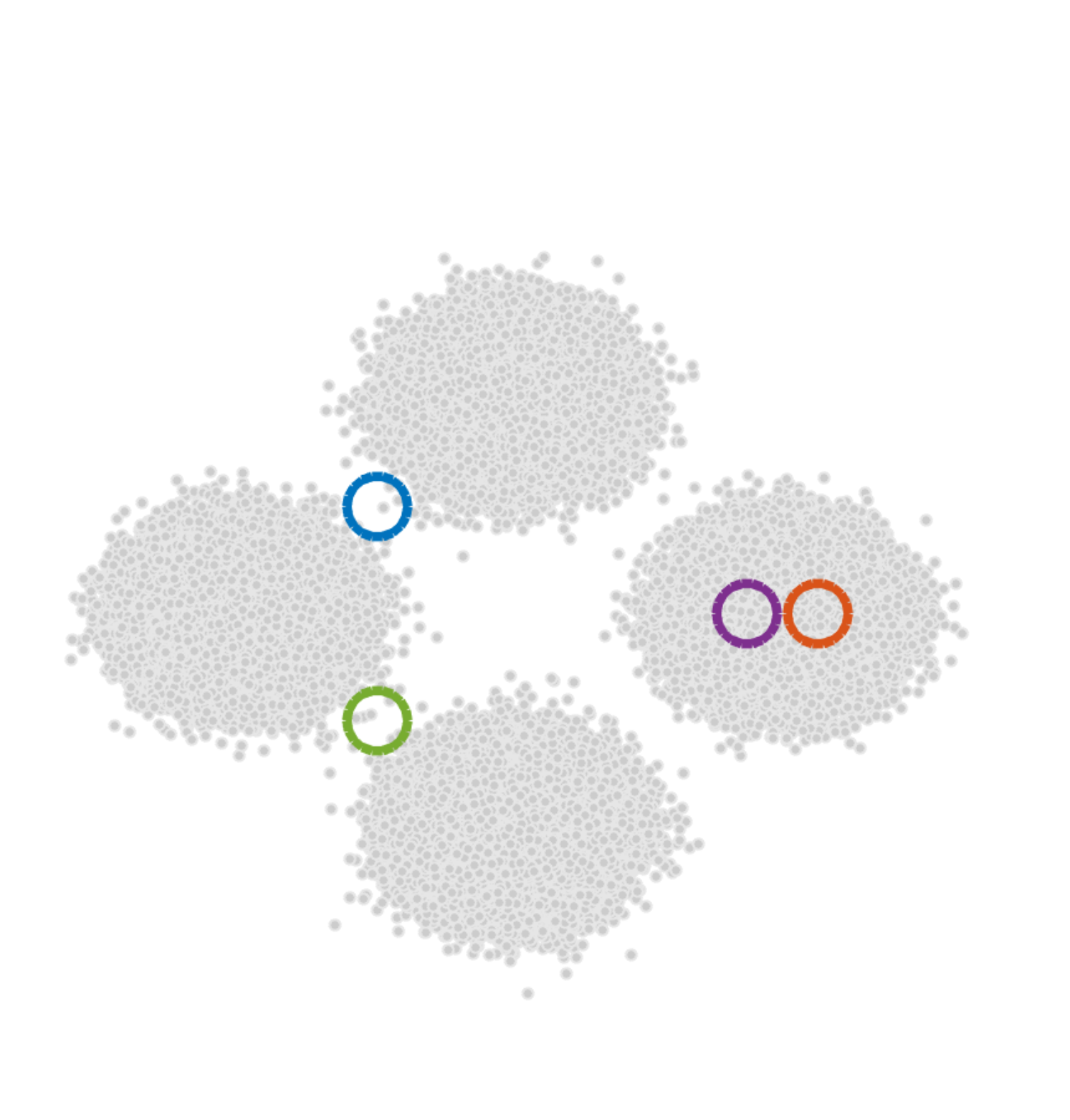} & {\scriptsize 1d} \includegraphics[viewport=110bp 200bp 540bp 650bp,clip,scale=0.2]{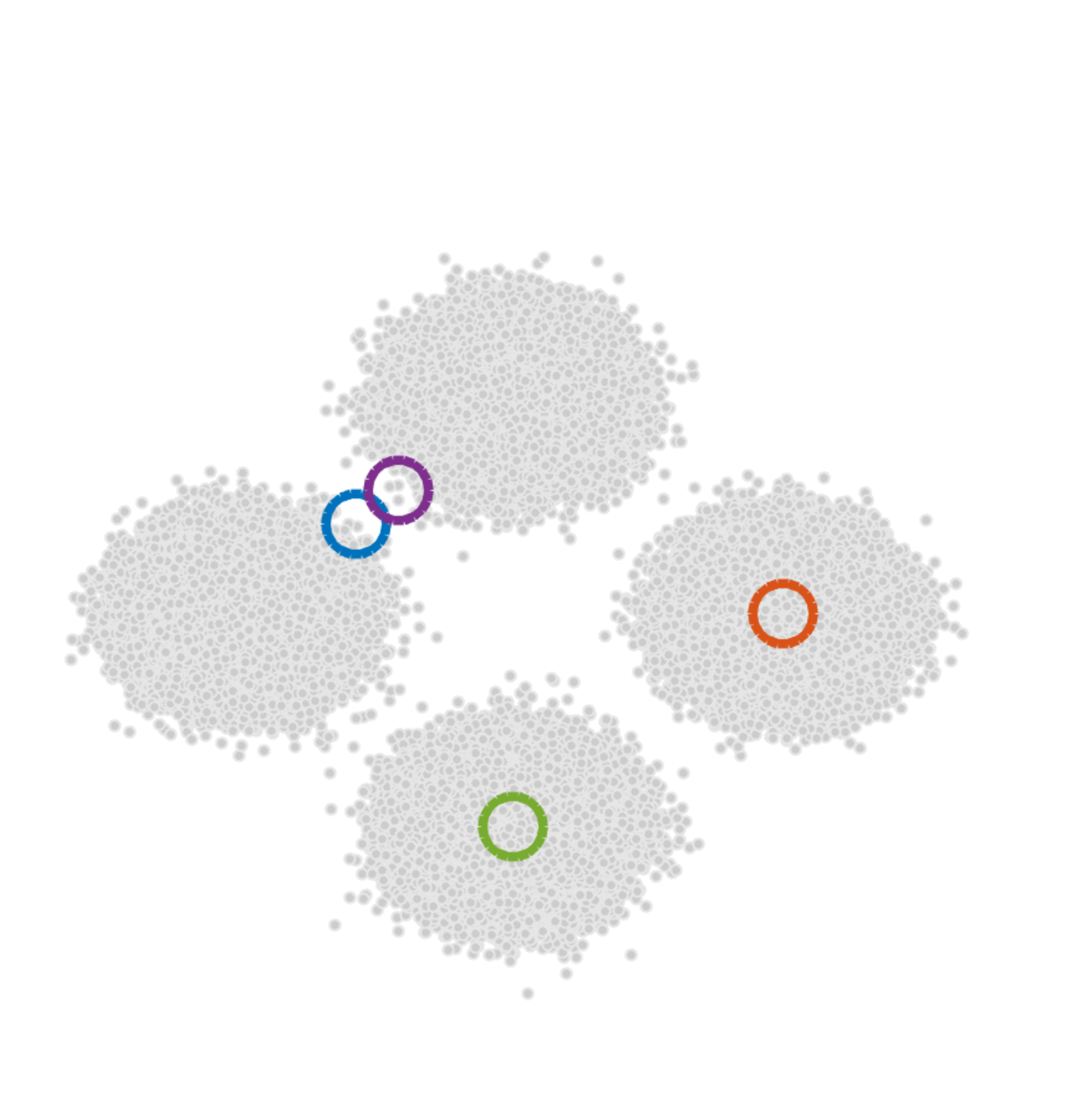}
\tabularnewline

\hline 
Global Min & Spurious Local Min & \multicolumn{2}{c|}{Not Local Min}\tabularnewline
\hline 

\hline 
 {\scriptsize 2a} \includegraphics[viewport=110bp 200bp 540bp 650bp,clip,scale=0.2]{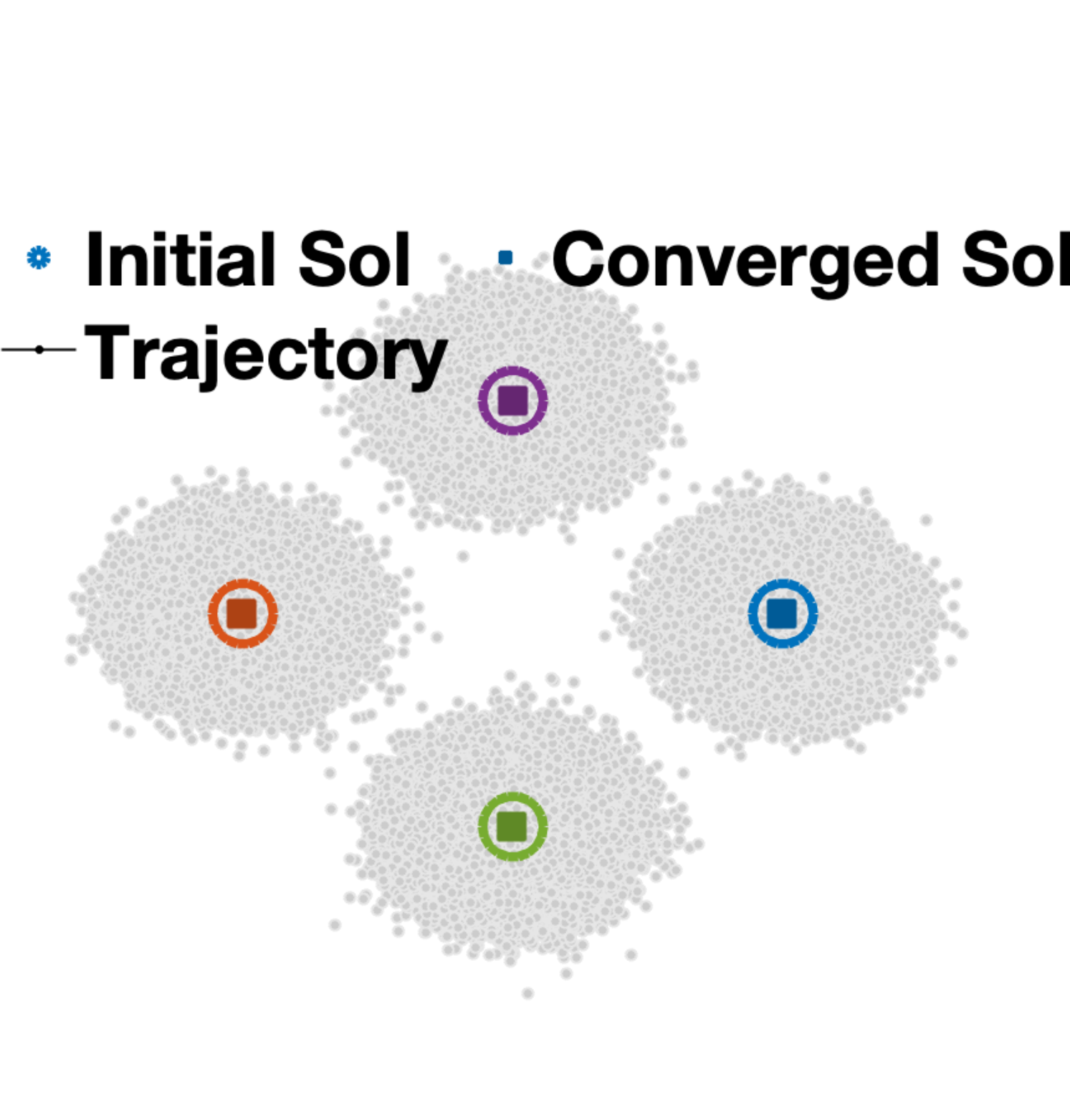} 
& {\scriptsize 2b} \includegraphics[viewport=100bp 200bp 550bp 650bp,clip,scale=0.2]{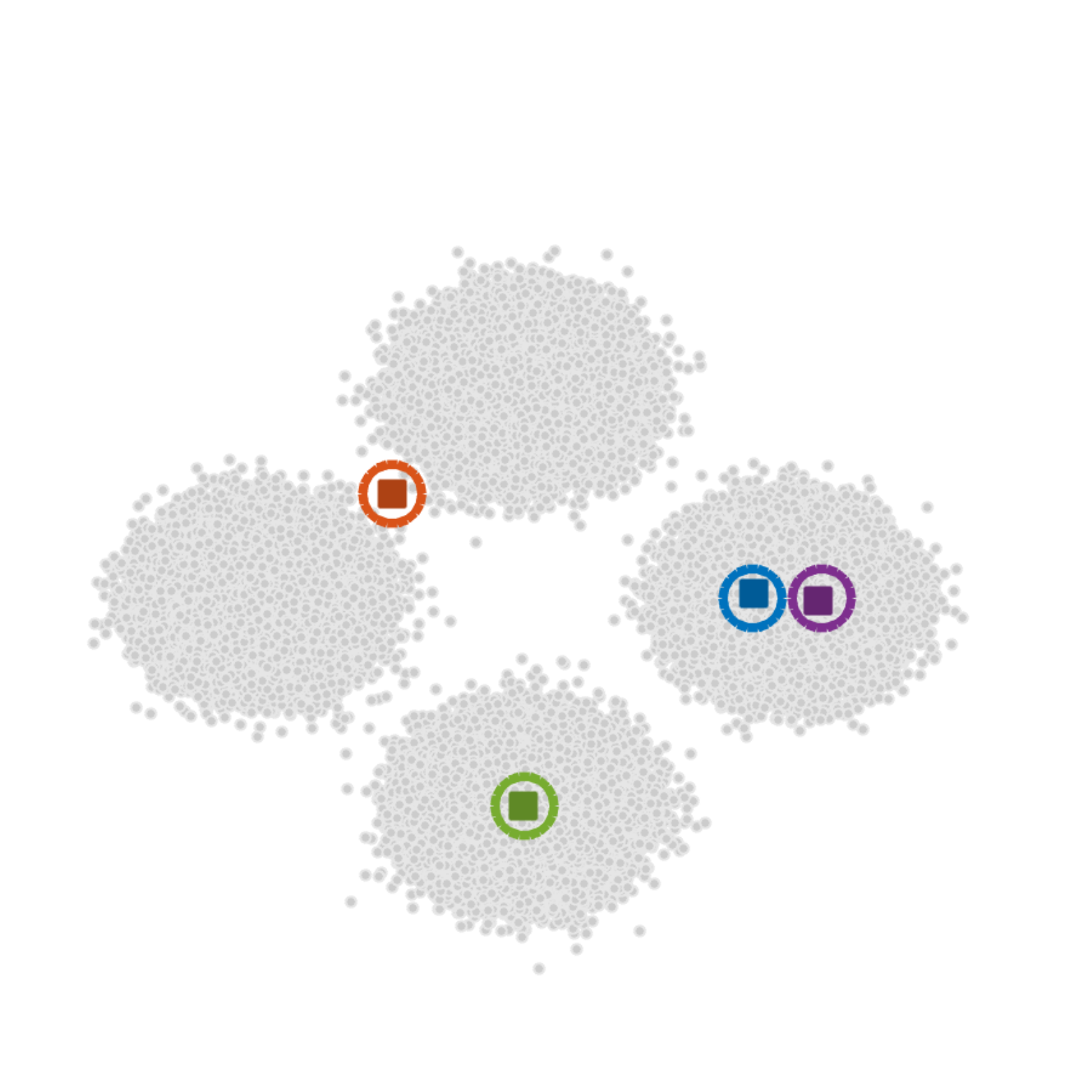}
& {\scriptsize 2c}\includegraphics[viewport=110bp 200bp 540bp 650bp,clip,scale=0.2]{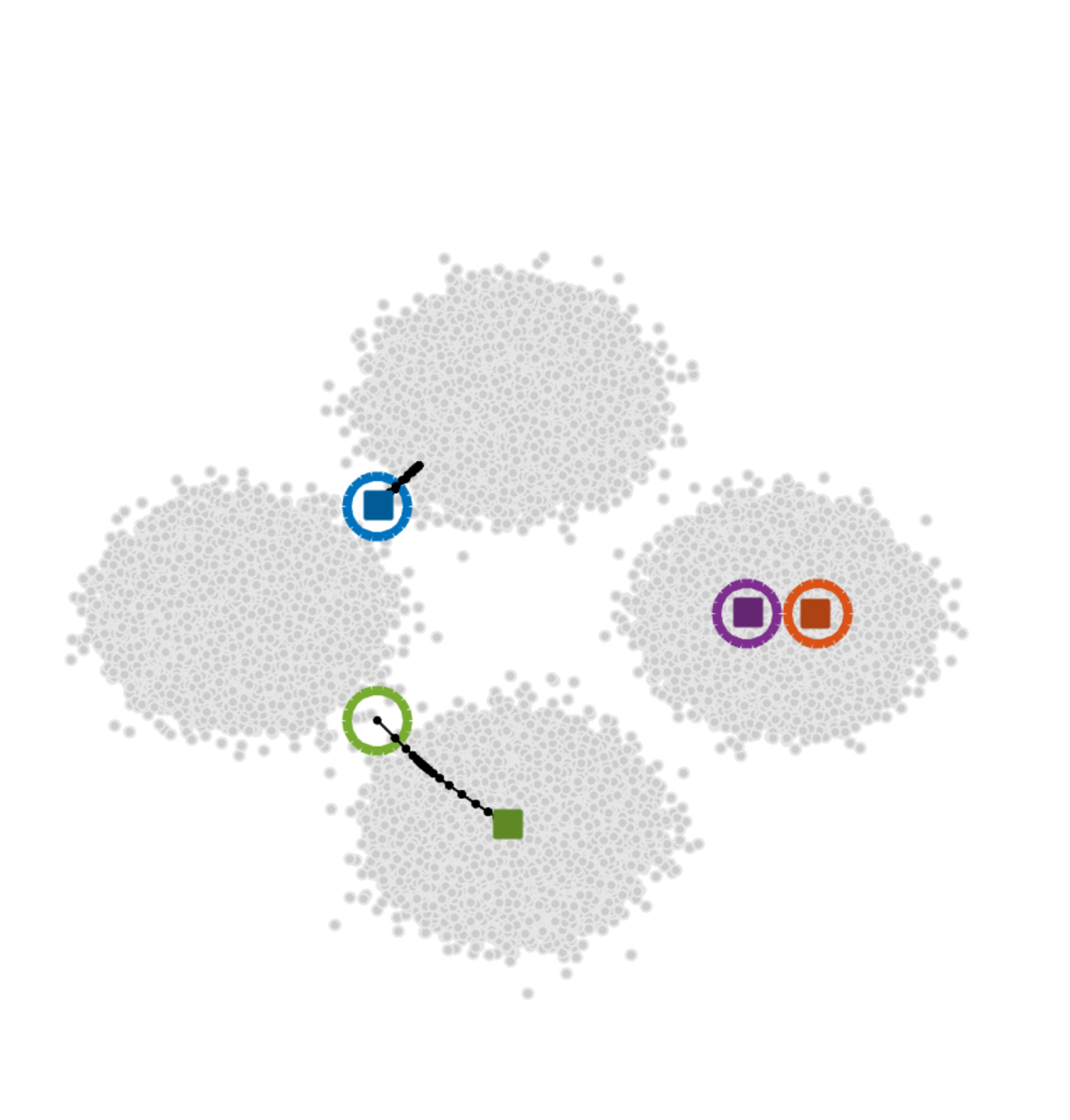} 
& {\scriptsize 2d} \includegraphics[viewport=110bp 200bp 540bp 650bp,clip,scale=0.2]{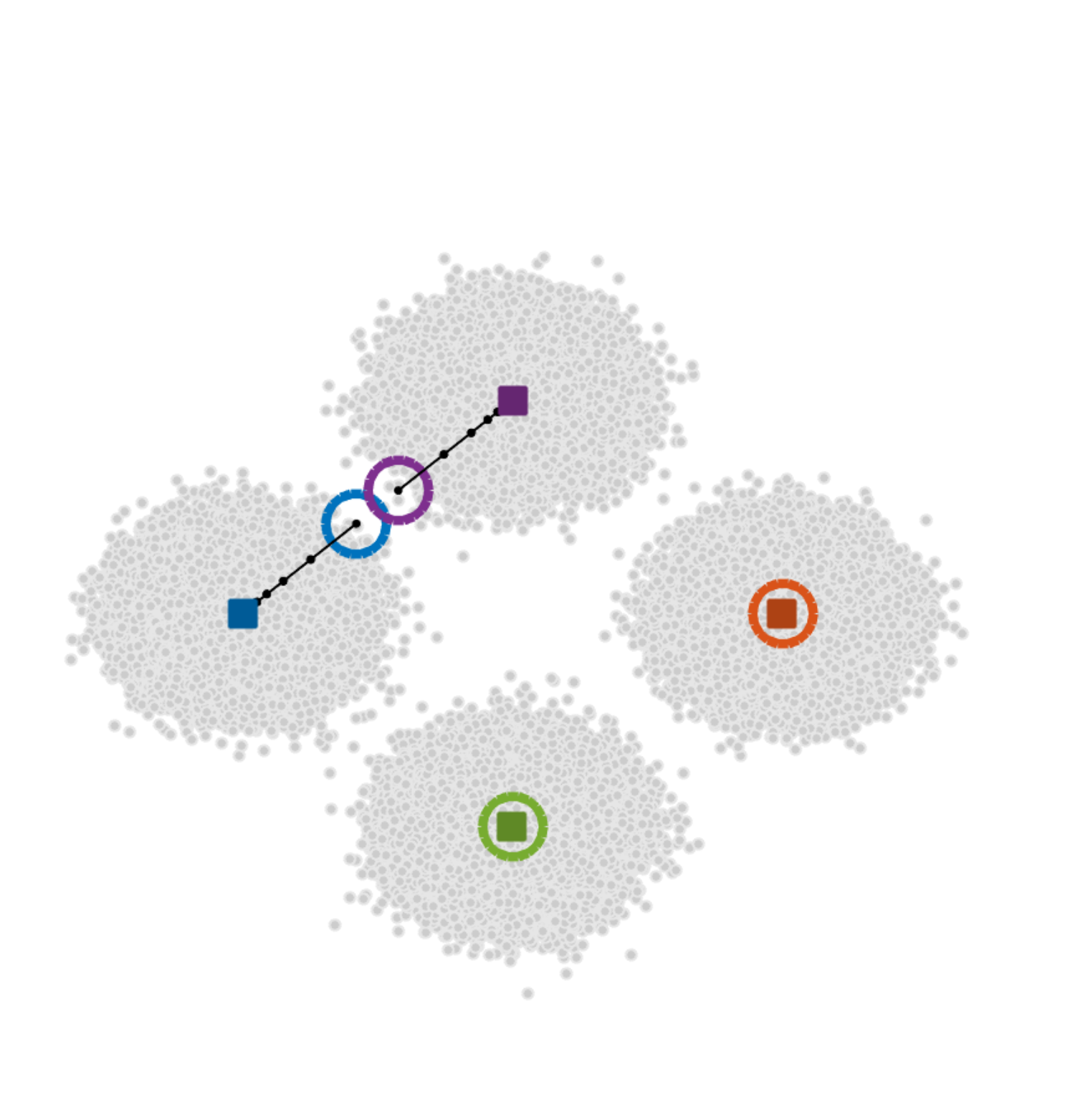}\tabularnewline
\hline 
\end{tabular}
\par\end{centering}
\caption{\label{fig:illustration} \emph{Top panels:} Local minima and non-minima in GMM with 4 components. Solutions with many-fit-many configurations are not local minima. \emph{Bottom panels:} Trajectory of greedy algorithm when initialized at different solutions. The colored circles correspond to an initial configuration of $\bbeta$. Running the Lloyd's $k$-means algorithm from this initialization converges to a solution denoted by colored squares. The black lines correspond to the trajectory of the intermediate iterates. The algorithm escapes from non-minima and converges to a global or local minimum.}
\end{figure}

We illustrate the above results under a two-dimensional GMM with 4 components in Figure~\ref{fig:illustration}. The top panels show different candidate solutions of $k$-means. The ground-truth centers are the only global minimum, as in Panel~1a. Panel~1b shows a spurious local minimum, where the orange center fits two clusters, and the blue and purple centers fit one cluster. In Panel~1c, the blue and green centers together fit 3 clusters; in Panel~1d, the blue and purple centers together fit 2 clusters. These two solutions contain many-fit-many associations and are \emph{not} local minima.

For further verification, we run the Lloyd's algorithm~\cite{lloyd1982least} with the above four solutions as the initial solution. The Lloyd's algorithm is an iterative greedy method that alternates between assigning each data point to its closest center and updating the centers to be the means of the new clusters. It can be viewed as a quasi-Newton algorithm applied to the objective function~\eqref{eq:k-means-formulation} with a specific choice of step size~\cite{bottou1995kmeans}. The bottom panels in Figure~\ref{fig:illustration} show the trajectories of intermediate solutions of Lloyd's algorithm and the final solutions they converge to. When initialized at a global or local minimum, the algorithm stays at the initial solutions as expected (Panels~2a and~2b). In Panel~2c, the algorithm escapes from the initial solution, which is not a local minimum, and then converges to the spurious local minimum plotted in Panel~1b. In Panel~2d, the algorithm again escapes from the initial solution and converges to the globally optimal ground-truth solution plotted in Panel~1a.\\

We conclude this section by mentioning that Srebro~\cite{srebro2007question} posted the following question in 2007: are all local optimal solutions of the population GMM likelihood function globally optimal? In general the answer has been shown to be negative~\cite{jin2016local}, as demonstrated by an example similar to that in equation~\eqref{eq:bad_local_min}. However, our results above provide a positive message in the context of the $k$-means objective (a limit version of the log-likelihood function; see Section~\ref{sec:models_ball}): all local minima \emph{partially} recover the global minimum, in the sense that they identify some of true cluster centers and the means of the other true cluster centers; again see Panel~2b in Figure~\ref{fig:illustration} for an illustration.

\subsection{Related Work \label{sec:related}}

With a history of more than 50 years~\citep{lloyd1982least,macqueen1967some}, the $k$-means problem has found broad applications in computer science, astronomy, biology, social science and beyond. We refer to the papers in~\citet{steinley2006k,jain2010data} for a comprehensive survey of the work on this problem.

Without additional assumptions on the data points, optimizing the $k$-means objective in~\eqref{eq:k-means-formulation} is NP-hard when the number of components $k$ is fixed~\citep{dasgupta2008hardness} or when the dimension $d$ is fixed~\citep{mahajan2009planar}. Even finding a $(1+\epsilon)$ approximation with varying $(k,d)$ is hard~\citep{awasthi2015hardness}. Progress has been made on designing constant-ratio approximation algorithms; see, e.g., the results in~\citet{kumar2004simple,kanungo2004local} among many others.

Lloyd's algorithm \citep{lloyd1982least}, often called \emph{the} $k$-means algorithm, is arguably the most popular algorithm for the $k$-means problem. In general,  Lloyd's algorithm is only guaranteed to converge to a local minimum of the $k$-means objective and is sensitive to initialization~\citep{milligan1980examination}. Moreover, it may take exponentially many steps to converge in the worst case~\citep{har2005fast, arthur2006slow}. Under certain probabilistic assumptions of the data, several theoretical guarantees have been established for the Lloyd's algorithm~\citep{chaudhuri2009learning, kumar2010clustering, lu2016statistical}. There is also substantial work on  designing provably efficient initialization schemes for Lloyd's algorithm~\citep{arthur2007k, ostrovsky2012effectiveness}.  Particularly relevant to us is the work in~\citep{dasgupta2007probabilistic}, which considers over-parametrized $k$-means/EM (which fits $k$ clusters using more than $k$ centers) equipped with extra pruning steps. Interestingly, the fitted centers they try to prune correspond to, in our language, many-fit-one associations (as well as almost-empty associations; see our main theorems). As Lloyd's algorithm finds local minima of $k$-means, our results can be used to characterize the structural properties of the output of Lloyd's. Note that our results are in fact more general, applicable to the general $k$-means objective function (with or without over-parametrization) and hence are not tied to a specific algorithm.

We mention that recent work also considers convex relaxation methods for the $k$-means problem based on linear or semidefinite programming~\citep{peng2005new,peng2007approximating,elhamifar2012finding}. Theoretical guarantees have been established on when the solution of the convex program coincides with (or approximates) the global minimum of $k$-means~\citep{nellore2015recovery,awasthi2015relax,li2020birds,fei2018hidden}.

As mentioned, the $k$-means objective function can be viewed as a ``hard'' or limit version of the negative log-likelihood function for the Gaussian Mixture Model (GMM); see Section~\ref{sec:models_gaussian}. As such, our results are related to recent theoretical work on the Expectation-Maximization (EM) algorithm~\citep{dempster1977maximum}, which is a local/greedy algorithm for optimizing the likelihood function. Positive results have been obtained for provable convergence of EM under GMM with $k=2$ components~\citep{balakrishnan2017statistical,xu2016global,daskalakis2016ten,qian2019global,kwon2019global}. In particular, these results suggest that the negative log-likelihood function has  no spurious local minima for a balanced mixture of two Gaussians with the same covariance matrix. However, in more general mixture models, it has been proved that spurious local minima do exist with high probability. Examples include a mixture of $k\ge3$ equally weighted components~\citep{jin2016local}, and a mixture of $k=2$ unequally weighted components with known mixing weights~\citep{xu2018benefits}.

\section{Problem Setup\label{sec:models and set up}}
In this section, we introduce the statistical models and notations
for our main results.
We shall consider two concrete instantiations of the mixture model in equation~(\ref{eq:ground_truth_density}).

\subsection{Ball Mixture} \label{sec:models_ball}

The first instantiation is a mixture of uniform distributions on $k$ disjoint balls. For each $\bu\in\real^{d}$, let $\ball_{\bu}(\rad)$ denote the Euclidean ball centered at $\bu$ with radius $\rad$. As the true centers $\{\bbeta_{s}^{*}\}_{s\in[k]}$ and the radius $\rad$ are fixed throughout this paper, we use the shorthand $\ball_{s}\equiv\ball_{\bbeta_{s}^{*}}(\rad)$ for brevity. We assume that each data point $\x$ is sampled independently and uniformly from one of $k$ disjoint balls centered at the true centers $\bbeta_{s}^{*}$; that is, $\x\sim\text{unif}\left(\ball_{s}\right)$ with probability $\frac{1}{k}$. 

This model, sometimes called the Stochastic Ball Model~\citep{nellore2015recovery}, is formally described below.
\begin{definition}[Stochastic Ball Model]
\label{def:SBM}
The Stochastic Ball Model is the mixture~(\ref{eq:ground_truth_density}) where each component has density 
\begin{equation*}
f_{s}(\x)=\frac{1}{\vol(\ball_{s})}\indic_{\ball_{s}}(\x),\quad s\in[k].
\end{equation*}
\end{definition}
Here $\vol(T)$ denotes the volume of a set $T\subseteq\real^{d}$ with respect to the Lebesgue measure, and and $\indic_{T}$ is the indicator function for $T$.

\subsection{Gaussian Mixture}\label{sec:models_gaussian}
The second instantiation is the (spherical) Gaussian mixture model, where each data point $\x$ is sampled independently from one of $k$ Gaussian distributions whose means are the true centers $\{\bbeta_{s}^{*}\}$; that is, $\x\sim\mathcal{N}(\bbeta_{s}^{*},\std^{2}\bm{I})$ with probability  $\frac{1}{k}$. A formal description of GMM is given below.
\begin{definition}[Gaussian Mixture Model]
\label{def:GMM}
The (spherical) Gaussian Mixture Model is the mixture~(\ref{eq:ground_truth_density}) where each component has density
\[
f_{s}(\x)=
\frac{1}{(\sqrt{2\pi}\std)^{d}}\exp\left(-\frac{\|\x-\bbeta_{s}^{*}\|^{2}}{2\std^{2}}\right),\quad s\in[k].
\]

\end{definition}
We point out that the population negative likelihood function of GMM (with a positive variance parameter $\tau^2$), namely\footnote{Here we ignore an constant additive term independent of the variable  $\bbeta$.} 
\begin{equation*}
    L_\tau (\bbeta) := 
    - \int \log \bigg[\sum_{j\in[k]} \exp\left(-\frac{\|\x-\bbeta_{j}\|^{2}}{2\tau^{2}}\right) \bigg] f(\x) \ddup \x ,
\end{equation*}
is closely related to the population $k$-means objective function $G$ defined in equation~\eqref{eq:pop-kmean}. As the log-sum-exp function above is a form of soft maximum, $L_\tau$ is a smooth approximation of $G$. Moreover, as $\tau \to 0$, we have $2\tau^2  L_\tau (\bbeta) \to G(\bbeta)$ for all $\bbeta$. In other words, the $k$-means objective function corresponds to the limit case of the GMM log-likelihood function, and hence results for one have immediate bearing upon the other.

\subsection{Model Parameters}

For both of the above models, we define the quantities 
\[
\deltamax:=\max_{s\neq s'}\|\bbeta_{s}^{*}-\bbeta_{s'}^{*}\|\qquad\text{and}\qquad\deltamin:=\min_{s\neq s'}\|\bbeta_{s}^{*}-\bbeta_{s'}^{*}\|,
\]
which are the maximum and minimum pairwise separations between the true centers. Accordingly, we introduce two quantities measuring the  Signal-to-Noise Ratios (SNR) of the models. In particular, for the Stochastic Ball Model we define
\begin{align*}
\snrmax:=\frac{\deltamax}{\rad}\qquad\text{and}\qquad\snrmin:=\frac{\deltamin}{\rad},
\end{align*}
which are the maximum and minimum separations normalized by the radius of the balls. For the Gaussian Mixture Model, we similarly define
\begin{align*}
\snrmax:=\frac{\deltamax}{\std \sqrt{\min(2k,d)}}\qquad\text{and}\qquad\snrmin:=\frac{\deltamin}{\std\sqrt{\min(2k,d)}}.
\end{align*}
Note the $\sqrt{\min(2k,d)}$ factor in the denominators above. This factor is the typical value of the norm of a random vector from a $d$-dimensional standard Gaussian distribution when projected to the $2k$-dimensional subspace spanned by the true and fitted centers $\{\bbeta^*_s\}_{s\in[k]}$ and $\{\bbeta_i\}_{i\in[k]}$. 

The above models are sometimes said to be \emph{well-separated} if $\snrmin = \Omega(1)$~\citep{jin2016local}. Also note that the ratio $\frac{\snrmax}{\snrmin}\in[1,\infty)$ measures how evenly-spaced the true centers are. In particular, this ratio is close to $1$ when the true centers are approximately evenly spaced. 

\subsection{\label{sec:voronoi}Voronoi sets}

Each candidate solution $\bbeta=(\bbeta_{1},\ldots,\bbeta_{k})$ of the $k$-means problem induces a  \emph{Voronoi diagram}, namely, a partition of the space $\real^{d}$ based on proximity to the $\bbeta_s$'s. The Voronoi diagram plays a crucial role in understanding the $k$-means objective~(\ref{eq:pop-kmean}), which is defined by the quantity $\min_{j\in[k]}\|\x-\bbeta_{j}\|$, the distance of a point $\x$ to its closest center. Here we review some basic concepts related to Voronoi diagrams, which are useful for future development.

Given a set of $k$ centers  $\bbeta=(\bbeta_{1},\ldots,\bbeta_{k})\in\mathbb{\real}^{d\times k}$ in $\real^{d}$, let $\vor_{i}(\bbeta)$ be the region consisting of points that are closer to $\bbeta_{i}$ than to any other center $\bbeta_{j}$, $j\neq i$. Formally, for each $i\in[k]$ we define
\begin{equation}
\vor_{i}(\bbeta):=\{\x\in\real^{d}:\|\x-\bbeta_{i}\|\leq\|\x-\bbeta_{j}\|,\forall j\neq i\}.\label{eq:Voronoi-set-associated-with-pi}
\end{equation}
We call each $\vor_{i}(\bbeta)$ the \emph{Voronoi set} associated with $\bbeta_{i}$. The \emph{Voronoi diagram} of $\bbeta$ is the collection of the Voronoi sets, that is,  $\vor(\bbeta) := \{\vor_{i}(\bbeta): i\in[k]\}.$ Note that each Voronoi set is a polyhedron in $\real^d$ with at most $k-1$ facets,\footnote{A facet is a $(d-1)$ dimensional face of a $d$-dimensional polyhedron.} as we can rewrite the definition in~(\ref{eq:Voronoi-set-associated-with-pi}) as 
\[
\vor_{i}(\bbeta)=\{\x\in\real^{d}:2\langle\bbeta_{j}-\bbeta_{i},\x\rangle\leq\|\bbeta_{j}\|^{2}-\|\bbeta_{i}\|^{2},\forall j\neq i,j\in[k]\}.
\]

In addition, for each index pair $(i,j):i\neq j$, we define the \emph{Voronoi boundary} $\partial_{i,j}(\bbeta)$ as the intersection of the Voronoi sets  associated with $\bbeta_{i}$ and $\bbeta_{j}$; that is,
\[
\partial_{i,j}(\bbeta):=\vor_{i}(\bbeta)\cap\vor_{i}(\bbeta)=\{\x\in\vor_{i}(\bbeta)\cup\vor_{j}(\bbeta):\|\x-\bbeta_{i}\|=\|\x-\bbeta_{j}\|\}.
\]
Note that $\partial_{i,j}(\bbeta)$ is the set of points with equal distance to $\bbeta_{i}$ and $\bbeta_{j}$. If $\partial_{i,j}(\bbeta)$ has dimension $d-1$, we say that $\vor_{i}(\bbeta)$ is \emph{adjacent} to $\vor_{j}(\bbeta)$, written as $\vor_{i}(\bbeta)\sim\vor_{j}(\bbeta)$. In this case, the two Voronoi sets $\vor_{i}(\bbeta)$ and $\vor_{j}(\bbeta)$ intersect at a common (full dimensional) facet of the two polyhedra.  We use the notation $\partial(\bbeta): = \{\partial_{i,j}(\bbeta): \vor_{i}(\bbeta)\sim\vor_{j}(\bbeta)\}$ to denote the collection of the Voronoi boundaries of adjacent Voronoi sets.

\section{\label{sec:main}Main Results}

In this section, we present our main theoretical results on the structures of the local minima of the population $k$-means objective $G$ defined in equation~\eqref{eq:pop-kmean}. In what follows, we use $\mP$ to denote the probability measure with respect to the distribution of the ground truth mixture, whose density is $f$. Similarly, for each $s\in[k]$, we use $\mP_{s}$ to denote the probability measure with respect to the distribution of the $s$-th ground truth cluster, whose density is $f_{s}$.

\subsection{Stochastic Ball Model}\label{sec:main_sbm}
Consider the Stochastic Ball Model in Definition~\ref{def:SBM}. We first state two basic results concerning the global and local minima of the $k$-means objective $G$. The first proposition, proved in Appendix~\ref{sec:proof_truth_global_opt}, states that the ground truth centers is the only global minimum of $G$.
\begin{proposition}[Ground truth is global minimum]
\label{prop:truth_is_global_opt} 
Under the Stochastic Ball Model, if $\snrmin\ge 6\sqrt{k}$, then the true centers $\bbeta^{*}=(\bbeta_1^*,\ldots,\bbeta_k^*) \in \real^{d\times k}$ (up to permutation of its components) is the unique global minimum of $G$.
\end{proposition}

The next proposition, proved in Appendix~\ref{sec:proof_existence}, states that $G$ has a spurious local minimum that is not a global minimum. An illustration is given in Figure~\ref{fig:local-min-3-intervals}.
\begin{proposition}[Existence of spurious local minima]
\label{prop:existence}
Consider the Stochastic Ball Model in one dimension with $\beta_{1}^{*}=-2$, $\beta_{2}^{*}=0$, $\beta_{3}^{*}=2$, where each ground truth ball/interval has radius $\rad$. When $\rad<0.4$ or equivalently $\snrmin>5$, the solution $\bbeta=(\beta_1,\beta_2,\beta_3)\in\real^{1\times3}$ with $\beta_{1}=-2-\frac{\rad}{2}$, $\beta_{2}=-2+\frac{\rad}{2}$ and $\beta_{3}=1$ is a local minimum of $G$.
\end{proposition}

\begin{figure}
\centering
\includegraphics[scale=0.5,clip, trim=0 20 0 20]{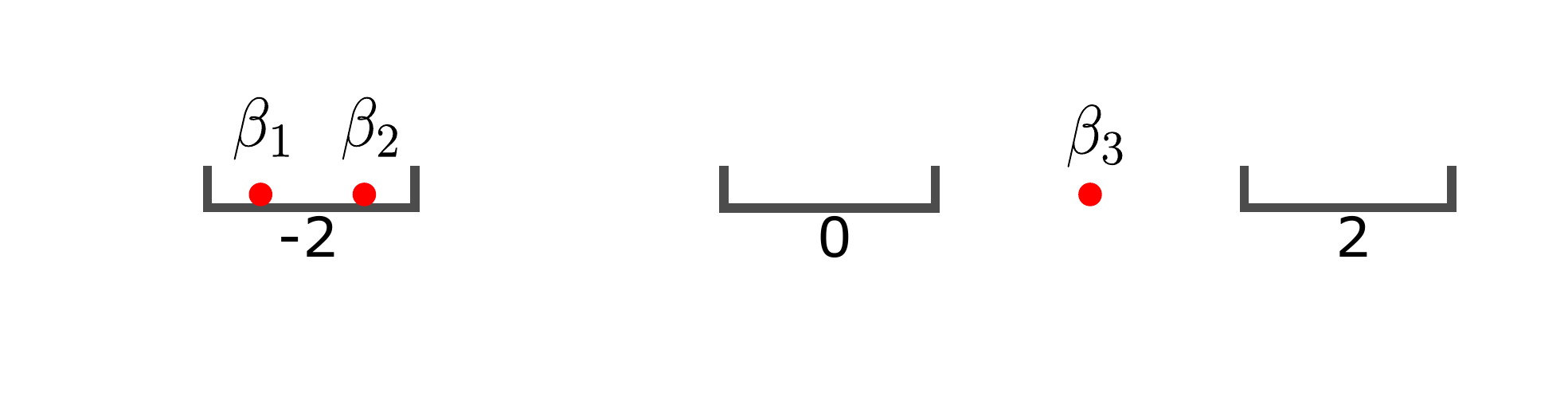}
\caption{\label{fig:local-min-3-intervals}One-dimensional Stochastic Ball Model with radius $\rad<0.4$ and ground truth cluster centers $\bbeta^{*}=(-2,0,2)$. The solution $\bbeta=(-2-\frac{\rad}{2},-2+\frac{\rad}{2},1)$ is a spurious local minimum.}
\end{figure}

Conceptually, Proposition~\ref{prop:truth_is_global_opt} shows that $G$ is a \emph{statistically} sensible objective function for clustering, as its global minimum recovers the ground truth clustering. On the other hand, Proposition~\ref{prop:existence} highlights the \emph{computational} difficulty of this optimization task, due to the existence of spurious local minima in the form of the configuration plotted in Figure~\ref{fig:local-min-3-intervals}.

As the main result of this paper, we show that the above configuration is essentially the \emph{only} local minimum, in a precise sense formalized in the theorem below.
\begin{theorem}[Local minimum structures, Stochastic Ball Model]
\label{thm:main_ball} 
Under the Stochastic Ball Model, assume that $\snrmax>4c^{2}k^{4}$ and $\snrmin\geq 10ck^2 \sqrt{\snrmax}$ for some universal constant $c\ge 3$. If $\bbeta =(\bbeta_1,\ldots,\bbeta_k)\in\real^{d\times k}$ is a local minimum of $G$, then the ground truth centers and fitted centers can be partitioned as $[k] = \bigcup_{a=1}^m S_a^*$ and $[k] = \bigcup_{a=0}^{m} S_a$ respectively, such that for each $a\in[m]$, exactly one of the following holds:
\begin{itemize}
    \item \textbf{(many/one-fit-one association)} $|S_a| \ge 1$ and $S_a^* = \{s\}$ for some $s\in[k]$; moreover, 
    \begin{align*}
        \|\bbeta_i-\bbeta_s^*\|\leq \deltamax\frac{8ck^2}{\sqrt{\snrmax}} = 8ck^2 \sqrt{\rad\deltamax},\quad\forall i\in S_a.
    \end{align*}

    \item \textbf{(one-fit-many association)} $S_a = \{i\} $ for some $i\in[k]$ and $|S_a^*|\ge 2$; moreover, 
    \begin{align*}
    \bigg\|\bbeta_i-\frac{1}{|S_a^*|}\sum_{s\in S_a^*}\bbeta_s^* \bigg\|\leq \deltamax\frac{11ck^{2}}{\sqrt{\snrmax}} = 11ck^2 \sqrt{\rad \deltamax}.
    \end{align*}

\end{itemize}
In addition, for each $i\in S_{0}$, we have $\mP\big(\vor_{i}(\bbeta)\big)\leq\frac{ck}{\sqrt{\snrmax}}$ \textbf{(almost-empty association)}.
\end{theorem}
We prove this theorem in Section~\ref{sec:proof_main_ball}.

Theorem~\ref{thm:main_ball} states that \textit{all local minima have the same type of structure.} In particular, if we view a candidate solution $\bbeta = (\bbeta_i)_{i=1}^{k}$ as configuring the centers $\bbeta_i$'s to fit the true clusters, then any local minimum $\bbeta$ must be composed of only the following configurations:
\begin{itemize}
    \item [(i)] many-fit-one: multiple $\bbeta_{i}$'s are close to the same ground truth center;
    \item [(ii)] one-fit-many: one $\bbeta_{i}$ is close to the mean of several ground truth centers;
    \item [(iii)] almost-empty: a $\bbeta_i$ is far (relatively to other $\bbeta_j$'s) from any ground truth center, in the sense that the Voronoi set of $\bbeta_{i}$ is almost empty with a small measure.
\end{itemize}
Moreover, the configurations (i), (ii) and (iii) must involve disjoint sets of $\bbeta_j$s' and $\bbeta_s^*$s'.  For concrete examples, recall Figure~\ref{fig:illustration}: the ground truth solution in Panel~1a has 4 {one-fit-one} associations, whereas the spurious local minimum in Panel~1b consists of a {two-fit-one}, a {one-fit-two} and a {one-fit-one} association.

Put differently, Theorem~\ref{thm:main_ball} implies that if a solution $\bbeta$ involves any configuration other than the three above, then $\bbeta$ can be perturbed locally that strictly decreases its objective value. For example, the solutions in Panels~1c and~1d in Figure~\ref{fig:illustration} use two centers to fit three and two true clusters, respectively. The objective value can be decreased by moving these two centers \emph{away} towards different true clusters, as shown in Panels~2c and~2d. Our proof of Theorem~\ref{thm:main_ball} in fact makes use of this geometric idea in an analytical way, by studying the behavior of the objective function $G$ when $\bbeta$ is perturbed locally in certain directions.\\

It is instructive to specialize Theorem~\ref{thm:main_ball} to the limit case of a ``point model'', where $\rad \to 0$ or equivalently $\snrmax \to \infty$; that is, each ground truth cluster $s$ has a \emph{point mass} at $\bbeta_{s}^{*}$. In this case, the three possibilities guaranteed in the theorem reduce to: (i) several $\bbeta_i$'s are located exactly at one true cluster $\bbeta^*_s$ (many-fit-one); (ii) one center $\bbeta_i$ is located at the mean of several true $\bbeta_s^*$'s (one-fit-many); (iii) for all the other $\bbeta_i$'s, their Voronoi sets do not contain any true clusters. 

In the general setting with $r>0$, Theorem~\ref{thm:main_ball} guarantees that the above result for the point model still holds approximately, with an approximation error due to each true cluster having a mass spread around the true center. The three bounds in Theorem~\ref{thm:main_ball} control the approximation errors with respect to cases (i)--(iii) in the point model above. These error bounds all scale with $1/\sqrt{\snrmax}$, which becomes smaller if the SNR $\snrmax$ increases.

\paragraph{Tightness of the error bounds:}

The approximation errors above are unavoidable in general. We have already shown in Proposition~\ref{prop:existence} that there exists a local minimum $\bbeta = (\bbeta_1,\bbeta_2, \bbeta_3)$ where $\bbeta_1$ and $\bbeta_2$ are close but not exactly equal to $\bbeta^*_1$; see Figure~\ref{fig:local-min-3-intervals}. Here the mass of the first true cluster $\ball_1$ is equally split between the Voronoi sets of $\bbeta_1$ and $\bbeta_2$, each of which lies at the corresponding center of mass (cf.\ Lemma~\ref{lem:necessary}), leading to a nonzero approximation error in the many-fit-one association. In addition, in Example~\ref{ex:example2} in Appendix~\ref{sec:Additional Examples}, we demonstrate another local minimum with a non-zero approximation error in the one-fit-many association. We note that Theorem~\ref{thm:main_ball} gives upper bounds for these errors, and the bounds take the form $\mathrm{poly}(k)/\sqrt{\snrmax}$. 

In fact, the proof of Theorem~\ref{thm:main_ball} effectively establishes a family of bounds (see Theorem~\ref{prop:ball-family-bounds}) that provide a trade-off between the errors for the three types of associations. In particular, for each number $\lambda \in (0, \frac{1}{2k^2\rad})$, one can derive the bounds
\begin{align*}
&\text{(many-fit-one)}&\|\bbeta_i-\bbeta_s^*\| &\leq \frac{8k^2}{\lambda} , \quad\forall i\in S_a: |S^*_a|=1
\\
&\text{(one-fit-many)}&\bigg\|\bbeta_i-\frac{1}{|S_a^*|}\sum_{s\in S_a^*}\bbeta_s^* \bigg\| &\leq 11 \lambda k^2 \rad \deltamax,  \quad \forall i \in S_a: |S_a|=1\\
&\text{(almost-empty)}&\mP\big(\vor_{i}(\bbeta)\big) &\leq \lambda k\rad, \quad  \forall i \in S_0,
\end{align*}
where the partitions $[k] = \bigcup_{a=1}^m S_a^* = \bigcup_{a=0}^{m} S_a$ may depend on $\lambda$. Taking $\lambda = \frac{c}{\sqrt{\rad \deltamax} } = \frac{c}{\rad \sqrt{\snrmax}}$ gives the bounds in Theorem~\ref{thm:main_ball}. We are currently not sure though whether these bounds are tight in general. 
  
\paragraph{Necessity of the separation assumption:}

The result in Theorem~\ref{thm:main_ball} holds under the separation condition that the SNRs $\snrmax$ and $\snrmin$ are not too small. Such a separation condition is in general necessary. In Example~\ref{ex:example1} in Appendix~\ref{sec:Additional Examples}, we show that if $\snrmax$ is too small, then there exists a local minimum that fails to satisfy the structural properties in Theorem~\ref{thm:main_ball}. On the other hand, it is not clear to us whether the current form of the condition, $\snrmax \gtrsim k^4$, can be improved.

\paragraph{Generalization to over/under-parametrization:}

Inspecting the proof of Theorem~\ref{thm:main_ball}, one can see that the arguments therein do not actually require the number of fitted centers to be equal to that of true clusters. Therefore, our results can be extended to the setting where one fits $m$ centers to $k$ clusters with $m>k$ (over-parametrization) or $m<k$ (under-parametrization). As we discuss in greater details in Section~\ref{sec:discussion}, such a generalization has important algorithmic implications. 

\subsection{Gaussian Mixture Model}\label{sec:main_gaussian}

We next consider the Gaussian Mixture Model in Definition~\ref{def:GMM}. The main difference between this model and the Stochastic Ball Model is that the Gaussian distribution has an unbounded support and thus the tails of the mixture components overlap with each other. Nevertheless, much of the results for the Ball Model can be extended to the Gaussian case. For example, one can establish results analogous to Propositions~\ref{prop:truth_is_global_opt} and \ref{prop:existence} regarding the global minima and the existence of spurious local minima. Here we focus on establishing an analogue of Theorem~\ref{thm:main_ball}, which characterizes the structures of all local minima of the population $k$-means objective $G$.

Our main result is given in the following theorem.
\begin{theorem}[Local minimum structures, Gaussian Mixture Model]
\label{thm:main_gaussian}
Let $t>1$ be any number satisfying $\tail(t):=2\exp(-t^{2}\min(d,2k)/8)<\frac{1}{4}$. Under the Gaussian Mixture Model, assume that $\snrmax\geq16c^{2}k^{4}t$ and $\snrmin\geq 8c\sqrt{t}k^2\sqrt{\snrmax}+7k\tail(t)\snrmax$ for some constant $c\geq3$. If $\bbeta =(\bbeta_1,\ldots,\bbeta_k)\in\real^{d\times k}$ is a local minimum of $G$, then the ground truth centers and fitted centers can be partitioned as $[k] = \bigcup_{a=1}^m S_a^*$ and $[k] = \bigcup_{a=0}^{m} S_a$, respectively, such that for each $a\in[m]$, exactly one of the following holds:
\begin{itemize}
    \item \textbf{(many/one-fit-one association)} $|S_a| \ge 1$ and $S_a^* = \{s\}$ for some $s\in[k]$; moreover,
    \begin{align}
        \|\bbeta_i-\bbeta_s^*\|\leq \deltamax\left(7k^{2}\frac{c\sqrt{t}}{\sqrt{\snrmax}}+7k\tail(t)\right),\quad\forall i\in S_a.
    \end{align}
    \item \textbf{(one-fit-many association)} $S_a = \{i\} $ for some $i\in[k]$ and $|S_a^*|\ge 2$; moreover,
\begin{align}
    \bigg\| \bbeta_{i}-\frac{1}{|S_{a}^*|}\sum_{s\in S_{a}^*}\bbeta_{s}^{*} \bigg\| \leq\deltamax\left(9k^{2}\frac{c\sqrt{t}}{\sqrt{\snrmax}}+7k\tail(t)\right).
\end{align}
\end{itemize}
In addition, for each $i\in S_{0}$, we have $\mP\big(\vor_{i}(\bbeta)\big)\le\frac{ck\sqrt{t}}{\sqrt{\snrmax}}+\tail(t)$ \textbf{(almost-empty association)}.
\end{theorem}
We prove this theorem in Appendix~\ref{sec:proof_main_gaussian}.

Theorem~\ref{thm:main_gaussian} is qualitatively similar to Theorem~\ref{thm:main_ball}, showing that the local minima in GMM have a similar type of structure. The only difference is that the separation condition in Theorem~\ref{thm:main_gaussian} has an additional $t$ factor, and that the bounds for the three possibilities have an additional error term $\tail(t)$ that decays exponentially in $t^2$. The $\tail(t)$ term reflects the influence of the exponential tail of a Gaussian
distribution outside a ball of radius $t\sigma\sqrt{\min(d,2k)}$. In fact, the proof of Theorem~\ref{thm:main_gaussian} proceeds by effectively reducing GMM to the Stochastic Ball Model, treating the bulk of the Gaussian as a bounded distribution and the tail as additional errors. The choice of $t$ here controls the trade-off between the separation condition and the two terms in the error bounds. For a rough interpretation of the theorem, one could simply think of $t$ as a numerical constant large enough so that $\tail(t)$ is dominated by the other terms in the error bounds.

\section{Implications and Connections} \label{sec:discussion}

The theorems in the last section provide \emph{structural} results for the $k$-means objective. In this section, we discuss some \emph{algorithmic} implications of these results for solving the $k$-means problem and remark on their connections to the literature. 

\paragraph{Algorithmic Implications:}

Our result implies that one can find the global minimum of $k$ means as long as the characteristic many-fit-one association for local minima can be avoided (in this case one-fit-many association will also disappear as the number of true clusters and that of fitted centers are equal). This observation suggests that one should initialize a greedy clustering algorithm without putting fitted centers close to each other. Interestingly, several popular heuristics for $k$-means implement precisely this idea. For instance, the celebrated $k$-means$++$ algorithm~\citep{arthur2007k} is a version of the Lloyd's algorithm in which the initial centers are generated iteratively as follows: the first center
is selected uniformly from the data points; after selecting $m<k$ centers, one computes the minimal distance of each data point to these $m$ centers, and select a data point randomly as the $(m+1$)-th center with probability proportional to the above distance. This procedure therefore tends to pick $k$ initial centers that are far away from each other. Many other heuristics for $k$-means follow a similar spirit; see, e.g., the work in~\citet{ball1967promenade,astrahan1970speech,barakbah2005optimized,barakbah2009pillar}.

On the other hand, our structural results also highlight the inherent combinatorial difficulty of the problem. In particular, when the number of clusters grows, there is a growing number of possible configurations with many-fit-one and one-fit-many associations. It then becomes easier to get trapped into one of the corresponding local minima. This is consistent with the systematic empirical study in~\citet{franti2018k}, which observes that algorithms for $k$-means perform worse when there are more clusters.

\paragraph{Connection to Over-Parametrization:}

As mentioned after Theorem~\ref{thm:main_ball}, our results can be extended to the over-parametrization setting where $m>k$ centers are used to fit $k$ ground truth clusters. Over-parametrization appears to be a promising approach for avoiding local minima. In particular, when $m$ is much bigger than $k$, a random initial solution is likely to assign \emph{at least} one center to each true cluster. In this case, one-fit-many association would be avoided. Running a greedy algorithm from this initial solution, one would expect that it converges to a solution with only many-fit-one and almost-empty associations, which can then be pruned by inspecting the pairwise distances of the fitted centers and the sizes of their Voronoi sets.
The work in~\citet{dasgupta2007probabilistic} implements this idea in the context of over-parametrized EM. In particular, after EM converges, they remove fitted centers with low mixing weights (corresponding to almost-empty association) and combine fitted centers that are close to each other (corresponding to many-fit-one association). 

In fact, the extensive empirical study in~\citet{buhai2019benefits} shows that the above idea can be applied to other latent variable models, as these models often have a similar solution structure, i.e., some estimated latent variables having duplicated values or low prior probabilities.

\section{Preliminary Properties for the $k$-means Objective}
\label{sec:prelim}
In this section, we derive several preliminary results on the analytical properties of the population $k$-means objective function $G$ defined in~(\ref{eq:pop-kmean}), focusing on the Stochastic Ball Model. These properties are later used in the proofs of our main theorems.

When $\bbeta$ has pairwise distinct components (i.e., $\bbeta_i \neq \bbeta_j,\forall i\neq j\in[k]$), it is often convenient to rewrite the function $G$ using the notation of Voronoi sets:
\begin{equation}
G(\bbeta)=\sum_{i=1}^{k}\int_{\vor_{i}(\bbeta)}\|\x-\bbeta_{i}\|^{2}f(\x)\ddup\x.\label{eq:pop-kmean-alt}
\end{equation}
We can see that $G$ depends on $\bbeta$ in a complicated way through both $\|\x-\bbeta_{i}\|^{2}$ and $\vor_{i}(\bbeta)$. As shall become clear later, the dependence through the squared distance $\|\x-\bbeta_{i}\|^{2}$ determines the first-order condition for local optimality for $G$; on the other hand, understanding second-order conditions requires us to study the behaviors of the Voronoi sets $\vor_{i}(\bbeta)$ under small perturbation of $\bbeta$. To deal with this complication, our main strategy is to understand the directional behaviors of $G$ along certain (judiciously chosen) directions, and to construct upper bounds on $G$ that are easier to work with.

\subsection{Directional Behaviors of $G$ \label{sec:directional}}

Throughout the remainder of this section, we fix a candidate solution $\bbeta=(\bbeta_{1},\ldots,\bbeta_{k})\in\real^{d\times k}$. For a given direction $\bv=(\bv_{1},\bv_{2}\ldots,\bv_{k})\in\real^{d\times k}$, we are interested in how the objective $G(\bbeta)$ changes after we perturb $\bbeta$ to $\bbeta+t\bv$. Restricting the function $G$ to the direction $\bv$, we define the directional objective function
\[
H^ {\bv}(t):=G(\bbeta+t\bv).
\]
Note that $\bbeta$ is a local minimum of $G$ if and only if $\bm{0}$ is local minimum of $H^ {\bv}$ for all $\bv$.

The functions $G$ and $H^ {\bv}$ are not everywhere differentiable, as they involve the minimum of quadratic functions. However, they are differentiable almost everywhere. In particular, whenever $\bbeta$ has pairwise distinct components, the directional derivative $\frac{\ddup}{\ddup t}H^ {\bv}(0)$ is guaranteed to exist and admits a simple expression, as shown in the following lemma.

\begin{lemma}[Directional derivative]
\label{lem:differentiability}Suppose that $\bbeta$ satisfies $\bbeta_{i}\neq\bbeta_{j}$ whenever $i\neq j$. For any choice of direction $\bv$, the directional derivative $\frac{\ddup}{\ddup t}H^ {\bv}(0)$ exists and has the following analytic formula:
\[
\frac{\ddup}{\ddup t}H^ {\bv}(0)=-\sum_{i=1}^{k}\int_{\vor_{i}(\bbeta)}2\langle\bv_{i},\x-\bbeta_{i}\rangle f(\x)\ddup\x.
\]
\end{lemma}
The lemma follows from the Leibniz integral rule; we defer the proof to Appendix~\ref{sec:proof-differentiability}. Note that the above expression only involves differentiating the integrand in the expression~\eqref{eq:pop-kmean-alt}; the Voronoi sets $\vor_{i}(\bbeta)$ remain unchanged when only the first-order derivative is concerned.

\subsection{First-Order Necessary Condition for Local Optimality\label{sec:necessary}} Using the first-order derivative expression in Lemma~\ref{lem:differentiability}, we can derive a necessary condition for $\bbeta$ being a local minimum. In particular, the following lemma states that any local minimum (satisfying a non-degeneracy condition) must have pairwise distinct components, each of which must be the center of its Voronoi set.
\begin{lemma}[Local minimum must be Voronoi centers]
\label{lem:necessary}
Suppose that $\bbeta$ is a local minimum of $G$. Then for each pair $i\neq j$,  we must have $\bbeta_{i}\neq\bbeta_{j}$ whenever $\vor_{i}(\bbeta)\cup\vor_{j}(\bbeta)$ having a positive measure (with respect to $f$). Moreover, for each $\bbeta_{i}$ whose Voronoi set $\vor_{i}(\bbeta)$ has a positive measure, $\bbeta_{i}$ must be at the center of probability mass of the Voronoi set $\vor_{i}(\bbeta)$; that is, 
\begin{equation}
\bbeta_{i}=\frac{\int_{\vor_{i}(\bbeta)}\x f(\x)\ddup\x}{\int_{\vor_{i}(\bbeta)}f(\x)\ddup\x}.\label{eq:center-mass}
\end{equation}
\end{lemma}
The proof of Lemma~\ref{lem:necessary} is deferred to Appendix~\ref{sec:proof-necessary-condition}. The conclusion of the above lemma can be written equivalently in a more explicit way. In particular, let $m_{i,s}(\bbeta)$ and $\bc_{i,s}(\bbeta)$ denote the probability mass and center of mass of the set $\vor_{i}(\bbeta)$ with respect to $f_{s}$ respectively:
\begin{align*}
m_{i,s}(\bbeta)  :=\int_{\vor_{i}(\bbeta)}f_{s}(\x)\ddup\x
\qquad\textup{and}\qquad
\bc_{i,s}(\bbeta)  :=\frac{\int_{\vor_{i}(\bbeta)}\x f_{s}(\x)\ddup\x}{m_{i,s}(\bbeta)}.
\end{align*}
Then equation~(\ref{eq:center-mass}) can be rewritten as
\[
\bbeta_{i}=\frac{\sum_{s=1}^{k}m_{i,s}(\bbeta)\bc_{i,s}(\bbeta)}{\sum_{s=1}^{k}m_{i,s}(\bbeta)}.
\]

\subsection[for the toc]{Decomposition of $H^{\bv}$ \label{sec:decomposition}}

\begin{figure}[t]
\begin{centering}
\includegraphics[scale=0.28, clip, trim=0 60 0 50]{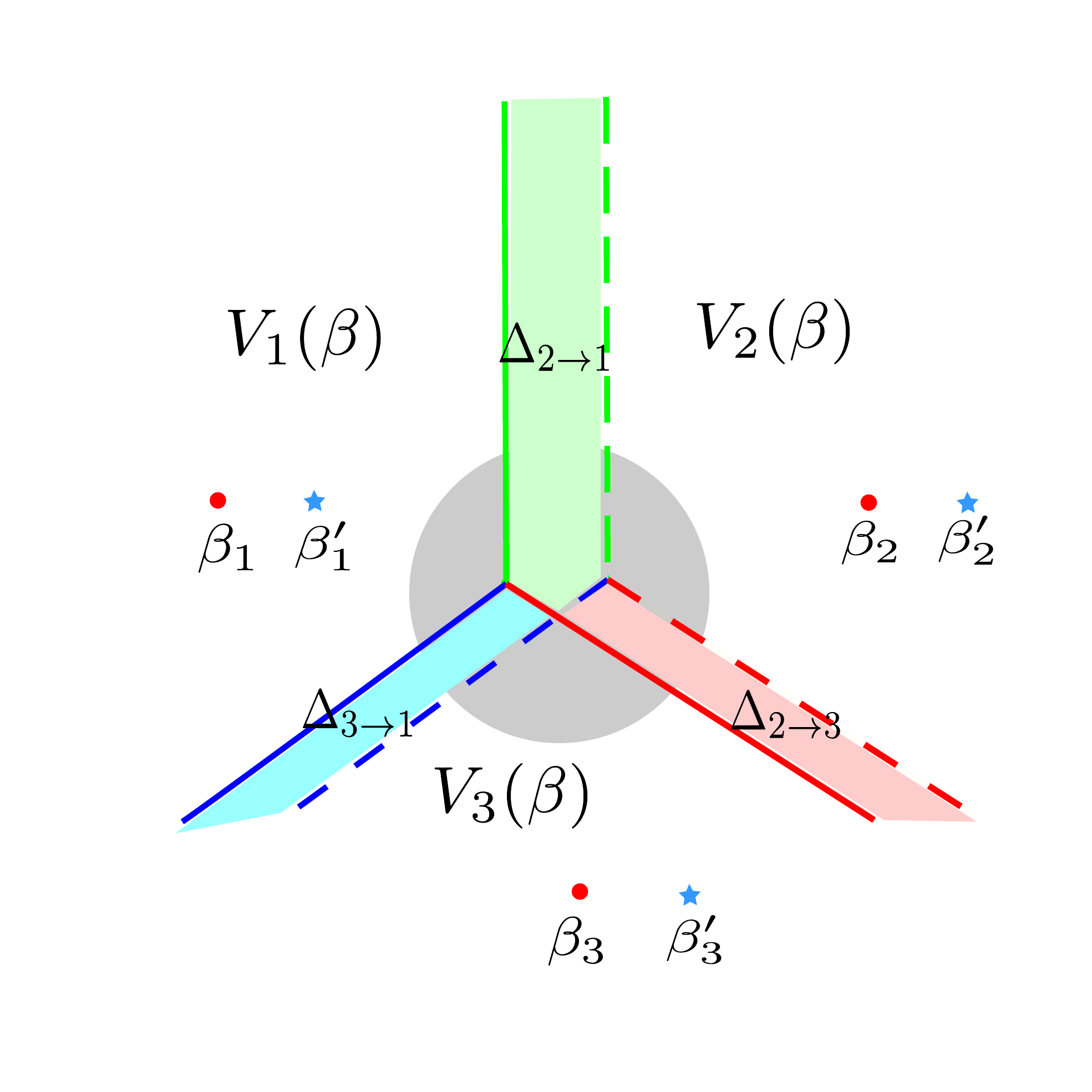}
\includegraphics[scale=0.28, clip, trim=0 60 0 50]{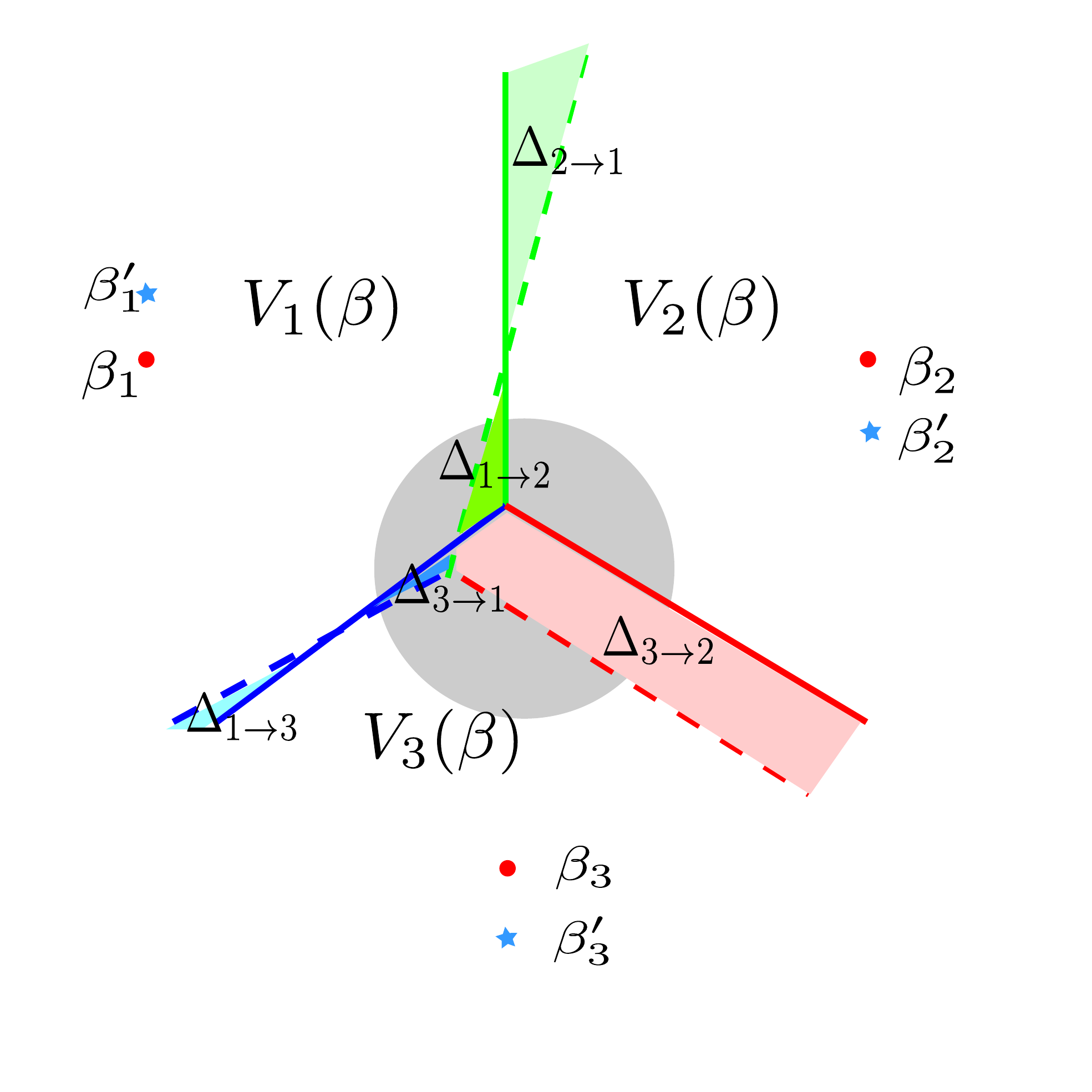}
\par\end{centering}
\caption{\label{fig:Illustration-change-assignment}Illustration of the set $\Delta_{i\to j}^{\bm{v}}(t).$ The red dots represent the original centers $ \{\bbeta_{i}\}$ and the blue stars represent the perturbed centers $ \{ \bbeta_{i}^{\prime} = \bbeta_{i}+ t\bm{v}_i \}$. The blue solid lines are the original Voronoi boundaries $\partial(\bm{\beta})$,  and the red dashed lines are the perturbed Voronoi boundaries $\partial(\bm{\beta}^{\prime})$. Left Panel: the moving direction $ \bv=( \bv_{1}, \bv_{2}, \bv_{3})$ satisfies $\bm{v}_{1}=\bm{v}_{2}=\bm{v}_{3}$, in which case the Voronoi boundaries are shifted parallelly by $t\bm{v}_{1}$. Right panel: the moving direction satisfies $\bm{v}_{1}=-\bm{v}_{2}=-\bm{v}_{3}$, in which case the boundary $\partial_{1,2}(\bm{\beta}^{\prime})$ rotates around the mid point $\frac{ \bbeta_{1}+ \bbeta_{2}}{2}$, $\partial_{1,3}(\bbeta^{\prime})$ rotates around the mid point $\frac{\bbeta_{1}+\bbeta_{3}}{2}$, and $\partial_{2,3}( \bbeta^{\prime})$ shifts parallelly in the direction of $\bm{v}_{2}$. Each colored region represents $\Delta_{i\to j}^{\bm{v}}(t)$, the set of points that change the association from the $i$-th center to the $j$-th center. }
\end{figure}

Lemmas~\ref{lem:differentiability} and \ref{lem:necessary} provide a first-order characterization of the local minima of $G$. For a more precise characterization, we need to account for the change in the Voronoi sets $\vor(\bbeta)$ and its boundaries $\partial(\bbeta)$ when perturbing $\bbeta$ to $\bbeta+t\bv$. With $t>0$ considered arbitrarily small, we make two observations:
\begin{enumerate}
\item The Voronoi set boundaries $\partial(\bbeta+t\bv)$ change continuously with respect to $t$.
\item When $\bbeta$ is perturbed by $t\bv$, the points swept by the boundaries $\partial(\bbeta+t\bv)$ change their association from one Voronoi set to another.
\end{enumerate}
Formally, for each pair $(i,j)\in[k]\times[k]$ we define the set
\[
\Delta_{i\to j}^ {\bv}(t):=\vor_{i}(\bbeta)\cap\vor_{j}(\bbeta+t\bv),
\]
which is the set of points that change association from the $i$-th fitted center to the $j$-th fitted center due to the perturbation $t\bv$. Being the intersection of two polyhedra, the set $\Delta_{i\to j}^ {\bv}(t)$ is a also polyhedron. An illustration of $\Delta_{i\to j}^ {\bv}(t)$ is provided in Figure~\ref{fig:Illustration-change-assignment}.

As previously shown in Lemma~\ref{lem:necessary}, any non-degenerate local minimum $\bbeta$ must have distinct components, so the corresponding Voronoi sets are also pairwise distinct. The same holds for the perturbed solution  $\bbeta+t\bv$ when $t$ is sufficiently small. In this case, we can decompose the directional objective function $H^ {\bv}$ as follows:
\begin{align}
H^ {\bv}(t) & =\sum_{i=1}^{k}\int_{\vor_{i}(\bbeta+t\bv)}\|\x-\bbeta_{i}-t\bv_{i}\|^{2}f(\x)\ddup\x\nonumber \\
 & =\underbrace{\sum_{i=1}^{k}\int_{\vor_{i}(\bbeta)}\|\x-\bbeta_{i}-t\bv_{i}\|^{2}f(\x)\ddup\x}_{U^ {\bv}(t)}\nonumber \\
 &\qquad+\underbrace{\sum_{(i,j):i\neq j}\int_{\Delta_{i\to j}^ {\bv}(t)}(\|\x-\bbeta_{j}-t\bv_{j}\|^{2}-\|\x-\bbeta_{i}-t\bv_{i}\|^{2})f(\x)\ddup\x}_{W^ {\bv}(t)}.
 \label{eq:H_decompose}
\end{align}
Here $U^ {\bv}(t)$ and $W^ {\bv}(t)$ correspond to the change in the objective value from two different sources. In particular,  $U^ {\bv}(t)$ is due to the change in the distance between the data points and the centers, and $W^ {\bv}(t)$ is due to the data points changing association with the Voronoi sets.
\begin{remark}
\label{rem:non-positive}Note that for each $\x\in\Delta_{i\to j}^ {\bv}(t)$, the integrand $\|\x-\bbeta_{j}-t\bv_{j}\|^{2}-\|\x-\bbeta_{i}-t\bv_{i}\|^{2}$ in the definition of $W^ {\bv}(t)$ is non-positive.
\end{remark}
\begin{proof}[Proof of equation (\ref{eq:H_decompose})]
When the Voronoi sets $\vor(\bbeta)$ are perturbed to $\vor(\bbeta+t\bv)$, each point $\x$ in $\real^{d}$ either remains associated with the $i$-th center for some $i$, or changes its association from the $i$-th center to the $j$-th center for some $j\neq i$. In the first case, assuming that $\x\in \vor_{i}(\bbeta) \cap \vor_{i}(\bbeta+t\bv)$, we see that the contribution from $\x$ to $H^ {\bv}(t)$ appears in $U^ {\bv}$. In the second case, assuming that $\x\in\Delta_{i\to j}^ {\bv}(t)=\vor_{i}(\bbeta) \cap \vor_{j}(\bbeta+t\bv)$, we can write the contribution from $\x$ as 
\[
\|\x-\bbeta_{j}-t\bv_{j}\|^{2}=\|\x-\bbeta_{i}-t\bv_{i}\|^{2}+(\|\x-\bbeta_{j}-t\bv_{j}\|^{2}-\|\x-\bbeta_{i}-t\bv_{i}\|^{2}),
\]
which appears in both $U^ {\bv}$ and $W^ {\bv}$.
\end{proof}

\subsection[for the toc]{\label{sec:core_upper_bound}Smooth Upper Bounds of $H^{\bv}$}

The expression~(\ref{eq:H_decompose}) for $H^ {\bv}$ is quite complicated. To understand the local minima of $H^ {\bv}$, we instead study a simpler, better-behaved upper bound function  of $H^ {\bv}$ that preserves the local minima and is amenable to calculus tools. In particular, we make use of the following lemma.
\begin{lemma}[Smooth upper bound]
\label{lem:core-upper-bound}
Suppose that $h,\widetilde{h}:\real\to\real$ are two continuous functions that satisfy $h\leq\widetilde{h}$ and $h(0)=\widetilde{h}(0)$. If $0$ is a local minimum of $h$, then $0$ is also a local minimum of $\widetilde{h}$; moreover, we have $\lim_{t\to0}\frac{\widetilde{h}(t)-\widetilde{h}(0)}{t}=0$ and $\lim_{t\to0}\frac{\widetilde{h}(t)-\widetilde{h}(0)}{t^{2}}\ge0$ whenever the limits exist.
\end{lemma}
\begin{proof}
Since $0$ is a local minimum of $h$, we have $\widetilde{h}(0)=h(0)\le h(t)\le\widetilde{h}(t)$ for all $t$ in a neighborhood of $0$, so $0$ is also a local minimum of $\widetilde{h}$. The first-order optimality condition for $0$ gives $\lim_{t\to0}\frac{\widetilde{h}(t)-\widetilde{h}(0)}{t}=\widetilde{h}'(0)=0$. Moreover, we have $\widetilde{h}(t)-\widetilde{h}(0)\ge0\implies\frac{\widetilde{h}(t)-\widetilde{h}(0)}{t^{2}}\ge0$ for all $t\neq0$ in a neighborhood of $0$, which implies that  $\lim_{t\to0}\frac{\widetilde{h}(t)-\widetilde{h}(0)}{t^{2}}\ge0$.
\end{proof}
With the above lemma, we can study the structure of each local minimum of $H^ {\bv}$ (and hence that of $G$) by exploiting the optimality conditions of a smooth upper bound of $H^ {\bv}$ that is tight at the minimum. Let us take a first step in constructing such an upper bound. In view of Remark~\ref{rem:non-positive}, we can obtain an upper bound of the function $W^ {\bv}$ defined in~(\ref{eq:H_decompose}) by only considering those  pairs $(i,j)$ for which $\vor_{i}(\bbeta)\sim\vor_{j}(\bbeta)$ are adjacent:
\begin{align*}
W^ {\bv}(t)
&\; \leq\sum_{(i,j):\vor_{i}(\bbeta)\sim\vor_{j}(\bbeta)}\int_{\Delta_{i\to j}^ {\bv}(t)}(\|\x-\bbeta_{j}-t\bv_{j}\|^{2}-\|\x-\bbeta_{i}-t\bv_{i}\|^{2})f(\x)\ddup\x\\
 &\; =\sum_{(i,j):\vor_{i}(\bbeta)\sim\vor_{j}(\bbeta)}\frac{1}{k}\sum_{s=1}^{k} \underbrace{\int_{\Delta_{i\to j}^ {\bv}(t)}(\|\x-\bbeta_{j}-t\bv_{j}\|^{2}-\|\x-\bbeta_{i}-t\bv_{i}\|^{2})f_{s}(\x)\ddup\x}_{W_{i\to j,s}^ {\bv}(t)}.
\end{align*}
Here the quantity $W_{i\to j,s}^ {\bv}$ represents the contribution from the points in the $s$-th true cluster that change association from the $i$-th center to the $j$-th center. Combining the above inequality with equation~(\ref{eq:H_decompose}), we obtain the following upper bound
\begin{equation}
H^ {\bv}(t)\leq U^ {\bv}(t)+\sum_{(i,j):\vor_{i}(\bbeta)\sim\vor_{j}(\bbeta)}\frac{1}{k}\sum_{s=1}^{k}W_{i\to j,s}^ {\bv}(t).\label{eq:H_upper_bound}
\end{equation}
In the proofs of our main theorems, we will build upon equation~\eqref{eq:H_upper_bound} to derive further smooth upper bounds of $H^ {\bv}$.

\section{Proof of Theorem~\ref{thm:main_ball}}
\label{sec:proof_main_ball}

In this section, we prove our main result under the Stochastic Ball Model (Definition~\ref{def:SBM}). Throughout the proof, let $\bbeta$ be a fixed local minimum of the $k$-means objective $G$. Note that $G$ is invariant under translation of the space and permutation of the true centers. Consequently, we may assume that $\bbeta_{1}^{*}=\bzero$ and $\max_{s\in[k]}\|\bbeta_{s}^{*}\|\leq\deltamax$ without loss of generality.

\paragraph{Notations:}

We use $V_{d}$ to denote the volume of a unit ball in $\real^{d}$ with respect to the Lebesgue measure. For a set $T\subset \real^{d}$, $\textup{int}(T)$ denotes its interior and $\revol(T)$ denotes its relative volume with respect to the Lebesgue measure on the affine hull of $T$, with the convention that $\revol(\emptyset) = 0$. For two vector $\bu,\bu'\in\real^{d}$, $\angle(\bu,\bu'):=\arccos(\frac{\bu^{\top}\bu'}{\|\bu\|\|\bu'\|})\in[0,\pi]$ is the angle between $\bu$ and $\bu'$. For each pair $i\neq j$, we use $\linspace_{i,j,s}(\bbeta)$ to denote the two-dimensional plane that contains $\bbeta_{i}$, $\bbeta_{j}$ and $\bbeta_{s}^{*}$ (if such a plane is not unique, we fix an arbitrary one). Since we are concerned with a fixed local minimum $\bbeta$, we sometimes suppress the dependency on $\bbeta$ and write, for example, $\vor_{i}\equiv\vor_{i}(\bbeta)$, $\partial_{i,j}\equiv\partial_{i,j}(\bbeta)$ and $\linspace_{i,j,s}\equiv\linspace_{i,j,s}(\bbeta)$.\\

To prove Theorem~\ref{thm:main_ball}, we in fact establish a more general result as given in Theorem~\ref{prop:ball-family-bounds}, which provides a family of bounds parametrized by $\lambda>0$.

\begin{theorem}[Family of bounds for ball model] 
\label{prop:ball-family-bounds}
Under the Stochastic Ball Model, let $\bbeta=(\bbeta_1,\ldots, \bbeta_k)$ be a local minimum of the $k$-means objective function $G$ defined in~\eqref{eq:pop-kmean} and $\lambda>0$ be an arbitrary fixed number. For each $i,j,s\in[k]$, let $\rho_s(\partial_{i,j}):=\frac{1}{V_d\rad^d}\revol(\partial_{i,j}\cap \ball_s)$. For each $i\in[k]$, define the sets
\begin{align*}
    T_i: = \big\{s\in [k]:\vor_i\cap \ball_s\neq\emptyset\big\}
    \quad\text{and}\quad
    \Wint_i:=\big\{s\in [k]:\bbeta_s^* \in \textup{int}(\vor_i)\big\} \subseteq T_i.
\end{align*}
Then the following is true for each $i\in[k]$.
\begin{enumerate}
    \item If  $\rho_s(\partial_{j,\ell})>\lambda$ for some $s\in T_i$ and some pair $(j,\ell)$, then 
    \begin{align*}
        \|\bbeta_i-\bbeta_s^*\|\leq \frac{k}{\lambda}+3r.
    \end{align*}
    \item 
    For each $s\in T_i$, if $\rho_s(\partial_{j,\ell})\leq \lambda$ for all pair $(j,\ell)$, then the following bounds hold:
    \begin{align*}
    \mP_s\big(\vor_i\big) \geq& 1-k^2\lambda\rad, \quad \text{if } s\in \Wint_i,\\
    \mP_s\big(\vor_i\big) \leq& k\lambda\rad, \,\quad\qquad \text{if } s\in T_i \setminus \Wint_i.
    \end{align*}
    Furthermore, if $\rho_s(\partial_{j,\ell})\leq \lambda$ for all $s\in T_i$ and all pair $(j,\ell)$, then:
    \begin{enumerate}
    \item When $|\Wint_i|=0$, we have 
    $$
        \mP (\vor_i)\leq k\lambda\rad.
    $$
    \item When $|\Wint_i|>0$, we have
    \begin{align*}
        \|\bbeta_i-\bb_i\| \leq& \frac{k\rad}{1-k^{2}\lambda\rad}+\frac{k\rad(k^{2}\lambda\rad)}{(1-k^{2}\lambda\rad)^{2}}+\frac{2k^{2}\lambda\rad}{1-k^{2}\lambda\rad}\deltamax,
    \end{align*}
    where  $\bb_i:= \frac{1}{|\Wint_i|}\sum_{s\in \Wint_i}\bbeta_s^*$.
    \end{enumerate}
\end{enumerate}
\end{theorem}

The proof of Theorem~\ref{prop:ball-family-bounds}, which lies in the core of our analysis, is given in Section~\ref{sec:proof-ball-family-bounds}. \\

We now derive our main Theorem~\ref{thm:main_ball} from Theorem~\ref{prop:ball-family-bounds}. Doing so involves several elementary though somewhat tedious steps. To this end, we fix $\lambda = \frac{c}{\sqrt{\rad \deltamax}} = \frac{c}{r\sqrt{\snrmax}}$. Recall the assumption in the main theorem that $\snrmin\geq 10ck^2\sqrt{\snrmax}$ and $\snrmax > 4c^2k^4$ for $c>3$. This assumption implies that $k^2\lambda \rad<0.5$. If $\rho_s(\partial_{i,j})>\lambda$, we say that a true cluster  $\ball_s$ encloses the Voronoi boundary $\partial_{i,j}$ with a large relative volume; otherwise, we say that $\ball_s$  encloses the Voronoi boundary $\partial_{i,j}$ with a small relative volume. 

We first state two simple implications of Theorem~\ref{prop:ball-family-bounds}. These observations are used frequently in the subsequent proof.
\begin{enumerate}[label=Observation~{\arabic*}.,ref={\arabic*},leftmargin=8em]
    \item \label{pro1} For each $i\in[k]$, there exists at most one $s\in T_i$ such that $\rho_s(\partial_{j,\ell})>\lambda$ for some pair $(j,\ell)$. In words, each Voronoi set $\vor_i$ can intersect at most one true cluster $\ball_s$ that encloses some Voronoi boundary with a large relative volume. 
    \item \label{pro2} 
    For each $s\in [k]$, if $\rho_s(\partial_{j,\ell})\leq \lambda$ for all pair $(j,\ell)$, then $\bbeta_s^*\in \vor_i$ implies that  $s \in \Wint_i$. In words, if all Voronoi boundaries enclosed by a true cluster $\ball_s$ have small relative volumes, then the center $\bbeta_s^*$ cannot itself lie on a Voronoi boundary.
\end{enumerate}
\begin{proof}
We prove these observations by contradiction. 

For Observation~\ref{pro1}, suppose otherwise that there exists $s \neq s^\prime \in T_i$ for which the statement holds. Part~1 of Theorem~\ref{prop:ball-family-bounds} ensures that $\|\bbeta_i-\bbeta_s^*\|\leq \frac{k}{\lambda}+3\rad$ and $\|\bbeta_i-\bbeta_{s^{\prime*}}\|\leq \frac{k}{\lambda}+3\rad$. Using the triangle inequality and the value for $\lambda$, we obtain that $\|\bbeta_s^* - \bbeta_{s^{\prime}}^*\|\leq \frac{2k}{\lambda}+6\rad \leq \deltamax\frac{8k}{c\sqrt{\snrmax}}$, which contradicts the assumption on $\snrmin$. 

For Observation~\ref{pro2}, suppose otherwise that $\bbeta_s^* \in \vor_i$ and $ s \notin  \Wint_i$ for some $i\in [k]$, which implies that $\bbeta^*_s$ lies on a Voronoi boundary and hence $\bbeta_s^*\not\in \Wint_j, \forall j\in[k]$. If $s \in T_j$, then Part~2 of Theorem~\ref{prop:ball-family-bounds} ensures that $\mP_s(\vor_j)\leq k\lambda\rad, \forall j\in [k]$; if $s\notin T_j$, then $\mP_s(\vor_j)=0$ by definition of $T_j$. Summing over $j\in[k]$, we obtain $1=\mP_s(\ball_s)=\sum_{j\in [k]} \mP_s(\vor_j)\leq k^2\lambda\rad<0.5$, thus a contradiction. 
\end{proof}

We now construct a partition $\bigcup_{a=0}^m S_a = [k]$ of the fitted centers and a partition $\bigcup_{a=1}^m S^*_a = [k]$ of the true centers that satisfy the conclusion of Theorem~\ref{thm:main_ball}. These partitions lead to an association between the fitted centers in $S_a$ and the true centers in $S^*_a$, for each $a=1,\ldots, m$. The construction proceeds in three steps.

\paragraph{Step 1 (almost-empty association):}
First consider the fitted centers indexed by the set
\begin{align*}
    S_0: =\left\{i\in [k]: \rho_s(\partial_{j,\ell})\leq \lambda,\forall (s,j,\ell)\in T_i\times [k]\times[k]; |\Wint_i|=0 \right\} .
\end{align*}
Part~2(a) of Theorem~\ref{prop:ball-family-bounds} ensures that  for all $i\in S_0$, we have $\mP(\vor_i)\leq k\lambda \rad = \frac{ck}{\sqrt{\snrmax}}$ as claimed in Theorem~\ref{thm:main_ball}.

\paragraph{Step 2 (many/one-fit-one association):} 

We next consider the fitted centers indexed by the set
\begin{align*}
    \mathcal{J}:=\left\{i\in [k]:|\Wint_i|\leq 1\right\} \setminus S_0.
\end{align*}
For each $i\in \mathcal{J}$, there are two complementary cases:
\begin{itemize}
    \item  $\rho_s(\partial_{j,k})\leq \lambda $ for all $(s,j,\ell)\in T_i\times[k]\times[k]$; that is, all true clusters that intersect $\vor_i$ only enclose Voronoi boundaries with a small relative volume. Since $i\notin S_0$, by definition of $S_0$ and $\mathcal{J}$ we must have $|\Wint_i|=1$; say $\Wint_i = \{s\}$. Applying Part~2(b) of Theorem~\ref{prop:ball-family-bounds}, we have
        \begin{align}
        \|\bbeta_i-\bbeta^*_s\| = \|\bbeta_i-\bb_i\| \leq\frac{k\rad}{1-k^{2}\lambda\rad}+\frac{k\rad(k^{2}\lambda\rad)}{(1-k^{2}\lambda\rad)^{2}}+\frac{2k^{2}\lambda\rad}{1-k^{2}\lambda\rad}\deltamax\leq \deltamax\frac{8ck^2}{\sqrt{\snrmax}}, \label{eq:bound-to-mean-of-centers}
        \end{align}
    where the last step holds due to $k^2\lambda\rad<\frac{1}{2}$ and our separation assumption on $\snrmax$. 
    \item $\rho_s(\partial_{j,\ell})> \lambda$ for some $(s,j,\ell)\in T_i\times[k]\times[k]$;  that is, there exists some ground truth cluster $\ball_s$ that encloses a Voronoi boundary with a large relative volume. Applying Part~1 of Theorem~\ref{prop:ball-family-bounds} and plugging the value of $\lambda$, we obtain that $\|\bbeta_i-\bbeta_s^*\|\leq \frac{k}{\lambda}+3\rad\leq \deltamax\frac{4k}{c\sqrt{\snrmax}}\leq \deltamax\frac{8ck^2}{\sqrt{\snrmax}}$.
\end{itemize}
    In both cases, we have $\|\bbeta_i-\bbeta_s^*\| \leq \deltamax\frac{8ck^2}{\sqrt{\snrmax}}$ as claimed in Theorem~\ref{thm:main_ball}. For each distinct $s\in[k]$ that appears in the above arguments, let $S^*_a = \{s\}$ and let the corresponding $S_a$ index those $\bbeta_i$'s for which either of the two cases holds. It is clear that the sets $\{S_a\}$ constructed here are disjoint. Indeed, for each $i$ the above two cases are exclusive, where in the first case $\Wint_i$ contains $s$ and only $s$, and in the second case above the index $s$ is unique by Observation~\ref{pro1}.

\paragraph{Step 3 (one-fit-many association):} 
We are left with the fitted centers indexed by the set
\begin{align*}
    \mathcal{K}:=\left\{i\in[k]: |\Wint_i|\geq 2 \right\} = [k]\setminus\big(S_0 \cup \mathcal{J}\big).
\end{align*}
Similarly to before, for each $i\in \mathcal{K}$, there are two complementary cases:
\begin{itemize}
    \item  $\rho_s(\partial_{j,\ell})\leq \lambda $ for all $ (s,j,\ell) \in T_i \times [k] \times [k]$. Applying Part~2(b) of Theorem~\ref{prop:ball-family-bounds} and following the same steps as in equation~\eqref{eq:bound-to-mean-of-centers}, we obtain that $\|\bbeta_i-\bb_i\|\leq \deltamax\frac{8ck^2}{\sqrt{\snrmax}}$. In this case, we let $S_a=\{i\}$ and $S^*_a = \Wint_i$. Note that $|S^*_a| = |\Wint_i|\ge 2 $ by definition of $\mathcal{K}$.
    \item $\rho_s(\partial_{j,\ell})> \lambda$ for some $ (s,j,\ell) \in T_i \times [k] \times [k]$.  In this case, applying Part~1 of Theorem~\ref{prop:ball-family-bounds} would show that $\bbeta_i$ is close to $\bbeta^*_s$. In fact, we can establish a stronger result showing that $\bbeta_i$ is close to the mean of all the true centers contained in its Voronoi set, regardless of whether we include or exclude  $\bbeta_s^*$. This is the content of the following lemma, which is proved in Section~\ref{sec:proof-prox-center}.
    \begin{lemma}[Proximity to mean of true centers] 
    \label{lem:prox-center} Under the assumption of Theorem~\ref{thm:main_ball}, let $\bbeta$ be a local minimum of $G$. The following is true for each $i\in[k]$. If $\rho_s(\partial_{j,\ell})>\lambda=\frac{c}{\sqrt{\rad\deltamax}}$ for some $(s,j,\ell) \in T_i \times [k]\times [k]$ and  $|\Wint_i\setminus\{s\}|\geq 1$, then we have the bounds
    \begin{align*}
        \|\bbeta_i-\bb_{i}^{-}\|\leq  \deltamax\frac{11ck^2}{\sqrt{\snrmax}}
        \quad\text{and}\quad
        \|\bbeta_i-\bb_{i}^{+}\|\leq  \deltamax\frac{11ck^2}{\sqrt{\snrmax}},
    \end{align*}
    where $\bb_{i}^{-}:=\frac{1}{|\Wint_i\setminus\{s\}|}\sum_{s'\in \Wint_i\setminus\{s\}}\bbeta_{s'}^*$ and  $\bb_{i}^{+}:=\frac{1}{|\Wint_i\cup\{s\}|}\sum_{s'\in \Wint_i\cup\{s\}}\bbeta_{s'}^*$; moreover, we have $|\Wint_i \setminus \left\{s\right\}|\geq 2 $.
    \end{lemma}
     
    In this case, we let $S_a = \{i\}$; also let $S^*_a = \Wint_i \setminus \{s\} $ if the index $s$ has appeared in the sets  $\{S^*_{a'}\}$  constructed previously in Step~2 or in this step, and set $S^*_a = \Wint_i \cup \{s\} $ otherwise. Note that $|S^*_a| \ge 2$ by Lemma~\ref{lem:prox-center}. As shall become clear momentarily, the flexibility allowed by Lemma~\ref{lem:prox-center} is important for ensuring that $\{S_a^*\}$ indeed partitions $[k]$.
\end{itemize}
In both cases above, we have the bound $\big\| \bbeta_i - \frac{1}{|S_a^*|}\sum_{s'\in S^*_a} \bbeta_{s'}^* \big\| \le \deltamax\frac{11ck^2}{\sqrt{\snrmax}}$ as claimed in Theorem~\ref{thm:main_ball}. It is clear that the sets $\{S_a^*\}$ constructed in this step are disjoint from each other and from those constructed in Step~2, because the sets $\{\Wint_i\}$ are disjoint as each true center can be in the interior of only one Voronoi set.

\paragraph{Summary:}

The above procedure constructs a collection of sets $\{S_a\}_{a=0}^{m}$ and $\{S^*_a\}_{a=1}^{m}$, which index the fitted and true centers, respectively, and satisfy the bounds in Theorem~\ref{thm:main_ball}. The sets $\{S_a\}$ indeed form a partition of $[k]$, as we have $\bigcup_{a=0}^{m} S_a = S_0 \cup \mathcal{J} \cup \mathcal{K} = [k]$ by definition, and $S_a \cap S_b = \emptyset, \forall a\neq b$ by construction and the fact that $S_0, \mathcal{J}, \mathcal{K}$ are disjoint. For the sets $\{S^*_a\}$, we have argued in the construction above that they are disjoint. On the other hand, each true center $\bbeta^*_s$ must belong to at least one Voronoi set $\vor_i$, in which case we have $s\in T_i$. Consider two complementary cases: (i) If $\exists(j,\ell): \rho_s(\partial_{j,\ell}) > \lambda$, then $s$ must be covered in the second case of either Step~2 or Step~3. (ii) If $\forall(j,\ell): \rho_s(\partial_{j,\ell}) \le \lambda $, then Observation~\ref{pro2} ensures that $s\in \Wint_i \subseteq T_i$. In this case, if all other $s' \in T_i$ satisfies $\forall(j,\ell): \rho_{s'}(\partial_{j,\ell}) \le \lambda$ as well, then $s$ is covered in the first case of either Step~2 or Step~3. Otherwise, if there exists another $s' \in T_i $ satisfying $\exists (j,\ell): \rho_{s'}(\partial_{j,\ell}) > \lambda$, then $s$ is covered in the second case of Step~3 (with the role of $s$ and $s'$ exchanged therein) as $s\in \Wint_i\subseteq \Wint_i\setminus\{s'\}$. We conclude that the collection of sets $\{S^*_a\}$ covers all $s\in[k]$ and hence is indeed a partition of $[k]$. This completes the proof of Theorem~\ref{thm:main_ball}.

\subsection{Proof of Theorem~\ref{prop:ball-family-bounds}}
\label{sec:proof-ball-family-bounds}

In this section, we prove Theorem~\ref{prop:ball-family-bounds}, which shows that a local minimum $\bbeta$ satisfies a family of bounds that imply our main Theorem~\ref{thm:main_ball}.

\paragraph{Proof strategy:}

To derive structural properties of the local minimum $\bbeta$, we exploit the fact that $t=0$ is a local minimum of the directional objective $H^{\bv}(t)$ (or a smooth upper bound thereof) for any perturbation direction $\bv$; see Lemma~\ref{lem:core-upper-bound} and the discussion in Section~\ref{sec:prelim}.
The expression~\eqref{eq:H_decompose} of $H^{\bv}$ involves the set $\Delta_{i\to j}^{\bv}$ of points that switch from one Voronoi set from another when $\bbeta$ is perturbed to $\bbeta + t\bv$. These sets are quite complicated for a general direction $\bv$. Our main idea is to focus on a special class of directions satisfying
\begin{equation}
\|\bv_{i}\|=1,\forall i\in[k];\qquad\bv_{i}=\bv_{j}\text{ or }\bv_{i}=-\bv_{j},\forall i\neq j.\label{eq:direction_choice}
\end{equation}
That is, we perturb the $\bbeta_{i}$'s along the same or opposite
directions. For these choices of $\bv$, the Voronoi boundary $\partial_{i,j}(\bbeta+t\bv)$
behaves in a simple way. In particular, when $\bv_{i}=\bv_{j}$, the
boundary $\partial_{i,j}(\bbeta+t\bv)$ translates along the direction
of $\bv_{i}$; when $\bv_{i}=-\bv_{j}$, the boundary rotates around
the mid point $\frac{\bbeta_{i}+\bbeta_{j}}{2}$. See Figure~\ref{fig:Illustration-change-assignment}
for an illustration. Using this fact, we can construct simple, tractable
upper bounds of $H^{\bv}$, from which we can deduce the structural properties of the local minimum $\bbeta$.

\paragraph{Key quantities:}

Our analysis involves several key quantities related to the Voronoi sets of $\bbeta$ and their boundaries. In particular, for each pair $i\neq j$ whose associated Voronoi sets are adjacent, i.e., $\vor_i(\bbeta)\sim \vor_j(\bbeta)$, we introduce the following four quantities. 
\begin{figure}[t]
\centering
\includegraphics[scale=0.5, clip, trim=0 60 0 20]{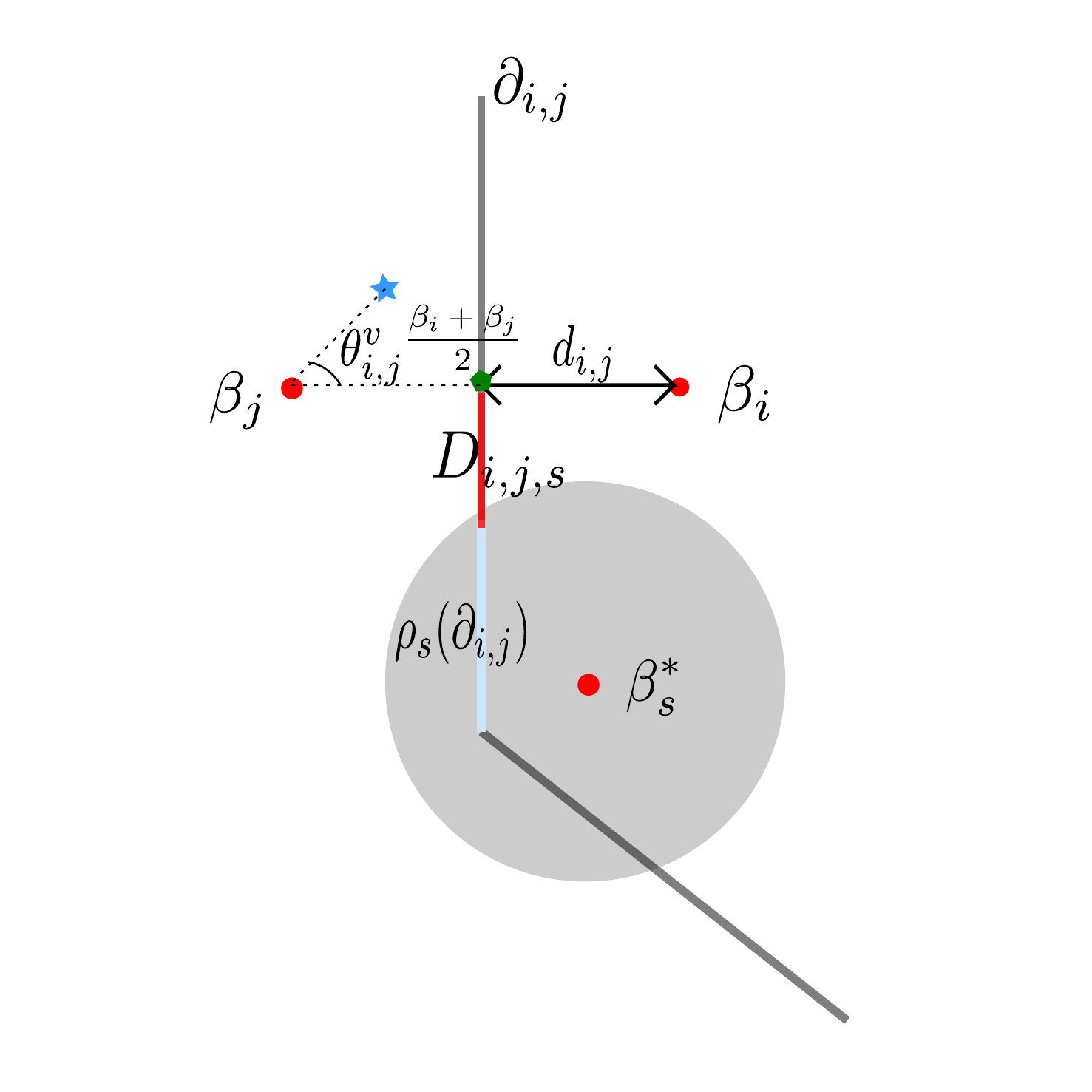}
\caption{\label{fig:Explanation-of-notation}Illustration of the quantities
$d_{i,j}$, $\theta_{i,j}^{\bv}$, $D_{i,j,s}$ and  $\rho_{s}(\partial_{i,j})$. The red dots represent $\bbeta_{i}$ and $\bbeta_{j}$, and the blue star represents the perturbed solution $\bbeta_{j}+t\bm{v}_{j}$.
Here $d_{i,j}$ is the distance between $ \bbeta_{i}$ and the mid point
$\frac{\bbeta_{i}+ \bbeta_{j}}{2}$ (represented by a green dot); $D_{i,j,s}$, which is represented by the red line segment, is the distance between $\frac{\bbeta_{i}+ \bbeta_{j}}{2}$ and the ball $\mB_{s}$ when computed within the hyperplane
containing the Voronoi boundary $\partial_{i,j}$; $\theta_{i,j}^{\bv}$
is the angle between the perturbation direction $\bv_{j}$ and the direction of $\bbeta_{i}- \bbeta_{j}$; $\rho_{s}(\partial_{i,j})$ is the normalized relative volume of the set $\partial_{i,j}\cap \ball_{s}$, which is represented by the blue line segment inside the ball.}
\end{figure}
\begin{enumerate}[label=(\alph*)]
    \item Denote by $d_{i,j}:=\frac{1}{2}\|\bbeta_{i}-\bbeta_{j}\|$ the distance between $\bbeta_{i}$ (or $\bbeta_{j}$) and the Voronoi boundary $\partial_{i,j} \equiv \partial_{i,j}(\bbeta)$. 
    \item Denote by $\theta_{i,j}^{\bv}=\angle(\bv_j,\bbeta_i-\bbeta_j)$ the (unsigned) angle of the perturbation direction $\bv_j$ with respect to $\bbeta_i-\bbeta_j$. 
    \item Define $D_{i,j,s}:=\dist\big(\frac{\bbeta_{i}+\bbeta_{j}}{2},\ball_{s}\cap \partial_{i,j} \big)$, with the convention that $D_{i,j,s}=1$ if $\ball_{s}\cap \partial_{i,j} = \emptyset$.  Here $D_{i,j,s}$ is the distance between the mid-point $\frac{\bbeta_{i}+\bbeta_{j}}{2}$ and $s$-th ground truth cluster $\ball_{s}$, where the distance is computed within the hyperplane containing the Voronoi boundary $\partial_{i,j}$.
    \item Recall the quantity $
    \rho_{s}(\partial_{i,j}):=\frac{1}{\rad^{d} V_{d}}\revol(\partial_{i,j}\cap\ball_{s})
    $
    defined in the statement of Theorem~\ref{prop:ball-family-bounds}. Note that $\rho_{s}(\partial_{i,j})$ is the relative volume of the intersection of the Voronoi boundary $\partial_{i,j}$ and the $s$-th ground truth cluster $\ball_{s}$, normalized by the volume of $\ball_{s}$.
\end{enumerate}
An illustration of these quantities is given in Figure~\ref{fig:Explanation-of-notation}.\\

We are now ready to prove Theorem~\ref{prop:ball-family-bounds}. We begin with the upper bound of $H^{\bv}$ given in equation~(\ref{eq:H_upper_bound}), restated in an equivalent way below:
\begin{align}
    H^{\bv}(t)
    \le U^{\bv}(t)+\frac{1}{2k}\sum_{(i,j):\vor_{i}\sim\vor_{j}} \sum_{s=1}^{k} \Big(W_{i\to j,s}^{\bv}(t)+W_{j\to i,s}^{\bv}(t)\Big). \label{eq:H-upper-bound-in-proof}
\end{align}
Note that the equality holds at $t=0$, since by definition $U^{\bv}(0)=H^{\bv}(0)$ and $W_{i\to j,s}^{\bv}(0)=0,\forall\; i,j,s$. Moreover, when $\bbeta$ is a local minimum of $G$, a quick calculation using Lemma~\ref{lem:necessary} shows that 
\begin{equation}
\lim_{t\to0}\frac{1}{t}\big(U^{\bv}(t)-U^{\bv}(0)\big)=0\qquad\text{and}\qquad\lim_{t\to0}\frac{1}{t^{2}}\big(U^{\bv}(t)-U^{\bv}(0)\big)=1.\label{eq:U_properties}
\end{equation}

Under the specific choice of the direction $\bv$ in equation~(\ref{eq:direction_choice}), the function $W_{i\to j}^{\bv}+W_{j \to i}^{\bv}$ can be further upper bounded, in a small neighborhood of $0$, by a smooth function with nice analytical properties. In particular, when the directions $\bv_{i}=\bv_{j}$ are the same, such an upper bound $\widetilde{W}_{i,j,s}^{\bv}$ is given in the following proposition, which is proved in Section~\ref{sec:same-direction}.
\begin{proposition}[Upper bound, same direction]
\label{prop:same-direction}
 Let $\bbeta$ be a local minimum of $G$ and $\bv$ satisfy $\|\bv_{i}\|=1,\forall i\in[k]$. If $\bv_{i}=\bv_{j}$, then $W_{i\to j,s}^{\bv}+W_{j\to i,s}^{\bv}$ is upper bounded in a neighborhood of 0 by some smooth function $\widetilde{W}_{i,j,s}^{\bv}$  satisfying the following
properties: 
\begin{enumerate}
\item $\widetilde{W}_{i,j,s}^{\bv}(0)=0;$
\item $\frac{\ddup}{\ddup t}\widetilde{W}_{i, j,s}^{\bv}(t)\mid_{t=0}=0$;
\item $\lim_{t\to0}\frac{1}{t^{2}}\widetilde{W}_{i,j,s}^{\bv}(t)\mid_{t=0}=-2\cos^{2}(\theta_{i, j}^{\bv})d_{i,j}\cdot\rho_{s}(\partial_{i,j})$.
\end{enumerate}
\end{proposition}
When the directions $\bv_{i}=-\bv_{j}$ are opposite and $d\ge2$, an upper bound $\widehat{W}_{i,j,s}^{\bv}$ is given in the following proposition, which is proved in Section~\ref{sec:opposite-direction}.
\begin{proposition}[Upper bound, opposite direction]
\label{prop:opposite-direction}%
Let $\bbeta$ be a local minimum of $G$ and $\bv$ satisfy $\|\bv_{i}\|=1,\forall i\in[k]$. If $\bv_{i}=-\bv_{j}$ and $\bv_{i},\bv_{j} \in\linspace_{i,j,s}$, then $W_{i\to j,s}^{\bv} + W_{j \to i,s}^{\bv}$ is upper bounded in a neighborhood of $0$ by some smooth function $\widehat{W}_{i,j,s}^{\bv}$ satisfying the following properties: 
\begin{enumerate}
\item $\widehat{W}_{i,j,s}^{\bv}(0)=0;$
\item $\frac{\ddup}{\ddup t}\widehat{W}_{i, j,s}^{\bv}(t)\mid_{t=0}=0$;
\item $\lim_{t\to0}\frac{1}{t^{2}}\widehat{W}_{i, j,s}^{\bv}(t)\mid_{t=0}=-2\frac{D_{i,j,s}^{2}}{d_{i,j}}\sin^{2}(\theta_{i,j}^{\bv})\cdot \rho_{s}(\partial_{i,j})$.
\end{enumerate}
\end{proposition}
Also note that $W_{i\to j,s}^{\bv} \le 0, \forall (i,j,s) $ by Remark~\ref{rem:non-positive}. For each $s\in [k]$ and each un-ordered pair $(i,j)\in[k]\times [k]$ satisfying $\vor_{i}\sim\vor_{j}$, combining equation~\eqref{eq:H-upper-bound-in-proof} and the last two propositions give the following smooth upper bound $\widetilde{H}_{i,j}^{\bv}$ of $H^{\bv}$:
\begin{align*}
 H^{\bv}(t) &\le\widetilde{H}_{i,j}^{\bv}(t) \\
 &:= U^{\bv}(t)
 +\frac{1}{k}\widetilde{W}_{i,j,s}^{\bv}(t)\indic_{\{\bv_{i}=\bv_{j}\}} +\frac{1}{k}\widehat{W}_{i,j,s}^{\bv}(t)\indic_{\{\bv_{i}=-\bv_{j}\in\linspace_{i,j,s}\}},
\end{align*}
which is valid in a neighborhood of $0$ and satisfies  $\widetilde{H}_{i,j}^{\bv}(0) = H^{\bv}(0)$. Since $t=0$ is a local minimum of $H^{\bv}$, Lemma~\ref{lem:core-upper-bound} ensures that 
\begin{align}
\lim_{t\to 0}\frac{1}{t^{2}}\left[\widetilde{H}_{i,j}^{\bv}(t)-\widetilde{H}_{i,j}^{\bv}(0)\right]  
 &\geq 0. \label{eq:second-derivative-upper-bound}
\end{align}
Moreover, by combining equation~(\ref{eq:U_properties}), Proposition~\ref{prop:same-direction} and Proposition~\ref{prop:opposite-direction}, we obtain that
\begin{equation}
\label{eq:second-derivative-equality}
\begin{aligned}
&\lim_{t\to 0}\frac{1}{t^{2}}\left[\widetilde{H}_{i,j}^{\bv}(t)-\widetilde{H}_{i,j}^{\bv}(0)\right]\\
&=1-\frac{2}{k}\cos^{2}(\theta_{i,j}^{\bv})d_{i,j}\cdot \rho_{s}(\partial_{i,j})\indic_{\{\bv_{i}=\bv_{j}\}} -\frac{2}{k}\frac{D_{i,j,s}^{2}}{d_{i,j}}\sin^{2}(\theta_{i,j}^{\bv})\rho_{s}(\partial_{i,j}) \indic_{\{\bv_{i}=-\bv_{j}\in\linspace_{i,j,s}\} }. 
\end{aligned}
\end{equation}

Since equation~(\ref{eq:second-derivative-equality}) holds for any choice of $\bv$ satisfying the  condition~(\ref{eq:direction_choice}), we may choose $\bv$ judiciously to simplify the right hand side of~(\ref{eq:second-derivative-equality}). By doing so we can show that for each $s\in[k]$ and each pair $i\neq j\in[k]$ satisfying $\vor_{i}\sim\vor_{j}$, there hold the inequalities
\begin{align}
d_{i,j}\cdot\rho_{s}(\partial_{i,j}) & \leq\frac{k}{2}\qquad\text{and}\label{eq:relation-1}\\
\frac{D_{i,j,s}^{2}}{d_{i,j}}\cdot\rho_{s}(\partial_{i,j}) & \leq\frac{k}{2},\label{eq:relation-2}
\end{align}
where the second inequality is valid when $d\geq2$. To prove the inequality~(\ref{eq:relation-1}), suppose otherwise that $d_{i,j}\cdot\rho_{s}(\partial_{i,j})>\frac{k}{2}$ for some $(i,j,s)$. We can choose the directions $\bv_{i}=\bv_{j}=\frac{\bbeta_{i}-\bbeta_{j}}{\|\bbeta_{i}-\bbeta_{j}\| }$, which satisfies $\theta_{i,j}^{\bv}=0$. Combining with equation~(\ref{eq:second-derivative-equality}) gives
\[
\lim_{t\to0}\frac{1}{t^{2}}\left(\widetilde{H}_{i,j}^{\bv}(t)-\widetilde{H}^{\bv}(0)\right)<0,
\]
which contradicts the inequality~(\ref{eq:second-derivative-upper-bound}). Similarly, to prove the inequality~(\ref{eq:relation-2}), suppose otherwise that $\frac{D_{i,j,s}^{2}}{d_{i,j}}\cdot\rho_{s}(\partial_{i,j})>\frac{k}{2}$ for some $(i,j,s)$ when $d\ge 2$. We can choose $\bv_{i}$ and $\bv_{j}$ to be two unit vectors in the two-dimensional plane $\linspace_{i,j,s}$ such that $\bv_{i}=-\bv_{j}$ and $\bv_{i}\perp(\bbeta_{j}-\bbeta_{i})$, which satisfies $\theta_{i,j}^{\bv}=\frac{\pi}{2}$. Combining with equation~(\ref{eq:second-derivative-equality}) gives
\[
\lim_{t\to0}\frac{1}{t^{2}}\left(\widetilde{H}_{i,j}^{\bv}(t)-\widetilde{H}^{\bv}(0)\right)<0,
\]
which again contradicts the inequality~(\ref{eq:second-derivative-upper-bound}).\\

In the remaining of the proof, fix an index $i\in[k]$ and a number $\lambda >0$. We shall use equations~\eqref{eq:relation-1} and \eqref{eq:relation-2} to derive the structural properties of $\bbeta_{i}$ and its Voronoi set $\vor_{i}$. To this end, we consider two complementary cases that correspond to Part~1 and Part~2 of Theorem~\ref{prop:ball-family-bounds}. Recall that $T_i :=\{s\in[k]:\vor_{i}\cap\ball_{s}\neq\emptyset\}$.

\subsubsection*{Case 1: there exists some $(s,j,\ell)\in T_i \times [k] \times [k]$ such that  $\rho_{s}(\partial_{j,\ell})>\lambda$.}

In this case, the Voronoi set $\vor_i$ intersects a true cluster $\ball_s$ that encloses some Voronoi boundary $\partial_{j,\ell}$ with a large relative volume. Note that this case corresponds to Part~1 of Theorem~\ref{prop:ball-family-bounds}.

Under the case condition, the inequality~(\ref{eq:relation-1}) implies that $d_{j,\ell}\leq \frac{k}{2\rho_{s}(\partial_{j,\ell})} \leq\frac{k}{2\lambda}$; equivalently,
\[
\Big\|\frac{\bbeta_{j}+\bbeta_{\ell}}{2}-\bbeta_{j}\Big\|
=\Big\|\frac{\bbeta_{j}+\bbeta_{\ell}}{2}-\bbeta_{\ell}\Big\|
\leq\frac{k}{2\lambda}.
\]
We consider the one-dimensional and high-dimensional cases separately. When $d=1$, the case condition $\rho_{s}(\partial_{j,\ell})>\lambda$ further implies that $\frac{\beta_{j}+\beta_{\ell}}{2}\in\mB_{s}$ and hence $|\frac{\beta_{j}+\beta_{\ell}}{2}-\beta_{s}^{*}|\leq\rad$. 
It follows that  $|\beta_{j}-\beta_{s}^{*}| \leq \big|\beta_{j}-\frac{\beta_{j}+\beta_{\ell}}{2}\big| + \big|\frac{\beta_{j}+\beta_{\ell}}{2}-\beta_{s}^{*}\big| \leq\frac{k}{2\lambda}+\rad$. 
When $d\geq2$,  we have $D_{j,\ell,s}\leq\frac{k}{2\lambda}$ by multiplying the inequalities~(\ref{eq:relation-1}) and~(\ref{eq:relation-2}). Let $\z\in\ball_{s}\cap \partial_{j,\ell}$ be the point that attains $\big\|\frac{\bbeta_{j}+\bbeta_{\ell}}{2}-\z\big\|=\dist\big(\frac{\bbeta_{j}+\bbeta_{\ell}}{2},\ball_{s}\cap\partial_{j,\ell}\big)=D_{j,\ell,s}.$
 It follows that 
\begin{align*}
\|\bbeta_{j}-\bbeta_{s}^{*}\| & \leq\Big\|\bbeta_{j}-\frac{\bbeta_{j}+\bbeta_{\ell}}{2}\Big\| +\Big\|\frac{\bbeta_{j}+\bbeta_{\ell}}{2}-\z\Big\|+\|\z-\bbeta_{s}^{*}\|\nonumber \\
 & \leq \frac{k}{2\lambda} + \frac{k}{2\lambda} +\rad
 =\frac{k}{\lambda}+\rad.\label{eq:bound-dist-cluster-center}
\end{align*}

In either case of $d$, we have the bound $\|\bbeta_{j}-\bbeta_{s}^{*}\| \le \frac{k}{\lambda}+\rad$. Since $s\in T_i$, there exist a point $\x\in \ball_s\cap \vor_{i}$. It follows that 
\begin{align*}
    \|\bbeta_i-\bbeta_s^*\|\leq & \|\bbeta_i-\x\|+ \|\x-\bbeta_s^*\|\\
    \overset{\textup{(i)}}{\leq} & \|\bbeta_j-\x\|+ \|\x-\bbeta_s^*\|\\
    \leq & (\|\bbeta_j-\bbeta_s^* \|+\|\x-\bbeta_s^*\|)+ \|\x-\bbeta_s^*\|
    \overset{\textup{(ii)}}{\leq} \frac{k}{\lambda}+3r,
\end{align*}
where step (i) follows from $\x \in \vor_i$, and step (ii) follows from $\x \in \ball_s$ and the bound on $\|\bbeta_{j}-\bbeta_{s}^{*}\|$ proved above. We have established Part~1 of Theorem~\ref{prop:ball-family-bounds}.

\subsubsection*{Case 2: for all $(s,j,\ell)\in T_i\times[k]\times[k]$ there holds $\rho_{s}(\partial_{j,\ell})\leq\lambda$.}

In this case, for all true clusters $\{\ball_s\}$ that intersect the Voronoi set $\vor_i$, all the Voronoi boundaries enclosed by $\ball_s$ have a small relative volume.

Let us partition the set $T_i:=\{s\in[k]:\vor_{i}\cap\ball_{s}\neq\emptyset\}$ into two subsets defined as follows:
\[
\Wint_i:=\{s\in T_i:\bbeta_{s}^{*}\in\textup{int}(\vor_{i})\}\quad\text{and}\quad \Wext_i:=T_i \setminus \Wint_i = \{s\in T_i:\bbeta_{s}^{*}\not\in\textup{int}(\vor_{i})\}.
\]
Here $\Wint_i$ indexes the ground truth clusters whose centers are in the interior of the Voronoi set $\vor_{i}$; $\Wext_i$ indexes the ground truth clusters that intersect $\vor_i$ but their centers are outside its interior (i.e., the center either lies on a Voronoi boundary or in some other Voronoi set $\vor_j$). Also recall the quantities $m_{i,s}\equiv m_{i,s}(\bbeta)$ and $\bc_{i,s}\equiv \bc_{i,s}(\bbeta)$ introduced after Lemma~\ref{lem:necessary}; in particular, $m_{i,s}$ is the probability mass of $\vor_{i}\cap\mB_{s}$ with respect to $\mP_{s}$, and $\bc_{i,s}$ is the corresponding center of mass.

Note the Voronoi set $\vor_{i}$ is a polyhedron with at most $k$ facets. For each $s\in \Wext_i$, if $\rho_{s}(\partial_{j,\ell})\leq\lambda$ for all $\forall (j,\ell)$ (and in particular, for $j=i$), then all facets of $\vor_i$ intersect the ball  $\ball_s$ with a small relative volume. Moreover, $\bbeta^*_s$ is not in  $\textup{int}(\vor_i)$. With these two facts, an elementary geometric argument (formally given in Lemma~\ref{lem:prob-intersection}) shows that the intersection $\vor_{i}\cap\mB_{s}$ must have a small mass; that is, 
\begin{equation}
m_{i,s} = \mP_s(\vor_i) \leq k\lambda\rad,\quad\forall s\in \Wext_i.\label{eq:T2-mass}
\end{equation}
On the other hand, for each $s\in \Wint_i$, we must have $\bbeta^*_s \notin \textup{int}(\vor_j),\forall j\neq i$, since the interiors of Voronoi sets are disjoint. Repeating the same argument above shows that $\mP_s(\vor_{j})\leq k\lambda, \forall j\neq i$, whence
\begin{align}
    m_{i,s} = \mP_s(\vor_{i}) = & 1 - \sum_{j:j\neq i}\mP_s(\vor_{j})
    \geq  1-k^2\lambda\rad, \quad\forall s\in \Wint_i. \label{eq:contained-mass}
\end{align}
The inequalities~\eqref{eq:T2-mass} and~\eqref{eq:contained-mass} establish the first two bounds in Part~2 of Theorem~\ref{prop:ball-family-bounds}.

We next turn to Part~2(a) of Theorem~\ref{prop:ball-family-bounds}, which concerns the case with $\Wint_i=\emptyset$. This means that $T_i = \Wext_i$. We therefore have 
\begin{align*}
    \mP(\vor_i) = \frac{1}{k}\sum_{s\in [k]} m_{i,s} \leq  k\lambda \rad,
\end{align*}
where the last step holds due to equation~\eqref{eq:T2-mass} and the fact that $m_{i,s}=0,\forall s \notin T_i$.

Finally, we consider  Part~2(b) of Theorem~\ref{prop:ball-family-bounds}, which concerns the case with $\Wint_i\neq \emptyset$ and hence $\mP(\vor_i)>0$. Since $\bbeta$ is a local minimum, Lemma~\ref{lem:necessary} and the discussion thereafter ensure that 
\[
\bbeta_{i}=\frac{\sum_{s=1}^{k}m_{i,s}\bc_{i,s}}{\sum_{s=1}^{k}m_{i,s}}=\frac{\sum_{s\in T_i}m_{i,s}\bc_{i,s}}{\sum_{s\in T_i}m_{i,s}}. 
\]
Recalling the definition $\bb_i :=\frac{1}{|\Wint_i|}\sum_{s\in \Wint_i}\bbeta_{s}^{*}$, we can decompose the quantity $(\bbeta_{i}-\bb_{i})$ of interest as follows:
\begin{align}
\bbeta_{i}-\bb_{i}= & \frac{\sum_{s\in \Wint_i}m_{i,s}\bc_{i,s}+\sum_{s\in \Wext_i}m_{i,s}\bc_{i,s}}{\sum_{s\in \Wint_i}m_{i,s}+\sum_{s\in \Wext_i}m_{i,s}}-\frac{1}{|\Wint_i|}\sum_{s\in \Wint_i}\bbeta_{s}^{*}\nonumber \\
= & \underbrace{\frac{\sum_{s\in \Wint_i}m_{i,s}\bc_{i,s}+\sum_{s\in \Wext_i}m_{i,s}\bc_{i,s}}{\sum_{s\in \Wint_i}m_{i,s}+\sum_{s\in \Wext_i}m_{i,s}}-\frac{\sum_{s\in \Wint_i}m_{i,s}\bc_{i,s}}{\sum_{s\in \Wint_i}m_{i,s}}}_{\bmu}+\underbrace{\frac{\sum_{s\in \Wint_i}m_{i,s}\bc_{i,s}}{\sum_{s\in \Wint_i}m_{i,s}}-\frac{1}{|\Wint_i|}\sum_{s\in \Wint_i}\bbeta_{s}^{*}}_{\bnu}.\label{eq:A and B}
\end{align}
The following two lemmas, proved in Section~\ref{sec:bound_A_B}, control the norms of the vectors $\bmu$ and $\bnu$.
\begin{lemma}
\label{lem:bound-A} We have $\|\bmu\|\leq\frac{k\rad}{1-k^{2}\lambda\rad}+\frac{k\rad(k^{2}\lambda\rad)^{2}}{(1-k^{2}\lambda\rad)^{2}}+\frac{k^{2}\lambda\rad}{1-k^{2}\lambda\rad}\deltamax.$
\end{lemma}
\begin{lemma}
\label{lem:bound-B} We have $\|\bnu\|\leq\frac{k\rad(k^{2}\lambda\rad)}{1-k^{2}\lambda\rad}+\frac{k^{2}\lambda\rad}{1-k^{2}\lambda\rad}\deltamax.$
\end{lemma}
Applying these two lemmas to bound the right hand side of equation~(\ref{eq:A and B}), we obtain that 
\begin{equation}
\|\bbeta_{i}-\bb_{i}\|\leq\frac{k\rad}{1-k^{2}\lambda\rad}+\frac{k\rad(k^{2}\lambda\rad)}{(1-k^{2}\lambda\rad)^{2}}+\frac{2k^{2}\lambda\rad}{1-k^{2}\lambda\rad}\deltamax,\label{eq:bound-dist-weighted-center}
\end{equation}
thereby proving Part~2(b) of Theorem~\ref{prop:ball-family-bounds}.

We have completed the proof of Theorem~\ref{prop:ball-family-bounds}.

\subsection{Proof of Proposition~\ref{prop:same-direction} (Same Direction)
\label{sec:same-direction}}
Under the  perturbation direction $\bv_{i}=\bv_{j}$, the new Voronoi boundary $\partial_{i,j}(\bbeta+t\bv)$ is a translation of the original boundary $\partial_{i,j}(\bbeta)$ by the amount $t\bv$; see left panel of Figure ~\ref{fig:Illustration-change-assignment}. When both  $\Delta_{i\to j}^{\bv}$ and $\Delta_{j\to i}^{\bv}$ have measure $0$ with respect to $\mP_s$, setting $\widetilde{W}_{i,j}^{\bv}\equiv 0$ satisfies the conclusions of the proposition as $\rho_s(\partial_{i,j})=0$ in this case. Thus we only need to consider the case where $\partial_{i,j}(\bbeta)$ intersects $\ball_{s}$ non-trivially. As exactly one of the sets $\mP_s(\Delta_{i\to j}^{\bv})$ and $\mP_s(\Delta_{j\to i}^{\bv})$ is non-zero, we assume WLOG that $\mP_s(\Delta_{i\to j}^{\bv})>0$. Since $W_{j\to i}^{\bv}$ is non-positive, we have  $W_{i\to j}^{\bv}+W_{j\to i}^{\bv} \le W_{i\to j}^{\bv}$. It suffices to upper bound $W_{i\to j}^{\bv}$.

Recall expression for $W_{i\to j,s}^{\bv}$:
\begin{align*}
W_{i\to j,s}^{\bv}(t)
:=&\int_{\Delta_{i\to j}^ {\bv}(t)}(\|\x-\bbeta_{j}-t\bv_{j}\|^{2}-\|\x-\bbeta_{i}-t\bv_{i}\|^{2})f_{s}(\x)\ddup\x \\
= & \int_{\Delta_{i\to j}^{\bv}(t)}[2\langle\x,\bbeta_{i}+t\bv_{i}-\bbeta_{j}-t\bv_{j}\rangle+(\|\bbeta_{j}+t\bv_{j}\|^{2}-\|\bbeta_{i}+t\bv_{i}\|^{2})]f_{s}(\x)\ddup\x.
\end{align*}
 Since the integrand above only involves the Euclidean norm, we are free to choose any coordinate system. In particular, we choose the origin to be $\frac{1}{2}(\bbeta_{i}+\bbeta_{j})$, the principal axis to be the direction of $\bbeta_{i}-\bbeta_{j}$, and the secondary axis to be the direction orthogonal to $\bbeta_{i}-\bbeta_{j}$ and in $\text{span}\{\bbeta_{i}-\bbeta_{j},\bv_{i}\}$. Under this coordinate system, we have 
\begin{align*}
W_{i\to j,s}^{\bv}(t) & =2\int_{\Delta_{i\to j}^{\bv}(t)}(\|\bbeta_{i}-\bbeta_{j}\|x_{1}-t\langle\bbeta_{i}-\bbeta_{j},\bv_{i}\rangle)f_{s}(\x)\ddup\x\\
 & =2\int_{x_{2},\ldots,x_{d}:\x\in\Delta_{i\to j}^{\bv}(t)}\int_{x_{1}=0}^{t\cos(\theta)}[\|\bbeta_{i}-\bbeta_{j}\|x_{1}-t\|\bbeta_{i}-\bbeta_{j}\|\cos(\theta)]f_{s}(\x)\ddup x_{1}\ddup x_{2}\ldots\ddup x_{d},
\end{align*}
where we introduce the shorthand $\theta\equiv\theta_{i, j}^{\bv}:=\angle(\bbeta_{i}-\bbeta_{j},\bv_{j})=\angle(\bbeta_{i}-\bbeta_{j},\bv_{i})$ and, slightly abusing notation,  still use $f_s$ to denote the density function under the new coordinate system. Note that the region
\[
S(z,t):=\{(x_{2},\ldots,x_{d}):\x\in\Delta_{i\to j}^{\bv}(t),x_{1}=z\}
\]
is a vertical slice under the current coordinate system; in particular, $S(z,t)$ is the intersection of the set $\Delta_{i\to j,s}^{\bv}(t)$ and the hyperplane that is parallel
to $\partial_{i,j}$ and at a distance $z$ from $\partial_{i,j}$. Defining the integral
\[
\rho_{i\to j,s}^{\bv}(z,t):=\int_{S(z,t)}f_{s}(z,x_2,\ldots,x_d)\ddup x_{2}\ldots\ddup x_{d},
\]
we can write 
\begin{equation}
W_{i\to j,s}^{\bv}(t)=2\int_{x_{1}=0}^{t\cos(\theta)} \Big[\|\bbeta_{i}-\bbeta_{j}\|x_{1}-t\|\bbeta_{i}-\bbeta_{j}\|\cos(\theta)\Big]\rho_{i\to j,s}^{\bv}(x_{1},t)\ddup x_{1}.\label{eq:W_expression}
\end{equation}
When $t$ is small, we have the sandwich bound $m(t)\leq\rho_{i\to j,s}^{\bv}(x_{1},t)\leq M(t)$, where
\[
m(t):=\min_{x_{1}\in[0,t\cos(\theta)]}\rho_{i\to j,s}^{\bv}(x_{1},t),
\]
 and 
\[
M(t):=\max_{x_{1}\in[0,t\cos(\theta)]}\rho_{i\to j,s}^{\bv}(x_{1},t).
\]
Here $m(t)$ and $M(t)$ are well-defined as they are the max/min of the bounded function $\rho_{i\to j,s}^{\bv}$ over the compact interval $[0,t\cos(\theta)]$. Moreover, $m(t)$ and $M(t)$ satisfy
\[
\lim_{t\to0}m(t)=\lim_{t\to0}M(t)=\frac{\revol(\partial_{i,j})}{\vol(\ball_{s}(\rad))}=\rho_{s}(\partial_{i,j}).
\]
Bounding the two terms in the bracket in equation~(\ref{eq:W_expression}) separately, we obtain that
\[
2d_{i,j}\cos^{2}(\theta)\cdot m(t)t^{2}\leq2\int_{x_{1}=0}^{t\cos(\theta)}\|\bbeta_{i}-\bbeta_{j}\|x_{1}\rho_{i\to j,s}^{\bv}(x_{1},t)\ddup x_{1}\leq2d_{i,j}\cos^{2}(\theta)\cdot M(t)t^{2}
\]
and 
\[
4d_{i,j}\cos^{2}(\theta)\cdot m(t)t^{2}\leq2t\int_{x_{1}=0}^{t\cos(\theta)}\|\bbeta_{i}-\bbeta_{j}\|\cos(\theta)\rho_{i\to j,s}^{\bv}(x_{1},t)\ddup x_{1}\leq4d_{i,j}\cos^{2}(\theta)\cdot M(t)t^{2},
\]
whence 
\[
2(m(t)-2M(t))d_{i,j}\cos^{2}(\theta)t^{2}\leq W_{i\to j,s}^{\bv}(t)\leq2(M(t)-2m(t))d_{i,j}\cos^{2}(\theta)t^{2}.
\]
It is then easy to see that $W_{i\to j,s}^{\bv}(0)=0$, $\frac{\ddup}{\ddup t}W_{i\to j,s}^{\bv}(t)\mid_{t=0}=0$
and 
\[
\lim_{t\to0}\frac{W_{i\to j,s}^{\bv}(t)}{t^{2}}=-2d_{i,j}\cos^{2}(\theta)\rho_{s}(\partial_{i,j}).\]
In summary, setting $\widetilde{W}_{i,j}^{\bv} = W_{i\to j}^{\bv}$, we have established that $W_{i\to j}^{\bv}+W_{j\to i}^{v} \le \widetilde{W}_{i\to j}^{\bv}$  and that $\widetilde{W}_{i\to j}^{\bv}$ satisfies the desired analytical properties in Proposition~\ref{prop:same-direction}.

\subsection{Proof of Proposition~\ref{prop:opposite-direction} (Opposite Direction)
\label{sec:opposite-direction}}

\begin{figure}
\centering
\includegraphics[scale=0.25, clip, trim=0 0 0 0]{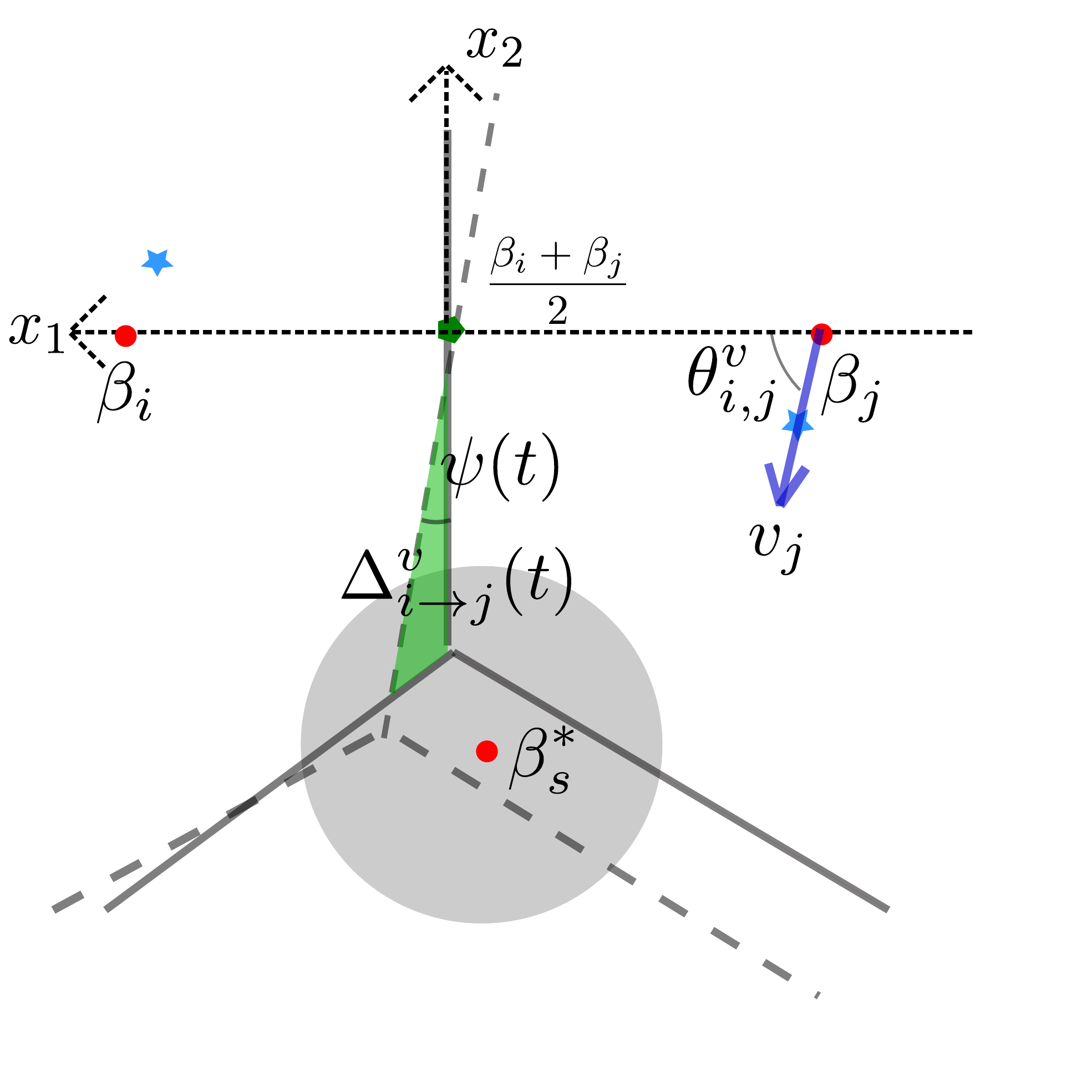}~~~~\includegraphics[scale=0.25, clip, trim=0 0 0 0 ]{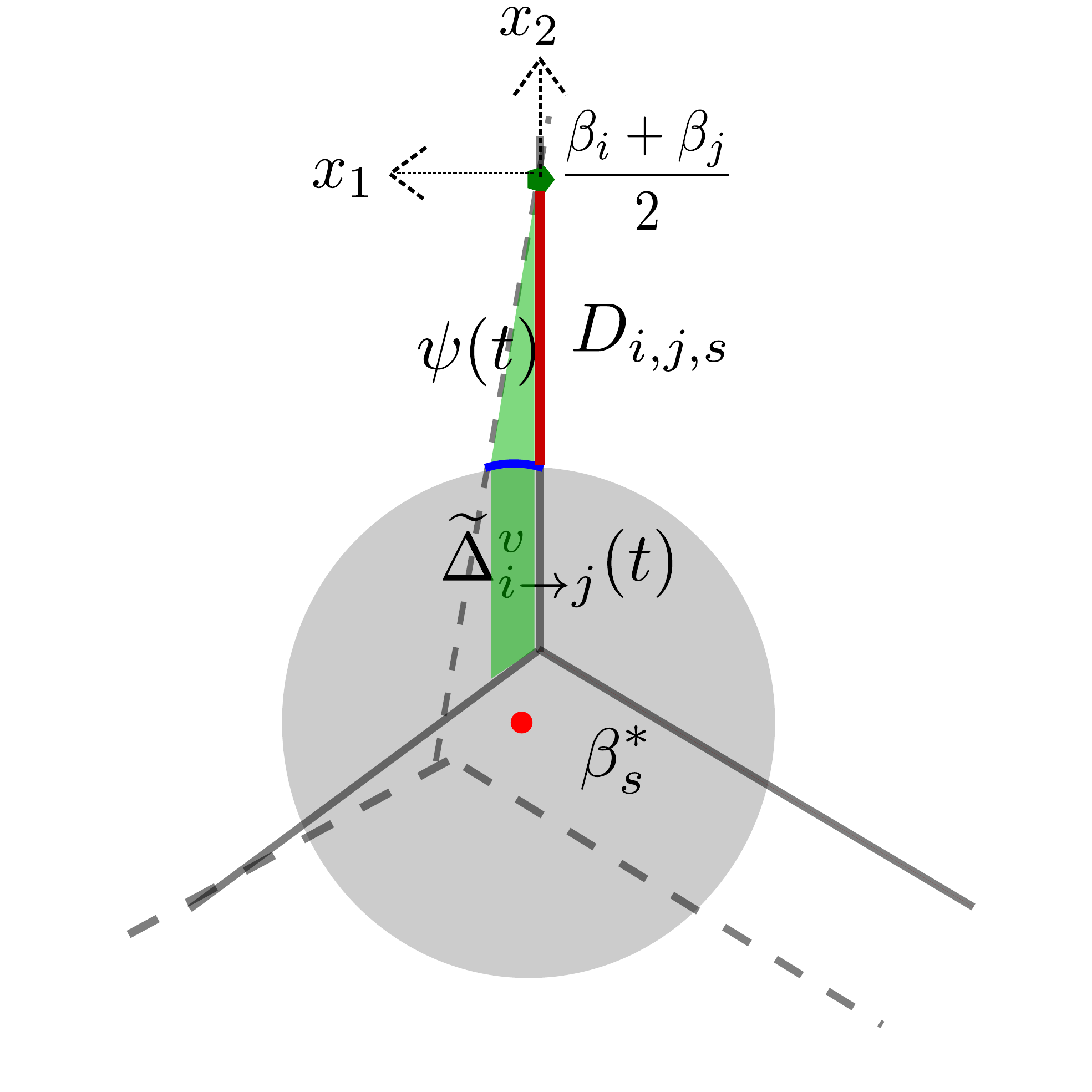}
\caption{\label{fig:upper-bound-opposite-direction}Illustration of the local
coordinate system and the upper bound function. The local coordinate system has the origin at $\frac{\bbeta_i+\bbeta_j}{2}$, represented by the dark green dot. Its principal axis is in the direction of $\bbeta_i-\bbeta_j$, plotted as the $x_1$ axis, and its secondary axis is plotted as the $x_2$ axis. In the left panel, the red dots represent $\bbeta_i$ and $\bbeta_j$ respectively; the blue stars represent $\bbeta_i+t\bv_i$ and $\bbeta_j+t\bv_j$ respectively, with $\bv_i=-\bv_j$. The dark blue arrow indicates the direction of $\bv_j$ and it has an angle $\theta_{i,j}^{v}$ with the vector $\bbeta_i-\bbeta_j$, the $x_1$ axis. Correspondingly, the Voronoi boundary $\partial_{i,j}(\bbeta+t\bv)$ rotates around the origin with an angle $\psi(t)$. The boundaries $\partial(\bbeta+t\bv)$ are plotted using dotted lines. The shaded green region in the left panel corresponds to the set $\Delta_{i\to j}^{\bv}(t)$, in which the point becomes closer to $\bbeta_j+t\bv_j$ than $\bbeta_i+t\bv_i$ after $\bbeta$ is moved to $\bbeta+t\bv$. 
In the right panel, we demonstrate the set $\widetilde{\Delta}_{i\to j}^{\bv}(t)$ using the shaded green region. It is a subset of $\Delta_{i\to j}^{\bv}(t)$, enclosed by the hyperplane $\{\x:x_1=0\}$ and the translated hyperplane $\{\x:x_1=D_{i,j,s}\tan(\psi(t))\}$}.
\end{figure}

Under the perturbation direction $\bv_i=-\bv_j$, the Voronoi boundary $\partial_{i,j}(\bbeta)$ rotates around the mid point $\frac{\bbeta_i+\bbeta_j}{2}$; see right panel of Figure~\ref{fig:Illustration-change-assignment}. When both $\Delta_{i\to j}^{\bv}(t)$ and $\Delta_{j\to i}^{\bv}(t)$ has measure $0$ with respect to $\mP_s$, setting $\widehat{W}_{i,j,s}\equiv 0$ satisfies the conclusion of the proposition, as $\rho_s(\partial_{i,j})=0$ in this case. When $\frac{1}{2}(\bbeta_{i}+\bbeta_{j})\in\ball_{s}$, we can also set $\widehat{W}_{i\to j,s}^{\bv}\equiv 0$, as $D_{i,j,s}=0$ in this case. 

In the rest of the proof, We assume WLOG that $\frac{1}{2}(\bbeta_{i}+\bbeta_{j})\not\in\ball_{s}$ and $\mP_s(\Delta_{i\to j,s}^{\bv}(t))>0$. Note that in this case we have $\partial_{i,j}\cap \ball_s \neq \emptyset$,  $\rho_s(\partial_{i,j})>0$ and $D_{i,j,s} >0.$ Since $W_{i\to j,s}^{\bv}+W_{j\to i,s}^{\bv} \le W_{i\to j, s}^{\bv}$, it suffices to find an function $\widehat{W}_{i,j,s}^{\bv}$ that upper bounds $W_{i\to j,s}^{\bv}$ in a neighborhood of $0$ and satisfies the desired analytical properties.

Similarly to Section~\ref{sec:same-direction}, We may use any convenient coordinate system. In particular, we choose the origin to be $\frac{1}{2}(\bbeta_{i}+\bbeta_{j})$, the principal axis to be the direction of $\bbeta_{i}-\bbeta_{j}$, and the secondary axis to be the direction that is orthogonal to $\bbeta_{i}-\bbeta_{j}$ and in the plane $\linspace_{i,j,s}$; see the left panel of Figure~\ref{fig:upper-bound-opposite-direction}. Under this coordinate system, we have the representation
\begin{align}
\label{eq:opp_dir_coordinate}
\bbeta_{i}=(d_{i,j},0,\ldots,0),\quad \bbeta_{j}=(-d_{i,j},0,\ldots,0),\quad \bbeta_{s}^{*}=(b_{1}^{*},b_{2}^{*},0,\ldots,0)
\end{align}
for some $b_{1}^{*},b_{2}^{*}\in\real$. The orientation of the secondary axis can be chosen to satisfy $b_2^*<0$. We note that the boundary $\partial_{i,j}(\bbeta)$ rotates around $\frac{1}{2}(\bbeta_{i}+\bbeta_{j})$ by an angle $\psi(t)$ that satisfies
\begin{align}
\tan(\psi(t))=\frac{t\sin(\theta)}{d_{i,j}-t\cos(\theta)} \label{eq:perturbed-angle},\end{align}

where we recall that $\theta\equiv\theta_{i, j}^{\bv} \in [0,\pi/2]$ is the (unsigned) angle between $\bv_j$ and $\bbeta_i-\bbeta_j$. Moreover, since we assume $\mP_s(\Delta_{i\to j}^{\bv}(t))>0$, the directions $\bv_i$ and $\bv_j$ have the following coordinate representation:
\begin{align*}
    \bv_j= \big(\cos(\theta),-\sin(\theta),0,\ldots,0 \big) = -\bv_i.
\end{align*}

We now proceed to upper bound the function $W_{i\to j,s}^{\bv}$. Define the polyhedron set
\begin{align}
 \widetilde{\Delta}_{i\to j}^{\bv}(t):=\big\{\x\in \Delta_{i\to j}^{\bv}(t):x_1\leq D_{i,j,s}\tan(\psi(t)) \big\}.\label{eq:sub-delta-set}
 \end{align}
The set $\widetilde{\Delta}_{i\to j}^{\bv}(t)$ is sandwiched between the two hyperplanes $x_1=0$ and $x_1=D_{i,j,s}\tan(\psi(t))$; see the right panel of Figure~\ref{fig:upper-bound-opposite-direction} for an illustration. With the above notations, we can upper bound $W_{i\to j,s}^{\bv}$ as follows:
\begin{align}
W_{i\to j,s}^{\bv} 
&\overset{\text{(i)}}{\le} \int_{\widetilde{\Delta}_{i\to j}^{\bv}(t)} \Big[2\langle\x,\bbeta_{i}+t\bv_{i}-\bbeta_{j}-t\bv_{j}\rangle+\big(\|\bbeta_{j}+t\bv_{j}\|^{2}-\|\bbeta_{i}+t\bv_{i}\|^{2}\big)\Big] f_{s}(\x)d\x \nonumber\\
& \overset{\text{(ii)}}{=} \int_{\tilde{\Delta}_{i\to j}^{\bv}(t)} 2\langle\x,\bbeta_{i}+t\bv_{i}-\bbeta_{j}-t\bv_{j}\rangle f_{s}(\x)\ddup\x \nonumber\\
& =\int_{\tilde{\Delta}_{i\to j}^{\bv}(t)} \Big[2\|\bbeta_{i}-\bbeta_{j}\|x_{1}-4t\cos(\theta)x_{1}+4t\sin(\theta)x_{2}\Big]f_{s}(\x)\ddup\x \label{eq:last-display-upperbound_W_ijs},
\end{align}
where step (i) holds because the integrand in the definition of $W_{i\to j,s}^{\bv}$ is non-positive and thus integrating over a smaller set $\widetilde{\Delta}_{i\to j}^{\bv}(t) \subseteq \Delta_{i\to j}^{\bv}(t)$ does not decrease the value of the integral, and step (ii) holds since under the current coordinate system, $\bbeta_i = -\bbeta_j$ and $\bv_i = - \bv_j$.

To proceed, we let 
\begin{align}
    D(t):=\max\big\{x_2:\x\in \widetilde{\Delta}_{i\to j}^{\bv}(t)\cap \ball_s\big\} \label{eq:max-x2}.
\end{align}
denote maximum of the second coordinate of the set $\widetilde{\Delta}_{i\to j}^{\bv}(t)\cap \ball_s$ under the current coordinate system. The following lemma, proved at the end of this section, characterizes the limit property of $D(t)$.
\begin{lemma}[Negative second coordinate at the boundary]
 \label{lem:second-coordinate} 
 Suppose that $\bv$  satisfies $\|\bv_i\|=\|\bv_j\|=1$  and $\bv_i=-\bv_j\in \linspace_{i,j,s}$, $\frac{\bbeta_{i}+\bbeta_{j}}{2}\not\in\mB_{s}$ and $\Delta_{i\to j}^{\bv}(t)\cap \ball_s\neq \emptyset$. We have  $\lim_{t\to 0} D(t) = -D_{i,j,s}$.
\end{lemma}
Lemma~\ref{lem:second-coordinate} ensures that
$
\lim_{t\to 0} D(t) = -D_{i,j,s} < 0.\label{eq:limit_of_x2}
$
Consequently, when $t$ is sufficiently small, we have $D(t)<0$ by the continuity.

Continuing from the last display equation~\eqref{eq:last-display-upperbound_W_ijs}, we obtain our final upper bound $\widehat{W}_{i,j,s}^{\bv}(t)$:
\begin{align*}
W_{i\to j,s}^{\bv}  
 & \leq\int_{\tilde{\Delta}_{i\to j}^{\bv}(t)}\Big[2\|\bbeta_{i}-\bbeta_{j}\|x_{1}-4t\cos(\theta)x_{1}+4t\sin(\theta)D(t)\Big] f_{s}(\x)\ddup\x
 =:\widehat{W}_{i,j,s}^{\bv}(t).
\end{align*}

To establish the analytical properties of $\widehat{W}_{i,j,s}^{\bv}(t)$, we follow a similar argument as in Section~\ref{sec:same-direction}. Define the integral
\[
\rho_{i\to j,s}^{\bv}(z,t):=\int_{x_{2},\ldots,x_{d}:\x\in\tilde{\Delta}_{i\to j}^{\bv}(t),x_{1}=z}f_{s}(z,x_2,\ldots,x_d)\ddup x_{2}\ldots\ddup x_{d},
\]
and rewrite $\widehat{W}_{i,j,s}^{\bv}(t)$ compactly as follows: 
\[
\widehat{W}_{i,j,s}^{\bv}(t)=\int_{0}^{D_{i,j,s}\tan(\psi(t))} \Big[2\|\bbeta_{i}-\bbeta_{j}\|x_{1}-4t\cos(\theta)x_{1}+4t\sin(\theta)D(t)\Big]\rho_{i\to j,s}^{\bv}(x_{1},t)\ddup x_{1}.
\]
We have the sandwich bound $m(t)\leq\rho_{i\to j,s}^{\bv}(x_{1},t)\leq M(t)$, valid for $x_{1}\in\big[0,D_{i,j,s}\tan(\psi(t))\big]$, where $m(t):=\min_{x_{1}\in\big[0,D_{i,j,s}\tanh(\psi(t))]}\rho_{i\to j,s}^{\bv}(x_{1},t)$ and $M(t):=\max_{x_{1}\in\big[0,D_{i,j,s}\tan(\psi(t))\big]}\rho_{i\to j,s}^{\bv}(x_{1},t)$.

Moreover, the functions $m(\cdot)$ and $M(\cdot)$ satisfy
\[
\lim_{t\to0}m(t)=\lim_{t\to0}M(t)=\frac{\revol(\partial_{i,j})}{\vol(\ball_{s})}=\rho_{s}(\partial_{i,j}).
\]
With some algebra as well as the non-positivity of $D(t)$, we obtain the following sandwich bound for $\widehat{W}_{i\to j}^{\bv}$ over a small neighborhood of $0$:
\begin{align*}
\widehat{W}_{i\to j,s}^{\bv}(t) & \geq 2m(t)d_{i,j}D^2_{i,j,s}\tan^2(\psi(t))+4tD_{i,j,s}\sin(\theta)D(t)M(t)\tan(\psi(t))\\
& -2tD^2_{i,j,s}\cos(\theta)M(t)\tan^2(\psi(t));
\end{align*}
and
\begin{align*}
\widehat{W}_{i\to j,s}^{\bv}(t) & \leq 2M(t)d_{i,j}D^2_{i,j,s}\tan^2(\psi(t))+4tD_{i,j,s}\sin(\theta)D(t)m(t)\tan(\psi(t))\\
&-2tD^2_{i,j,s}\cos(\theta)m(t)\tan^2(\psi(t)).
\end{align*}
 Proposition~\ref{prop:opposite-direction} follows immediately from the limit properties of $m(t)$, $M(t)$, $D(t)$ and $\tanh(\psi(t))$. %

\begin{proof}[Proof of Lemma~\ref{lem:second-coordinate}]
We use the same notations and coordinate system as before. Observe that in equation~\eqref{eq:max-x2}, the maximum that defines $D(t)$ must be attained at a point in $\widetilde{\Delta}_{i\to j}^{\bv}(t)\cap\mS_{s}$, where $\mS_{s}$ is the hypersphere of the ball $\ball_s$; see the right panel of Figure~\ref{fig:upper-bound-opposite-direction}. We claim that the maximum must also be attained by a point in the hyperplane $\linspace_{i,j,s}$. Indeed, recalling the representation in equation~\eqref{eq:opp_dir_coordinate}, we see that each point $\x\in \widetilde{\Delta}_{i\to j}^{\bv}(t) \cap \mS_{s}$ must satisfy 
\begin{align*}
(x_1-b_1^*)^2+(x_2-b_2^*)^2+\sum_{j\geq 3}x_j^2=\rad^2 
\quad\text{and}\quad
x_1\in \big[0, D_{i,j,s}\tan(\psi(t))\big] 
\end{align*}
From the above equation it is clear that for each fixed $z\in [0, D_{i,j,s}\tan(\psi(t))]$, over the set $\mS_{s}\cap \widetilde{\Delta}_{i\to j}^{\bv}(t)\cap \{\x:x_1=z\}$, the maximum of $x_2$ is attained exactly when $x_j=0$ for all $j\geq 3$, that is, $\x\in \linspace_{i,j,s}$. Combining these observations, we conclude that
\begin{align*}
D(t) = \max\big\{x_2:\x\in \widetilde{\Delta}_{i\to j}^{\bv}(t)\cap \mS_s\cap \linspace_{i,j,s}\big\}.    
\end{align*}
Note that the set $\widetilde{\Delta}_{i\to j}^{\bv}(t)\cap \mS_s\cap \linspace_{i,j,s}$ is compact, which is represented by the solid blue segment in $\ball_s$ in the right panel of Figure~\ref{fig:upper-bound-opposite-direction};  as $t\to 0$, this set shrinks continuously to a single point $(0,-D_{i,j,s}, \ldots,0)$. It thus follows that
$\lim_{t\to 0} D(t)=-D_{i,j,s}.$
This completes the proof of Lemma~\ref{lem:second-coordinate}.
\end{proof}

\subsection{Proofs of Lemmas~\ref{lem:bound-A} and~\ref{lem:bound-B} \label{sec:bound_A_B}}

For convenience we restate the bounds in equations~\eqref{eq:T2-mass} and ~\eqref{eq:contained-mass} as follows:
\begin{align}\label{eq:mis_bound}
    \begin{cases}
    m_{i,s} \le k \lambda \rad, & \text{if } s\in \Wext_i, \\
    m_{i,s} \ge 1-k^2\lambda\rad, & \text{if } s \in \Wint_i,
    \end{cases}
\end{align}
where we recall that $m_{i,s}=\mP(\vor_i)$ is the probability mass of the set $\vor_i$ with respect to the uniform distribution on the ball $\ball_s \equiv \ball_{\bbeta^*_s}(\rad)$, and $\bc_s$ is the corresponding center of mass. Applying the  simple geometric result in Lemma~\ref{lem:bound-dist-2-center-of-mass}, we further obtain the bound
\begin{align}\label{eq:cis_bound}
    \|\bbeta_{s}^{*}-\bc_{i,s}\| \le \frac{\rad \cdot (1-m_{i,s})}{m_{i,s}} \le \begin{cases}
    \frac{\rad}{m_{i,s}}, & \text{if }s\in\Wext_i,\\
    \frac{ k^{2}\lambda\rad^2}{m_{i,s}}, & \text{if } s\in\Wint_i,
    \end{cases}
\end{align}
where the last step follows from the fact that $m_{i,s}\le 1$ and equation~\eqref{eq:mis_bound}. We are ready to prove the two lemmas.

\begin{proof}[Proof of Lemma~\ref{lem:bound-A}]
We have the following decomposition of the vector $\bmu$:
\begin{align*}
\bmu= & \frac{\sum_{s\in \Wint_i}m_{i,s}\bc_{i,s}+\sum_{s\in \Wext_i}m_{i,s}\bc_{i,s}}{\sum_{s\in \Wint_i}m_{i,s}+\sum_{s\in \Wext_i}m_{i,s}}-\frac{\sum_{s\in \Wint_i}m_{i,s}\bc_{i,s}}{\sum_{s\in \Wint_i}m_{i,s}}\\
= & \frac{\sum_{s\in \Wext_i}m_{i,s}\bc_{i,s}}{\sum_{s\in \Wint_i}m_{i,s}+\sum_{s\in \Wext_i}m_{i,s}}-\frac{\sum_{s\in \Wext_i}m_{i,s}\sum_{s\in \Wint_i}m_{i,s}\bc_{i,s}}{(\sum_{s\in \Wint_i}m_{i,s}+\sum_{s\in \Wext_i}m_{i,s})\sum_{s\in \Wint_i}m_{i,s}}\\
= & \frac{\sum_{s\in \Wext_i}m_{i,s}(\bc_{i,s}-\bbeta_{s}^{*})}{\sum_{s\in \Wint_i}m_{i,s}+\sum_{s\in \Wext_i}m_{i,s}}-\frac{\sum_{s\in \Wext_i}m_{i,s}\sum_{s\in \Wint_i}m_{i,s}(\bc_{i,s}-\bbeta_{s}^{*})}{(\sum_{s\in \Wint_i}m_{i,s}+\sum_{s\in \Wext_i}m_{i,s})\sum_{s\in \Wint_i}m_{i,s}}\\
 & +\frac{\sum_{s\in \Wext_i}m_{i,s}\bbeta_{s}^{*}}{\sum_{s\in \Wint_i}m_{i,s}+\sum_{s\in \Wext_i}m_{i,s}}-\frac{\sum_{s\in \Wext_i}m_{i,s}\sum_{s\in \Wint_i}m_{i,s}\bbeta_{s}^{*}}{(\sum_{s\in \Wint_i}m_{i,s}+\sum_{s\in \Wext_i}m_{i,s})\sum_{s\in \Wint_i}m_{i,s}}.
\end{align*}
It follows that
\begin{align*}
\|\bmu\|\leq & \frac{\sum_{s\in \Wext_i}m_{i,s}\|\bc_{i,s}-\bbeta_{s}^{*}\|}{\sum_{s\in \Wint_i}m_{i,s}+\sum_{s\in \Wext_i}m_{i,s}}+\frac{\sum_{s\in \Wext_i}m_{i,s}\sum_{s\in \Wint_i}m_{i,s}\|\bc_{i,s}-\bbeta_{s}^{*}\|}{(\sum_{s\in \Wint_i}m_{i,s}+\sum_{s\in \Wext_i}m_{i,s})\sum_{s\in \Wint_i}m_{i,s}}\\
 & +\frac{2\sum_{s\in \Wext_i}m_{i,s}}{\sum_{s\in \Wint_i}m_{i,s}+\sum_{s\in \Wext_i}m_{i,s}}\deltamax. \\
\overset{\text{(i)}}{\leq}& \frac{k\rad}{1-k^{2}\lambda\rad} + \frac{(k^{2}\lambda\rad)(k(k^{2}\lambda\rad^2))}{(1-k^{2}\lambda\rad)^{2}} + \frac{k^{2}\lambda\rad}{1-k^{2}\lambda\rad}\deltamax, 
\end{align*}
where the last step holds due to the inequalities~\eqref{eq:mis_bound} and~\eqref{eq:cis_bound} as well as the fact that $|\Wint_i|,|\Wext_i| \in [1, k]$. This completes the proof of Lemma~\ref{lem:bound-A}.
\end{proof}
\begin{proof}[Proof of Lemma~\ref{lem:bound-B}]
We have the following decomposition of the vector $\bnu$:
\begin{align*}
\bnu= & \frac{\sum_{s\in \Wint_i}m_{i,s}\bc_{i,s}}{\sum_{s\in \Wint_i}m_{i,s}}-\frac{1}{|\Wint_i|}\sum_{s\in \Wint_i}\bbeta_{s}^{*}\\
= & \frac{\sum_{s\in \Wint_i}m_{i,s}(\bc_{i,s}-\bbeta_{s}^{*})}{\sum_{s\in \Wint_i}m_{i,s}}+\frac{\sum_{s\in \Wint_i}\big(m_{i,s}-\frac{1}{|\Wint_i|}\sum_{s'\in \Wint_i}m_{i,s'}\big)\bbeta_{s}^{*}}{\sum_{s\in \Wint_i}m_{i,s}}.
\end{align*}
Using the inequalities~\eqref{eq:mis_bound} and~\eqref{eq:cis_bound} as well as the triangle inequality, we obtain that
\[
\|\bnu\|\leq\frac{k(k^{2}\lambda\rad^2)}{1-k^{2}\lambda\rad}+\frac{k^{2}\lambda\rad}{1-k^{2}\lambda\rad}\deltamax,
\]
thereby proving Lemma~\ref{lem:bound-B}.
\end{proof}

\section{Conclusion}\label{sec:conclusion}
In this paper, we characterize the structure of all the local minima in the $k$-means problem. We show that under an appropriate separation condition of  the ground truth clusters, the local minima are always composed of one-fit-many, many-fit-one or almost-empty type associations between the fitted and ground truth centers. 

Several future directions are of interests for both theory and applications. An immediate direction is to generalize our results from the population case to the finite sample case, and from balanced spherical GMMs to more general mixture models with imbalanced clusters, general covariance matrices and heavy-tailed distributions.

Also, while we have focused on the $k$-means formulation, we expect that similar structural results hold for a much broader class of clustering formulations, particularly the maximum likelihood formulation of mixture problems. On the computational side, we have discussed the implications of our results for improving clustering algorithms. Rigorously justifying these algorithms (which are largely heuristic so far) in a broad range of models would be interesting.

Finally, it would be of great interest to establish similar structural results for other non-convex optimization problems that arise in machine learning and statistics applications.

\section*{Acknowledgement}
W. Qian and Y. Chen are partially supported by NSF CRII award 1657420 and grant 1704828.

\bibliographystyle{plain}
\bibliography{kmean-structural-highsnr}

\appendixpage
\appendix

\section{\label{sec:Equivalence-to-standard-k-means}Equivalence to the partition-based formulation}

A common way of formulating the $k$-means clustering problem is as follows: given a set of observations $\x_{1},\ldots,\x_{n}\in\mathbb{R}^{d}$, we find a partition $\partition=\{S_{1},\ldots,S_{k}\}$ of these observations such that the within-cluster sum of squared distances is minimized:
\begin{equation}\label{eq:k-means-formulation-partition}
\min_{\partition}\sum_{j=1}^{k}\sum_{\x\in S_{j}}\|\x-\bm{\mu}_{j}\|^{2},
\end{equation}
where $\bm{\mu}_{j} = \frac{1}{|S_j|} \sum_{\x \in S_{j}} \x$ is the mean of points in cluster $i$. Meanwhile, the formulation~\eqref{eq:k-means-formulation} used in this paper  is based on optimizing over the centers $\bbeta=(\bbeta_{1},\ldots,\bbeta_{k})$, restated as follows
\begin{align*}
\min_{\bbeta}\sum_{i=1}^{n}\min_{j\in[k]}\|\x_{i}-\bbeta_{j}\|^{2}.  
\end{align*}
Note that each solution $\bbeta$ induces a partition of $\real^{d}$ via the Voronoi diagram, hence a partition $\{S_j(\bbeta)\}$ of the observations, and that the sum of squared distances over a set of points is minimized by their mean. Combining this observations, we obtain that for any $\bbeta$:
\begin{equation}\label{eq:equivalence_1}
\min_{\partition}\sum_{j=1}^{k}\sum_{\x\in S_{j}}\|\x-\bm{\mu}_{j}\|^{2}
\leq \sum_{j=1}^{k}\sum_{\x \in S_j(\bbeta)} \|\x-\bbeta_{j}\|^{2} 
= \sum_{i=1}^{n}\min_{j\in[k]}\|\x_{i}-\bbeta_{j}\|^{2}.
\end{equation}
On the other hand, for any partition $\partition=\{S_{1},\ldots,S_{k}\}$ of the data points and its corresponding means $(\bm{\mu}_{1},\ldots,\bm{\mu}_{k})$, we have  
\begin{equation}\label{eq:equivalence_2}
\sum_{j=1}^{k}\sum_{\x\in S_{j}}\|\x-\bm{\mu}_{j}\|^{2}\ge\sum_{i=1}^{n}\min_{j\in[k]}\|\x_{i}-\bm{\mu}_{j}\|^{2}\geq\min_{\bbeta}\sum_{i=1}^{n}\min_{j\in[k]}\|\x_{i}-\bbeta_{j}\|^{2}.
\end{equation}
Taking the minimum over $\bbeta$ of both sides of equation~\eqref{eq:equivalence_1}, and the minimum over $\partition$ for equation~\eqref{eq:equivalence_2}, we conclude that the two formulations~\eqref{eq:k-means-formulation-partition} and~\eqref{eq:k-means-formulation} have the same optimal values. Moreover, an optimal solution for one formulation induces an optimal solution for the other. Hence these two formulations are equivalent.

\section{\label{sec:proof_truth_global_opt}Proof of Proposition~\ref{prop:truth_is_global_opt}}

In this section we prove Proposition~\ref{prop:truth_is_global_opt}, which states that under the Stochastic Ball Model, the ground truth centers $\bbeta^{*}$ is the global minimum of the $k$-means objective
function $G$.
\begin{proof}
We begin by upper bounding the objective value of the ground truth: 
\begin{align*}
G(\bbeta^{*}) & =\frac{1}{k}\sum_{s\in[k]}\int\min_{i\in[k]}\| \x-\bbeta_{i}^{*}\| ^{2}f_{s}(\x)\ddup\x\\
 & \le\frac{1}{k}\sum_{s\in[k]}\int\| \x-\bbeta_{s}^{*}\| ^{2}f_{s}(\x)\ddup\x\\
 & \le\rad^{2},
\end{align*}
where the second inequality follows from the fact that each true cluster $\ball_{s}$ has radius $\rad$.

Now let $\bbeta$ be a global minimum of $G$. By optimality of $\bbeta$, we have for each $s\in[k]$:
\begin{align*}
\rad^2 \ge G(\bbeta) 
& \overset{\text{(i)}}{\ge} \frac{1}{k}\int\min_{i\in[k]} \| \x-\bbeta_{i}\| ^{2}f_{s}(\x)\ddup\x \\
 & \overset{\text{(ii)}}{\ge} \frac{1}{k}\int\min_{i\in[k]}\left(\frac{1}{2}\| \bbeta_{s}^{*}-\bbeta_{i}\| ^{2}-\| \x-\bbeta_{s}^{*}\| ^{2}\right)f_{s}(\x)\ddup\x \\
 &  \overset{\text{(iii)}}{\ge} \frac{1}{2k} \min_{i\in[k]} \|\bbeta_s^* - \bbeta_i \|^2 - \frac{\rad^2}{k}.
\end{align*}
where step (i) holds by ignoring $k-1$ clusters, step (ii) holds by the inequality $(a-b)^{2}\ge\frac{1}{2}a^{2}-b^{2}$, and step (iii) holds because $f_s$ is a probability density. From the above equation we obtain that 
\[
\min_{i\in[k]}\| \bbeta_{s}^{*}-\bbeta_{i}\| <2\sqrt{k}\rad,\qquad\forall s\in[k];
\]
that is, each true center $\bbeta_{s}^{*}$ is $2\sqrt{k}\rad$-close to at least one $\bbeta_{i}$. We further observe that each $\bbeta_{i}$ is $2\sqrt{k}\rad$-close to at most one true center $\bbeta_{s}^{*}$; otherwise, by the triangle inequality we woud have $\deltamin\le\| \bbeta_{s}^{*}-\bbeta_{s'}^{*}\| \le\| \bbeta_{s}^{*}-\bbeta_{i}\| +\| \bbeta_{s'}^{*}-\bbeta_{i}\| <4\sqrt{k}\rad,$
contradicting the SNR assumption $\snrmin:=\frac{\deltamin}{\rad}\ge 6\sqrt{k}$. Since the number of $\bbeta_{i}$'s is equal to that of $\bbeta_{s}^{*}$'s, we deduce that each $\bbeta_{i}$ is $2\sqrt{k}\rad$-close to\emph{ exactly one} $\bbeta_{s}^{*}$. Without loss of generality, we may assume that 
\[
\| \bbeta_{s}^{*}-\bbeta_{s}\| <2\sqrt{k}\rad,\qquad\forall s\in[k].
\]

When the above inequality and the SNR assumption $\snrmin:=\frac{\deltamin}{\rad}\ge 6\sqrt{k}$ hold, we have for each pairs $(s,s')\in[k]\times[k]$ with $s\neq s'$ and each $\x\in\ball_{s}$:
\begin{align*}
\| \x-\bbeta_{s}\|  & \le\| \x-\bbeta_{s}^{*}\| +\| \bbeta_{s}^{*}-\bbeta_{s}\| \\
 & <\rad+2\sqrt{k}\rad\\
 & \le 6\sqrt{k}\rad-\rad-2\sqrt{k}\rad\\
 & <\| \bbeta_{s}^{*}-\bbeta_{s'}^{*}\| -\| \x-\bbeta_{s}^{*}\| -\| \bbeta_{s'}-\bbeta_{s'}^{*}\| \\
 & \le\| \x-\bbeta_{s'}\| ,
\end{align*}
which implies that $\ball_{s}\subseteq\vor_{s}(\bbeta)$. Applying Lemma~\ref{lem:necessary} to the global minimum $\bbeta$, we obtain that 
\[
\bbeta_{s}=\frac{\int_{\vor_{s}(\bbeta)}\x f(\x)\ddup\x}{\int_{\vor_{s}(\bbeta)}f(\x)\ddup\x}=\frac{\int_{\ball_{s}}\x f(\x)\ddup\x}{\int_{\ball_{s}}f(\x)\ddup\x}=\bbeta_{s}^{*},\qquad\forall s\in[k].
\]
thereby proving that $\bbeta^{*}$ is the only global minimum.
\end{proof}

\section{\label{sec:proof_existence}Proof of Proposition~\ref{prop:existence}}

In this section we prove Proposition~\ref{prop:existence}, which stats that under the one-dimensional Stochastic Ball Model in Figure~\ref{fig:local-min-3-intervals}, the solution $\bbeta=(\beta_1,\beta_2,\beta_3)=(-2-\frac{\rad}{2}, -2+\frac{\rad}{2}, 1)$ is a local minimum of the $k$-means objective function $G$.

\begin{proof}
Observe that 
$\vor_{1}(\bbeta)=(-\infty,-2]$, $\vor_{2}(\bbeta)=[2,\frac{-1+\rad/2}{2}]$, $\vor_{3}(\bbeta)=[\frac{-1+\rad/2}{2},\infty]$, $\partial_{1,2}(\bbeta)=-2$
and $\partial_{2,3}(\bbeta)=\frac{-1+\rad/2}{2}$. It is easy to see that for any $\bb=(b_{1},b_{2},b_{3})\in\real^{3}$  in a small neighborhood of $\bbeta$, $\partial_{2,3}(\bb)$ remains strictly between $-2+\rad$ and $-\rad$, and $\partial_{1,2}(\bb)$ remains strictly between $-2-\rad$ and $-2+\rad$. Therefore, for any such $\bb$ we can explicitly write down its objective value:
\[
G(\bb)=\int_{-2-\rad}^{\frac{b_{1}+b_{2}}{2}}(x-b_{1})^{2}\ddup x+\int_{\frac{b_{1}+b_{2}}{2}}^{-2+\rad}(x-b_{2})^{2}\ddup x+\int_{-\rad}^{\rad}(x-b_{3})^{2}\ddup x+\int_{2-\rad}^{2+\rad}(x-b_{3})^{2}\ddup x.
\]
We compute the derivative and Hessian for $G$ at $\bb$:
\begin{align*}
\nabla_{\bb}G & =\left[\begin{array}{c}
-2\int_{-2-\rad}^{\frac{b_{1}+b_{2}}{2}}(x-b_{1}) \ddup x\\
-2\int_{\frac{b_{1}+b_{2}}{2}}^{-2+\rad}(x-b_{2})\ddup x\\
-2\int_{-\rad}^{\rad}(x-b_{3})\ddup x-2\int_{2-\rad}^{2+\rad}(x-b_{3})\ddup x
\end{array}\right],\\
\nabla_{\bb}^{2}G & =\left[\begin{array}{ccc}
\frac{b_{1}-b_{2}}{2}+2(\frac{b_{1}+b_{2}+2+\rad}{2}) & \frac{b_{1}-b_{2}}{2} & 0\\
\frac{b_{1}-b_{2}}{2} & \frac{b_{1}-b_{2}}{2}+2(-2+\rad-\frac{b_{1}+b_{2}}{2}) & 0\\
0 & 0 & 8\rad
\end{array}\right].
\end{align*}
Evaluating the above expressions at $b_{1}=\beta_{1} = 2 =\frac{\rad}{2}$, $b_{2}=\beta_{2}=-2+\frac{\rad}{2}$ and $b_{3}=\beta_{3}=1$, we find that the derivative vanishes and the Hessian is positive definite:
\begin{align*}
\nabla_{\bb}G\big|_{\bb=\bbeta}&=0,\\
\nabla_{\bb}^{2}G\big|_{\bb=\bbeta}
&=\left[\begin{array}{ccc}
1.5\rad & -0.5\rad & 0\\
-0.5\rad & 1.5\rad & 0\\
0 & 0 & 8\rad
\end{array}\right] \succ 0.
\end{align*}
Therefore, $\bbeta$ is indeed a local minimum of $G$.
\end{proof}

\section{Proofs for Section \ref{sec:prelim}}

In this section, we prove the technical lemmas stated in Section \ref{sec:prelim}.

\subsection{\label{sec:proof-differentiability}Proof of Lemma~\ref{lem:differentiability}}
\begin{proof}
Our goal is to derive the existence and expression of the derivative of the function 
\[
H^{\bv}(t):=G(\bbeta+t\bv)=\int_{\x}\min_{i\in[k]}\|\x-\bbeta_{i}-t\bv_{i}\|^{2}f(\x)\ddup\x,
\]
at $t=0$. We make use of the following measure-theoretic version of the Leibniz integral
rule.
\begin{proposition}[Leibniz's integral rule]
\label{prop:Leibniz-integral-rule}Let $T$ be an open subset of $\real$, and $X$ be a measure space. Suppose $g:T\times X\to\real$ satisfies the following conditions: (i) $g(t,\x)$ is a Lebesgue-integrable function of $\x$ for each $t\in T$; (ii) for almost all $\x\in X$, the partial derivative $\frac{\partial}{\partial t}g(t,\x)$ exists for all $t\in T$; (iii) There is an integrable function $\theta:X\to\real$ such that $|\frac{\partial}{\partial t}g(t,\x)|\leq\theta(\x)$ for all $t\in T$ and almost every $\x\in X$, then we have
\[
\frac{\ddup}{\ddup t}\int_{X}g(t,\x)\ddup\x=\int_{X}\frac{\partial}{\partial t}g(t,\x)\ddup\x.
\]
\end{proposition}
We verify the above three conditions in the proposition for $H^{\bv}$. Without loss of generality, assume that $\|\bv_{i}\|\le1,\forall i\in[k]$. Let $\Delta:=\min_{i\neq j}\|\bbeta_{i}-\bbeta_{j}\|$, which satisfies $\Delta>0$ by the assumption that $\{\bbeta_{j}\}_{j=1}^{k}$ are pairwise distinct. For condition (i), we see that the function $g(t,\x):=\min_{i\in[k]}\|\x-\bbeta_{i}-t\bv_{i}\|^{2}f(\x)$ is integrable in $\x$ for each bounded $t$, since the density $f$ has bounded second moment. For condition (ii), note that when $t\in T:= [-\frac{\Delta}{4},\frac{\Delta}{4}]$, the perturbed solution $\bbeta+t\bv$ remains pairwise disjoint, hence the Voronoi boundary $\partial(\bbeta+t\bv)$ has measure 0. For all $t\in T$ and  all $\x\notin\partial(\bbeta+t\bv)$, the minimizer in the definition of $g(t+\epsilon,\x)$ remains fixed when $|\epsilon|$ is sufficiently small, hence the partial derivative $\frac{\partial}{\partial t}g(t,\x)$ exists at all $t\in T$ and satisfies 
\begin{align}
\x\in\vor_{i}(\bbeta+t\bv)\implies\frac{\partial}{\partial t}g(t,\x)=-2\langle\bv_{i},\x-\bbeta_{i}-t\bv_{i}\rangle f(\x) \label{eq:vor-derivative}.\end{align}

Finally for  condition (iii), for each $\x\in\text{support}(f)=\cup_{s\in[k]}\ball_{s}(\rad)$, we have the bound $\left|\langle\bv_{i},\x-\bbeta_{i}-t\bv_{i}\rangle\right|\le\max_{s\in[k]}\|\bbeta_{s}^{*}\|+\rad+\|\bbeta_{i}\|+\frac{\Delta}{4}$ when $t\in T$, hence $|\frac{\partial}{\partial t}g(t,\x)|$ is bounded by an integrable function. Applying the Leibniz's integral rule and equation~\eqref{eq:vor-derivative}, we obtain that 
\[
\frac{\ddup}{\ddup t}H^{\bv}(0)=\int_{\x}\frac{\partial}{\partial t}g(0,\x)\ddup\x=-\sum_{i=1}^{k}\int_{\vor_{i}(\bbeta)}2\langle\bv_{i},\x-\bbeta_{i}\rangle f(\x)\ddup\x
\]
as claimed.
\end{proof}

\subsection{\label{sec:proof-necessary-condition}Proof of Lemma \ref{lem:necessary}}
\begin{proof}
In view of the decomposition~(\ref{eq:H_decompose}) and Remark~\ref{rem:non-positive}, we have the following upper bound for $H^{\bv}$:
\begin{equation}
H^{\bv}(t) \le h^{\bv}(t):=\sum_{i=1}^{k}\int_{\vor_{i}(\bbeta)}\|\x-\bbeta_{i}-t\bv_{i}\|^{2}f(\x)\ddup\x,\label{eq:sum-over-Voronoi-sets}
\end{equation}
which satisfies $h^{\bv}(0)=H^{\bv}(0)$. Since $\bbeta$ is a local minimum of $G$, we know that $t=0$ is local minimum of $H^{\bv}$ for all $\bv$ , hence Lemma~\ref{lem:core-upper-bound} ensures that $t=0$ is also a local minimum of $h^{\bv}$.

Suppose that we have $\bbeta_{1}=\bbeta_{2}$ and $\vor_{1}(\bbeta)=\vor_{2}(\bbeta)$ has a positive measure with respect to $f$, and that all other $\bbeta_{j},j\ge3$ are pairwise distinct and different from $\bbeta_{1}$ and $\bbeta_{2}$. We may partition $\vor_{1}(\bbeta)=\vor_{2}(\bbeta)$ into two disjoint sets $S_{1}$ and $S_{2}$, each with positive measure. For $i\in\{1,2\}$ denote by $\bs_{i}:=\frac{\int_{S_i}\x f(\x)\ddup\x}{\int_{S_i}f(\x)\ddup\x}$ the center of mass of $S_{i}$ with respect to $f$. We can choose the partition in such a way that $\bs_{1}\neq\bbeta_{1}$ and $\bs_{2}\neq\bbeta_{2}$. Fix a direction $\bv=(\bv_{1},\bv_{2},\bzero,\ldots,\bzero)$ with $\bv_{1}=\bs_{1}-\bbeta_{1}$ and $\bv_{2}=\bs_{2}-\bbeta_{2}$. In this case the upper bound $h^{\bv}$ can be written as 
\begin{align*}
h^{\bv}(t) & =\int_{S_{1}}\|\x-\bbeta_{1}-t\bv_{1}\|^{2}f(\x)\ddup\x+\int_{S_{2}}\|\x-\bbeta_{2}-t\bv_{2}\|^{2}f(\x)\ddup\x+\text{constant}\\
 & =\int_{S_{1}}\|\x-\bs_{1}\|^{2}\ddup\x+\int_{S_{1}}\|\bs_{1}-\bbeta_{1}-t\bv_{1}\|^{2}f(\x)\ddup\x\\
 & \qquad+\int_{S_{2}}\|\x-\bs_{2}\|^{2}\ddup\x+\int_{S_{2}}\|\bs_{2}-\bbeta_{2}-t\bv_{2}\|^{2}f(\x)\ddup\x+\text{constant}.
\end{align*}
(In the calculation above we have avoided double counting the contribution from $\vor_{1}(\bbeta)=\vor_{2}(\bbeta)$.) With the above choices of $\bv_{1}$ and $\bv_{2}$, we see that $h^{\bv}(0) > h^{\bv}(t)$ for all $t\in (0,1)$ and hence $t=0$ is not a local minimum of $h^{\bv}$, which is a contradiction. Therefore, we must have $\bbeta_{1}\neq\bbeta_{2}$ whenever $\vor_i(\bbeta) \cup \vor_j(\bbeta)$ has a positive measure. The more general statement in Lemma \ref{lem:necessary} can be established in a similar manner.

Now suppose that $\bbeta_{i}$ has a Voronoi set $\vor_{i}(\bbeta)$
with a positive measure. In this case the center of mass $\bc_{i}:=\frac{\int_{\vor_{i}(\bbeta)}\x f(\x)\ddup\x}{\int_{\vor_{i}(\bbeta)}f(\x)\ddup\x}$ is well-defined. Choose the direction $\bv=(\bzero,\ldots,\bzero,\bc_{i}-\bbeta_{i},\bzero,\ldots,\bzero)$. Since $t=0$ is a local minimum of $H^{\bv},$ its derivative must vanish at $t=0$. Using the derivative expression from Lemma~\ref{lem:differentiability},\footnote{Lemma~\ref{lem:differentiability} is applicable for the following reason: we can ignore those $\vor_{i}(\bbeta)$'s with zero measure in the integrals defining $G$ and $H^{\bv}$, in which case we have just established that $\bbeta$ must have pairwise distinct components and thus satisfy the premise of Lemma~\ref{lem:differentiability}.} we obtain that 
\begin{align*}
0=\frac{\ddup}{\ddup t}H^{\bv}(0) & =-\int_{\vor_{i}(\bbeta)}2\langle\bv_{i},\x-\bbeta_{i}\rangle f(\x)\ddup\x\\
 & =-2\langle\bv_{i},\bc_{i}-\bbeta_{i}\rangle\int_{\vor_{i}(\bbeta)}f(\x)\ddup\x\\
 & =-2\|\bc_{i}-\bbeta_{i}\|^{2}\int_{\vor_{i}(\bbeta)}f(\x)\ddup\x,
\end{align*}
where the last step follows from our choice of $\bv$. Since $\int_{\vor_{i}(\bbeta)}f(\x)\ddup\x$ is the measure of $\vor_i(\bbeta)$ and positive, we must have $\bbeta_{i}=\bc_{i}$ as claimed. %
\end{proof}

\section{Proofs for Section~\ref{sec:proof_main_ball}}

We state and prove several technical lemmas that are used in Section~\ref{sec:proof_main_ball}.

\subsection{Proof of Lemma~\ref{lem:prox-center}}
\label{sec:proof-prox-center}

Recall that $m_{i,s}$ and $\bc_{i,s}$ denote the mass and the center of mass of the set $\vor_{i}$ with respect to the density $f_s$. We similarly define 
\begin{align*}
    \widetilde{m_{i,s}} =  \sum_{s^{\prime}\in [k]:s^{\prime}\neq s} m_{i,s^{\prime}} 
    \qquad\text{and}\qquad
    \widetilde{\bc_{i,s}} =  \frac{\sum_{s^{\prime}\in [k]:s^{\prime}\neq s } m_{i,s^{\prime}}\bc_{i,s^{\prime}}}{\sum_{s^{\prime}\in [k]:s^{\prime}\neq s} m_{i,s^{\prime}}},
\end{align*}
which are the mass and the center of mass of the set $\vor_{i}$ with respect to the density $\sum_{s'\neq s} f_{s'}$. With this notation, the local minimum $\bbeta$ must satisfy the necessary condition
\begin{align}
    \bbeta_{i}=\frac{m_{i,s}\bc_{i,s}+\widetilde{m_{i,s}}\widetilde{\bc_{i,s}} }{m_{i,s}+\widetilde{m_{i,s}}}, \label{eq:center-of-mass-condition}
\end{align}
which follows from Lemma~\ref{lem:necessary} and the text thereafter. Rearranging the expression~\eqref{eq:center-of-mass-condition} gives
\[
\widetilde{\bc_{i,s}} = \frac{(\widetilde{m_{i,s}}+m_{i,s})\bbeta_{i,s}-m_{i,s}\bc_{i,s}}{\widetilde{m_{i,s}}}.
\]
It then follows from the triangle inequality that
\begin{align}
    \|\widetilde{\bc_{i,s}}-\bbeta_{i}\| = & \frac{m_{i,s}}{\widetilde{m_{i,s}}}\|\bc_{i,s}-\bbeta_{i}\| 
    \leq  \frac{m_{i,s}}{\widetilde{m_{i,s}}}(\|\bc_{i,s}-\bbeta_{s}^*\|+\|\bbeta_{s}^*-\bbeta_{i}\|).
    \label{eq:bound-shifted-center}
\end{align}

We are now ready to prove Lemma~\ref{lem:prox-center}, whose assumption states that $\rho_s(\partial_{j,\ell})>\lambda=\frac{c}{\sqrt{\rad\deltamax}}$ for some $(s,j,\ell) \in T_i \times [k]\times [k]$. Observation~\ref{pro1} ensures that such an $s$ is unique, hence for all other $s'\in T_i \setminus\{s\}$, we must have  $\rho_{s'}(\partial_{j,\ell}) \le \lambda, \forall (j,\ell) $. Observation~\ref{pro2} ensures that for all these $s'$, if $\bbeta_{s'}^* \in \vor_i$ then $s' \in \Wint_i $. In view of these properties and equation~\eqref{eq:center-of-mass-condition}, we can see that $\widetilde{\bc_{i,s}}$ is similar to $\bbeta_i$ except that the density of the $s$-th true cluster is ignored. Therefore, we can follow the same arguments for proving Part~2(b) of Theorem~\ref{prop:ball-family-bounds} to obtain that 
\begin{align}
    \|\widetilde{\bc_{i,s}}-\bb_{i}^{-}\|\leq & \frac{k\rad}{1-k^{2}\lambda\rad}+\frac{k\rad(k^{2}\lambda\rad)}{(1-k^{2}\lambda\rad)^{2}}+\frac{2k^{2}\lambda\rad}{1-k^{2}\lambda\rad}\deltamax
    \leq  \deltamax\frac{8ck^2}{\sqrt{\snrmax}}. \label{eq:step-bound2}
\end{align}
On the other hand, under the assumption of the lemma, Part~1 of Theorem~\ref{prop:ball-family-bounds} ensures that 
\begin{align}
\|\bbeta_i-\bbeta_s^*\| \leq \frac{k}{\lambda}+3\rad\leq \deltamax\frac{4k}{c\sqrt{\snrmax}}. \label{eq:step-bound}  
\end{align} 

We proceed by bounding $\|\bb_{i}^{-}-\bbeta_s^*\| $ as follows:
\begin{align}
\|\bb_{i}^{-}-\bbeta_s^*\| 
 \leq & \|\bb_{i}^{-}-\widetilde{\bc_{i,s}}\|+\|\widetilde{\bc_{i,s}}-\bbeta_{i}\|+\|\bbeta_{i}-\bbeta_s^*\| \nonumber\\
 \overset{\text{(i)}}{\leq} & \|\bb_{i}^{-}-\widetilde{\bc_{i,s}}\|+ \frac{m_{i,s}}{\widetilde{m_{i,s}}}(\|\bc_{i,s}-\bbeta_{s}^*\|+\|\bbeta_{s}^*-\bbeta_{i}\|) +\|\bbeta_{i}-\bbeta_s^*\| \nonumber \\
 \overset{\text{(ii)}}{\leq } & \deltamax\frac{8ck^{2}}{\sqrt{\snrmax}} + \frac{m_{i,s}}{\widetilde{m_{i,s}}} \cdot \frac{\deltamax}{\snrmax} + \left(\frac{m_{i,s}}{\widetilde{m_{i,s}}} +1\right) \deltamax\frac{4k}{c\sqrt{\snrmax}}, \nonumber 
\end{align}
where in step (i) follows from the bound~\eqref{eq:bound-shifted-center}, and step (ii) follows from the bounds~\eqref{eq:step-bound2} and \eqref{eq:step-bound} as well as the fact that $\bc_{i,s}\in \ball_s$ so $\|\bc_{i,s}-\bbeta_s^*\|\leq \rad = \frac{\deltamax}{\snrmax}$. Now, note that since $|\Wint_i\setminus\{s\}|\geq 2$ by assumption, there exists some $s'\in \Wint_i \subseteq T_i$ such that $s'\neq s$. We have established above that this $s'$ must satisfy $\rho_{s'}(\partial_{j,\ell}) \le \lambda, \forall (j,\ell) $, hence applying Part~2 of Theorem~\ref{prop:ball-family-bounds} we obtain that $m_{i,s'} = \mP_{s'}(\vor_i) \ge 0.5$, which further implies $\frac{m_{i,s}}{\widetilde{m_{i,s}}}\leq 2$. Continuing from the above display equation, we obtain
\begin{align}
\|\bb_{i}^{-}-\bbeta_s^*\|
\leq & \deltamax\frac{8ck^{2}}{\sqrt{\snrmax}} + 3\deltamax\frac{4k}{c\sqrt{\snrmax}}+\deltamax\frac{2}{\snrmin}
 \leq \deltamax\frac{10ck^{2}}{\sqrt{\snrmax}} \label{eq:semi-final},
\end{align}
where the last step follows from the assumption that $c>3$ and $\snrmax\geq 4ck^2$. Combining the inequalities~\eqref{eq:step-bound} and~\eqref{eq:semi-final}, we obtain
\begin{align}
\| \bb_{i}^{-} -  \bbeta_{i}\|\leq \| \bb_{i}^{-} -  \bbeta_s^*\|+\| \bbeta_s^* -  \bbeta_{i}\|\leq \deltamax\frac{11ck^{2}}{\sqrt{\snrmax}}, \label{eq:bound-to-center-exclude-s}
\end{align}
thereby proving the first bound in Lemma~\ref{lem:prox-center}.

To prove the second bound in Lemma~\ref{lem:prox-center}, we observe that by  definition of $\bb_{i}^{+}$ and $\bb_{i}^{-}$, there holds
\begin{align*}
\bb_{i}^{+} := & \frac{1}{1+|\Wint_i \setminus \left\{s\right\}|} \left(\sum_{s^{\prime}\in \Wint_i \setminus \left\{s\right\}} \bbeta_{s^{\prime}}^* +\bbeta_s^*\right) 
 = \frac{|\Wint_i \setminus \left\{s\right\}|}{1+|\Wint_i \setminus \left\{s\right\}|} \bb_{i}^{-} + \frac{1}{1+|\Wint_i \setminus \left\{s\right\}|} \bbeta_s^*. 
\end{align*}
whence $  \| \bb_{i}^{+} -  \bbeta_s^*\|
  =  \frac{|\Wint_i \setminus \left\{s\right\}|}{1+|\Wint_i \setminus \left\{s\right\}|}\|\bb_{i}^{-}-\bbeta_s^*\|. $ It follows that 
\begin{align*}
\| \bb_{i}^{+} -  \bbeta_{i}\|
\leq \| \bb_{i}^{+} -  \bbeta_s^*\|+\| \bbeta_s^* -  \bbeta_{i}\|
\leq \| \bb_{i}^{-} -  \bbeta_s^*\|+\| \bbeta_s^* -  \bbeta_{i}\|
\leq \deltamax\frac{11ck^{2}}{\sqrt{\snrmax}},
\end{align*}
where the last step follows from equation~\eqref{eq:bound-to-center-exclude-s}.

It remains to show that $|\Wint_i \setminus \left\{s\right\}|\geq 2$. Note that $\Wint_i \neq \emptyset$ under the assumption $|W_i \setminus \left\{s\right\}|\geq 1$ of the lemma. For the sake of deriving a contradiction, assume that $ W_i \setminus \left\{s\right\} = \left\{s'\right\}$, in which case $\bb_{i}^{-} =\bbeta_{s'}^*$. It then follows from inequality~\eqref{eq:semi-final} that $ \|\bbeta_{s'}^* -\bbeta_{s}^*\|\leq \deltamax\frac{10ck^2}{\sqrt{\snrmax}}$, contradicting the separation assumption on $\snrmin$ in Theorem~\ref{thm:main_ball}. This completes the proof of Lemma~\ref{lem:prox-center}.

\subsection{\label{sec:control-vol}Controlling the Volume}

In this section, we show that the intersection of a Voronoi
set and a ground truth cluster must be small if (i) the true center is not in the Voronoi set and (ii) the intersection of the true cluster and the boundary of the Voronoi set  is small. This is formalized in the following lemma.
\begin{lemma}[Controlling the volume of intersection]
 \label{lem:prob-intersection}
Let $\mu$ be the uniform distribution on $\ball_{\bzero}(\rad)$. Let $P$ be a closed polyhedron with at most $k$ facets satisfying $\bzero\not\in\textup{int}(P)$. If each facet $F$ of $P$ satisfies $\frac{1}{\rad^d V_d}\revol(F\cap \ball_{\bzero}(\rad)) \le \lambda$, then we have $\mu(P)\leq k\lambda\rad$.
\end{lemma}

\begin{proof}
Introduce the shorthand $\ball:=\ball_{\bzero}(\rad)$.  We may assume that $P\cap\ball\neq\emptyset$, because otherwise the lemma is trivially true. We claim that one may shift the polyhedron  $P$ by a distance $\rad$ so that its intersection with the ball $\ball$ has zero measure. That is, there exists a unit vector $\bv$ such that $\big(P+ (\rad+\epsilon) \bv \big) \cap \ball = \emptyset $ for all $\epsilon >0$. We further claim that $\bv$ can be chosen in such a way that the intersection $P\cap\ball$  is enclosed by the original boundary and the shifted boundary; that is, $P\cap\ball \subseteq (\partial P \cap \ball) + L_\rad $, where we  $L_\rad := \{t\bv: t\in[0,\rad] \}$ is a line segment. Figure~\ref{fig:Shifting-the-boundary-bound-volume} provides an illustration of these two claims, whose proof is deferred to the end of this section. Therefore, we have the bound 
\begin{align*}
    \mu(P) = \frac{\vol(P\cap\ball)}{\vol(\ball)}
    \le \frac{\vol\big((\partial P \cap \ball) + L_\rad)}{\vol(\ball)}
    \le  \frac{ \rad \sum_{F\in \mathcal{F}} \revol(F\cap \ball) }{\rad^d V_d} \le \rad k\lambda,
\end{align*}
where $\mathcal{F}:= \{F: F \text{ is a facet of }P\}$ satisfies $|\mathcal{F}|\le k$ by assumption. This completes the proof of the lemma.

\begin{figure}
\centering
\includegraphics[scale=0.4, trim=0 30 0 40]{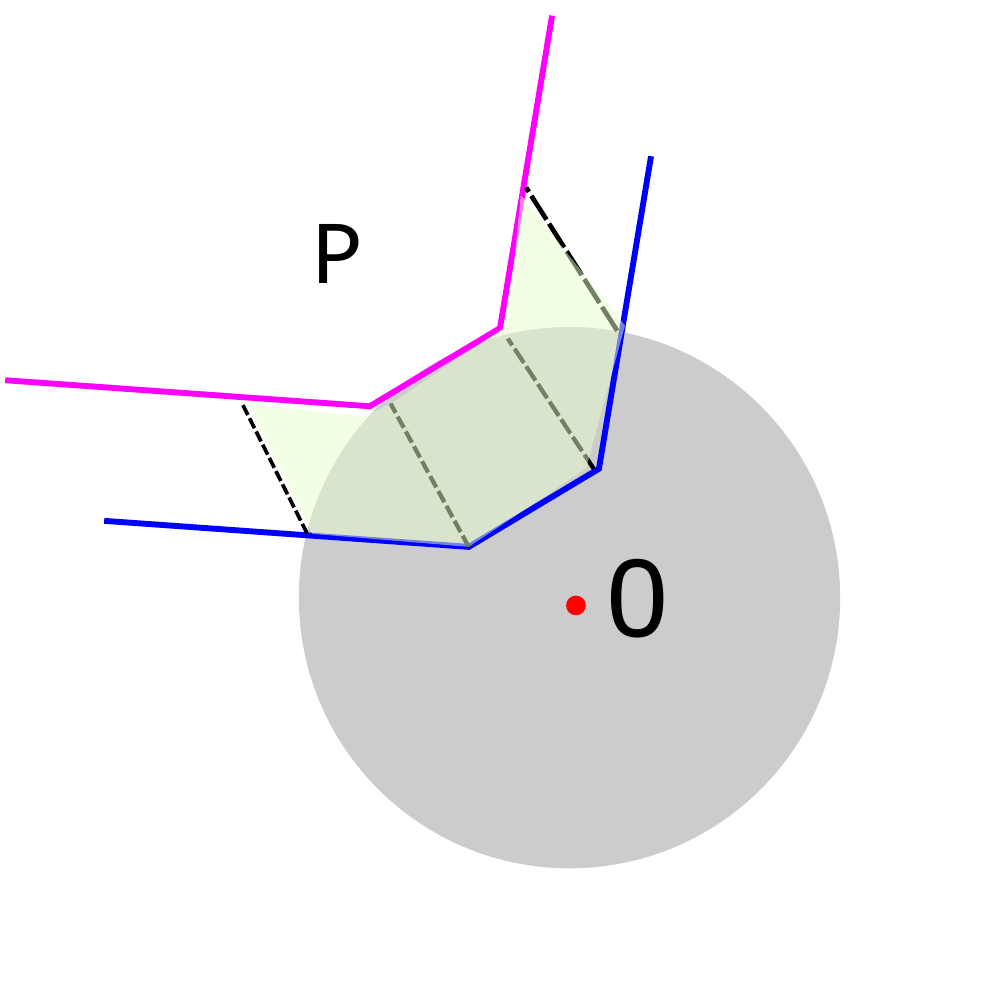}
\caption{\label{fig:Shifting-the-boundary-bound-volume}Shifting the boundary
of the polyhedron $P$ to bound the volume of $P\cap\ball_{\bzero}(\rad)$.}
\end{figure}

Let us prove the two claims  above. Because $P$ is convex and $\bzero \notin \textup{int}(P)$, the separating hyperplane theorem ensures that there exists some unit vector $\bv $ such that $\langle\x,\bv \rangle  \ge 0, \forall \x \in P$. Therefore, for all $\x \in P$ and $\epsilon >0$, we have 
\begin{align*}
    \| \x + (\rad + \epsilon) \bv \| ^2 
 &= \|\x\|^2 + 2 (\rad + \epsilon) \langle\x,\bv\rangle + (\rad + \epsilon)^2 \|\bv\|^2 \\
 &\ge 0 + 0 + (\rad + \epsilon)^2 
 > \rad^2,
\end{align*}
whence $\x + (\rad+\epsilon)\bv \notin \ball $, proving the first claim. To prove the second claim, fix an arbitrary $\x \in P\cap \ball$ and consider the half line $\ell := \{\x - t\bv: t \ge 0\} $. Note that $\ell$ must intersect the boundary $\partial P$; otherwise we would have $\ell \subseteq P$ and hence the separating hyperplane property implies that $ \langle \x -t\bv, \bv \rangle \ge 0$ for all $t\ge 0$, which cannot hold as $\bv$ has unit norm. Since $P$ is convex, $\ell$ intersects $\partial P$ at a unique point, say $\x_0 =\x - t_0 \bv$. We must have $ t_0 \le \rad $; otherwise we would have $\x = \x_0 + t_0 \bv \in P + (\rad + \epsilon) \bv $ for some $\epsilon>0$ and hence $\x \notin \ball$ by the first claim, which is a contradiction. Using the separating hyperplane property $0 \le \langle \x-t_0\bv,\bv \rangle \le \langle \x,\bv \rangle $ again, we have 
$$\|\x_0\|^2 = \|\x - t_0 \bv \|^2 
= \|\x\|^2 + \langle -t_0\bv, \x \rangle + \langle \x - t_0 \bv, -t_0 \bv \rangle 
\le \rad^2 + 0 + 0
$$ 
and thus $\x_0 \in \ball \cap \partial P$. Combining pieces, we conclude that  $\x = \x_0 + t_0 \bv \in (\partial P \cap \ball) + L_\rad $. As $\x \in P \cap \ball$ is arbitrary, we have  $P\cap\ball \subseteq (\partial P \cap \ball) + L_\rad $ as claimed. 
\end{proof}

\subsection{\label{sec:control_dist_center}Controlling the distance to the center}

In this section, we prove the following result:
\begin{lemma}[Bound on the center of mass]
\label{lem:bound-dist-2-center-of-mass} Let $\mu$ be the uniform
distribution over the ball $\ball_{\bzero}(\rad)\subset\real^{d}$.
Suppose that a subset $S\subset\real^{d}$ has probability measure $\mu(S)>0$.
Let $\bc_{S}$ be the center of mass of the set $S$ with respect to
$\mu$. We have the bound
\begin{equation*}
\|\bc_{S}\|\leq\frac{\rad\cdot\mu(\mathbb{\real}^{d}\setminus S)}{\mu(S)}.\label{eq:bound-center-mass}
\end{equation*}
\end{lemma}
\begin{proof}
Recall the expression of the center of mass $\bc_{S}=\frac{\int\x\indic_{S}(\x)\ddup\mu}{\mu(S)}$.
We have 
\begin{align*}
\mu(S)\cdot\|\bc_{S}\| & \overset{}{=}\|\int\x\indic_{S}(\x)\ddup\mu\|\\
 & \overset{\text{(i)}}{=}\|-\int\x\indic_{\real^d\setminus S}(\x)\ddup\mu\|\\
 & \overset{\text{(ii)}}{\le}\int\|\x\|\indic_{\real^d\setminus S}(\x)\ddup\mu\\
 & \overset{\text{(iii)}}{\le}\rad\mu(\real^{d}\setminus S),
\end{align*}
where step (i) holds since $\mu$ has mean $\bm{0}$, and step (ii)
holds by the Jensen's inequality, and step (iii) holds because $\|\x\|\leq\rad$
for all $\x\in \text{support}(\mu) =\ball_{\bzero}(\rad)$. Rearranging the inequality proves
the desired bound. 
\end{proof}

\begin{lemma}[Bound on the center of mass, Gaussian case]
\label{lem:bound-dist-2-center-of-mass-gaussian} Let $\mu$ be the
Gaussian distribution $\mathcal{N}(\bm{0},\sigma^{2}\bm{I}_d)$.
Suppose that a subset $S\subset\real^{d}$ has probability measure $\mu(S)>0$.
Let $\bc_{S}$ be the center of mass of the set $S$ with respect to
$\mu$. We have the bound
\begin{equation*}
\|\bc_{S}\|\leq\frac{2\sigma\cdot\sqrt{\mu(\mathbb{\real}^{d}\setminus S)}}{\mu(S)}.\label{eq:bound-center-mass-1}
\end{equation*}
\end{lemma}

\begin{proof}
Recall the variational characterization of the center of mass $\bc_{S}$:

Since $\bzero$ is the mean of $\mu$, we have 
\begin{align}
0 \le \|\bc_{S}\|^{2}
= & \int\|\x-\bc_{S}\|^{2} \ddup\mu-\int\|\x\|^{2}\ddup \mu \nonumber \\
= & \int\|\x-\bc_{S}\|^{2} \indic_{S}(\x) \ddup\mu-\int\|\x\|^{2} \indic_{S}(\x) \ddup\mu+\int\|\x-\bc_{S}\|^{2} \indic_{\real^d\setminus S}(\x)\ddup\mu-\int\|\x\|^{2} \indic_{\real^d \setminus S}(\x) \ddup\mu\nonumber \\
\overset{\text{(i)}}{\leq} & \int\|\x-\bc_{S}\|^{2} \indic_{\real^d \setminus S}(\x) \ddup\mu-\int\|\x\|^{2} \indic_{\real^d \setminus S}(\x) \ddup\mu\nonumber \\
= & \mu(\real^{d}\backslash S)\|\bc_{S}\|^{2}-2\int\langle\x,\bc_{S}\rangle  \indic_{\real^d \setminus S}(\x) \ddup\mu.\label{eq:sec22upperbound}
\end{align}
where step (i) follows from the variational characterization of the center of mass 
\[
\bc_{S}=\text{argmin}_{\z\in\real^{d}}\int \|\x-\z\|^{2} \indic_{S}(\x) \ddup\mu.
\]
Rearranging equation~\eqref{eq:sec22upperbound} gives
\begin{align*}
\mu(S)\|\bc_{S}\|^{2}
&\leq -2\int\langle\x,\bc_{S}\rangle \indic_{\real^d \setminus S}(\x) \ddup\mu \\
& \overset{\text{(ii)}}{\leq} 2\|\bc_{S}\|\sqrt{\int \Big\langle\x,\frac{\bc_{S}}{\|\bc_{S}\|}\Big\rangle^{2} \ddup\mu} \cdot \sqrt{\int \indic_{\real^d \setminus S}(\x) \ddup\mu}\\
 & \overset{\text{(iii)}}{=} 2\|\bc_{S}\| \sigma \cdot \sqrt{\mu(\real^{d}\setminus S)},
\end{align*}
where step (ii) follows from Cauchy-Schwarz and step (iii) follows from the fact that any one-dimensional margin of $\mathcal{N}(\bzero,\sigma^2\bm{I}_d)$ is the univariate Gaussian distribution $\mathcal{N}(0,\sigma^2)$.  Rearranging the above equation proves the desired bound. 
\end{proof}

\section{\label{sec:proof_main_gaussian}Proof of Theorem~\ref{thm:main_gaussian}}

As before, we use $\mP$ to denote the probability measure with respect to $f$ and $\mP_{s}$ to denote the probability measure with respect to $f_{s}$, where $f$ is the density of the Gaussian mixture and $f_{s}$ is the density of the $s$-th Gaussian component.  Recall the population $k$-mean objective function defined in~\eqref{eq:pop-kmean}:
\begin{align*}
G(\bbeta) & =\int_{\x}\min_{j\in[k]}\|\x-\bbeta_{j}\|^{2}f(\x)\ddup\x
  =\frac{1}{k}\sum_{s=1}^{k}\int_{\x}\min_{j\in[k]}\|\x-\bbeta_{j}\|^{2}f_{s}(\x)\ddup\x.
\end{align*}

\paragraph{Reduction to lower dimensions:}

We first argue that it suffices to prove the theorem in dimension $d' \le 2k$. Once this is established, then the theorem for  $d > 2k$ dimensions can be deduced as follows. Suppose that $\bbeta^{*} \in \real^{d\times k}$ is the ground truth solution and $\bbeta \in \real^{d\times k}$ is a candidate solution. We may choose a coordinate system such that the first $d'=2k$ dimensions correspond to $\text{span}\{\bbeta^{*}_1,\ldots,\bbeta^{*}_k, \bbeta_1,\ldots,\bbeta_k\}$. In this case,  for each $i\in[k]$ we have $\bbeta^{*}_i = (\bbeta^{*\prime}_i,\bzero)$ and $\bbeta_i = (\bbeta'_i, \bzero)$ for some $\bbeta^{*\prime}_i, \bbeta'_i \in \real^{d'}$. Moreover, thanks to Gaussian's rotational invariance, the $d$-dimensional Gaussian mixture is a product distribution with respect to the first $d'$ dimensions and the last $d-d'$ dimensions, where the first $d'$-dimensional margin is itself a Gaussian mixture. Indeed, for any $\x = (\x', \z) \in \real^{d}$ with $\x'\in\real^{d'}$ and $\z \in \real^{d-d'}$, the density of the Gaussian mixture factorizes:
\begin{align*}
    f(\x) &\propto \frac{1}{k} \sum_{s=1}^k \exp\bigg(\frac{\| \x - \bbeta^{*}_s \|^2}{2\std^2}\bigg) \\
    &= \frac{1}{k} \sum_{s=1}^k \exp\bigg( \frac{\| (\x', \z) - (\bbeta^{*\prime}_s, \bzero) \|^2}{2\std^2}\bigg) \\
    &= \bigg[\frac{1}{k} \sum_{s=1}^k \exp\bigg( \frac{\| \x' - \bbeta^{*\prime}_s \|^2}{2\std^2}\bigg) \bigg]  \cdot \exp\bigg(\frac{\|\z\|^2}{2\std^2}\bigg).
\end{align*}
Now, if $\bbeta$ is a local minimum of $G$, then $\bbeta'$ is also a local minimum of $G$ restricted to the first $d'=2k$ dimensions. Applying the theorem with dimension $d'$, we obtain bounds on the quantities $\|\bbeta'_i - \bbeta^{*\prime}_s \| $, $\| \bbeta'_i - \sum_{s\in S} \bbeta^{*\prime}_s \|$ and $\mP (\vor_i(\bbeta'))$. We claim that these three quantities are equal to $\|\bbeta_i - \bbeta^{*}_s \| $, $\| \bbeta_i - \sum_{s\in S} \bbeta^{*}_s \|$ and $\mP (\vor_i(\bbeta))$, respectively. Indeed, the first two equalities are immediate under our coordinate system; the last equality holds because the Gaussian mixture factorzes (shown above) and so do the Voronoi sets:  $\vor_i(\bbeta) = \vor_i(\bbeta') \times \real^{d-d'}$. We conclude that the same collection of bounds hold in dimension $d$ as well. In the rest of the proof, we can safely assume that $d \le 2k$, in which case $\min\{2k, d\} = d$.\\

As in the proof for the Stochastic Ball Model, we establish a general result, analogous to analogue of Theorem~\ref{prop:ball-family-bounds}, that provides a family of bounds parametrized by two numbers $\lambda, t>0$. To state this result, we introduce some additional notation. Let
\[
\ball_{s}(\rad):=\{\x\in\real^{d}:\|\x-\bbeta_{s}^{*}\|\leq \rad\}.
\]
denote the ball centered at $\bbeta_s^*$ with radius $\rad=t\sigma\sqrt{d}$. Recall that for each $i,j,s\in[k]$, the  Voronoi boundary $\partial_{i,j}$ lies in a $(d-1)$-dimensional affine subspace $\mathcal{L}$. Since the distribution $f_s$ of the $s$-th component is rotationally invariant, we may assume WLOG that $\mathcal{L}=\{\x\in \real^{d}: x_1 = z \}$ for some number $z$. Accordingly, we define the quantity
\begin{align}
    \rho_s(\partial_{i,j}):=\int  \indic\big\{(z, x_2, \ldots, x_d)\in\partial_{i,j}\cap \ball_{s}(\rad)\big\} \cdot f_s(z, x_2, \ldots, x_d) \, \ddup x_2\ldots \ddup x_d \label{eq:general-rho},
\end{align}
which is a measure of the relative probability mass of the Voronoi boundary $\partial_{i,j}$ when restricted to the ball $\ball_{s}(\rad)$.\footnote{Note that when $f_s$ is the uniform distribution over $\ball_s(\rad)$, the definition here reduces to $\rho_s(\partial_{i,j}) = \frac{1}{V_d\rad^d}\revol(\partial_{i,j}\cap \ball_s(\rad))$ and hence is consistent with our previous definition in the ball model.}
Also recall the function $\tail(\cdot)$ defined in the statement of Theorem~\ref{thm:main_gaussian}, which satisfies $\tail(t) = 2\exp(-t^{2}d/8)$ when $d\le 2k$.

With the above notations, we have the following result, which is an analogue of Theorem~\ref{prop:ball-family-bounds}. 
\begin{theorem}[Family of bounds for Gaussian] 
\label{prop:gaussian-family-bounds}
Under the Gaussian mixture model, let $\bbeta=(\bbeta_1,\ldots, \bbeta_k)$ be a local minima for the $k$-means objective function~G defined in~\eqref{eq:pop-kmean}. Let $\lambda>0$ and $t>0$ be two arbitrary fixed numbers and set  $\rad:=t\sigma\sqrt{d}$. For each $i\in[k]$, define the sets:
\begin{align*}
    T_i: = \big\{s\in [k]:\vor_i\cap \ball_s(\rad)\neq\emptyset\big\}
    \quad\text{and}\quad
    \Wint_i:=\big\{s\in [k]:\bbeta_s^* \in \textup{int}(\vor_i)\big\} \subseteq T_i.
\end{align*}
Then the following is true for each $i\in [k]$:
\begin{enumerate}
    \item If  $\rho_s(\partial_{j,\ell})>\lambda$ for some pair $(j,\ell)$ and $s\in T_i$, then 
    \begin{align*}
        \|\bbeta_i-\bbeta_s^*\|\leq \frac{k}{\lambda}+3\rad.
    \end{align*}
    \item For each $s\in T_i$, if $\rho_s(\partial_{j,\ell}) \le \lambda$ for all pairs $(j,\ell)$, then the following bounds hold 
    \begin{align*}
        \mP_s(\vor_i)\geq & 1-k^2\lambda\rad-\tail(t), \quad \forall s\in \Wint_i,\\
        \mP_s(\vor_i)\leq & k\lambda + \tail(t), \;\;\qquad\quad \forall s\in T_i \setminus \Wint_i.
    \end{align*}
    Furthermore, if $\rho_s(\partial_{j,\ell})\leq \lambda$ for all pair $(j,\ell)$ and and $s\in T_i$, then:
    \begin{enumerate}
        \item When $|\Wint_i|=0$, we have
    \begin{align*}
        \mP(\vor_i)\leq k\lambda\rad + \tail(t).
    \end{align*}
    \item When $|\Wint_i|>0$, we have 
    \begin{align*}
      \|\bbeta_i-\bb_i\| \leq\frac{4k\sigma}{1-k^{2}\lambda\rad-\tail(t)}+\frac{2k^{2}\sigma\sqrt{k^{2}\lambda\rad+\tail(t)}(k\lambda\rad+\tail(t))}{(1-k^{2}\lambda\rad-\tail(t))^{2}} +\frac{3k^{2}\lambda\rad+3k\tail(t)}{1-k^{2}\lambda\rad-\tail(t)}\deltamax,
    \end{align*} 
    where $\bb_i:= \frac{1}{|\Wint_i|}\sum_{s\in \Wint_i}\bbeta_s^*$.
    \end{enumerate}
\end{enumerate}
\end{theorem}
We prove Theorem~\ref{prop:gaussian-family-bounds} in  Section~\ref{sec:sketch-gaussian-family-bounds}. Note that Theorem~\ref{prop:gaussian-family-bounds} is similar to Theorem~\ref{prop:ball-family-bounds} except that  the error bounds here have an  additional error term $\tail(t)$.

The procedure for deriving the main Theorem~\ref{thm:main_gaussian} from Theorem~\ref{prop:gaussian-family-bounds} is the same as that for the Stochastic Ball Model. In particular, with $\rad$ fixed to be $t\sigma\sqrt{d}$, we set $\lambda = \frac{c}{\sqrt{\rad\deltamax}}$. The assumptions on $t$, $\snrmax$ and $\snrmin$ ensures that $\lambda k^2\rad<\frac{1}{4}$ and $\tail(t) = 2\exp(-t^2d/8)<\frac{1}{4}$. In this case, Observations~\ref{pro1} and~\ref{pro2} also hold in the current setting. Following the same arguments as in the proof of Theorem~\ref{thm:main_ball}, we complete the proof of Theorem~\ref{thm:main_gaussian} for the Gaussian model. We omit the details.

\subsection{Proof of Theorem~\ref{prop:gaussian-family-bounds}}
\label{sec:sketch-gaussian-family-bounds}

To establish Theorem~\ref{prop:gaussian-family-bounds} for the Gaussian model, we follow the same strategy for proving Theorem~\ref{prop:ball-family-bounds} for the ball model. The only technical difficulty is that each Gaussian component distribution has unbounded support. Our main idea is to identify a bounded ball, namely $\ball_s(\rad)$, that contains most of the probability mass of the $s$-th Gaussian component.
Using a standard concentration inequality for $\chi^{2}$ random
variables (e.g., \cite[Example 2.28]{wainwright2019book}), we know
that when $t>2$, there holds the tail bound
\begin{align}
  \mP_{s}\Big(\ball_{s}(\rad)^{\complement}\Big)\leq \tail(t) = 2\exp(-t^{2}d/8), \label{eq:concentration-bound}
\end{align}
where $S^{\complement}$ denotes the  complement of a set  $S\subseteq\real^{d}$. By restricting each $s$-th ground truth component to the ball $\mB_{s}(r)$ and treating the tail mass in equation~\eqref{eq:concentration-bound} as additional error terms, we can repeat most of the arguments used in the proof of the ball model. In what follows, we sketch the analysis and point out the minor modifications needed to adapt the proof of Theorem~\ref{prop:ball-family-bounds} to the Gaussian case.\\

The main step in the proof for the Ball model involves constructing smooth upper bounds for the function $W_{i\to j,s}^{\bv}+W_{j\to i,s}^{\bv}$, as done in Proposition \ref{prop:same-direction} and
Proposition \ref{prop:opposite-direction}. These two propositions still hold in the Gaussian case under the definition~\eqref{eq:general-rho} of the ``relative volume'' $\rho_s(\partial_{i,j})$. In particular, the value of the integral defining $W_{i\to j,s}^{\bv}+W_{j\to i,s}^{\bv}$ does not increase if we restrict integration to the small set subset ($\ball_{s}(\rad)$), as the integrand is non-positive. 
Consequently, we can establish the two key inequalities~(\ref{eq:relation-1}) and~(\ref{eq:relation-2}), restated below:
\begin{align}
d_{i,j}\cdot\rho_{s}(\partial_{i,j})  \leq\frac{k}{2}
\qquad\text{and}\qquad
\frac{D_{i,j,s}^{2}}{d_{i,j}}\cdot\rho_{s}(\partial_{i,j}) & \leq\frac{k}{2}.
\label{eq:relation-gaussian}
\end{align}

We can then  derive the structural properties of a local minimum $\bbeta$ from the inequalities~(\ref{eq:relation-gaussian}).
As in the proof of Theorem~\ref{prop:ball-family-bounds}, for each $i\in[k]$ indexing the fitted center $\bbeta_i$ and its Voronoi set $\vor_i$, we consider two complementary cases.

\subsubsection*{Case 1: there exist some $(s,j,\ell)\in T_i \times [k]\times[k]$ such that $\rho_{s}(\partial_{j,\ell})>\lambda$.}

In this case, following exactly the same argument as in the Ball model proof, we can derive from the inequalities~(\ref{eq:relation-gaussian}) that $\|\bbeta_{i}-\bbeta_{s}^{*}\|\leq\frac{k}{\lambda}+3\rad$. This proves Part~1 of Theorem~\ref{prop:gaussian-family-bounds}.

\subsubsection*{Case 2: for all $(s,j,\ell)\in T_i \times [k]\times[k]$ there holds  $\rho_{s}(\partial_{j,\ell})\leq\lambda$.}

Recall that $m_{i,s}$ is the probability mass of $\vor_{i}$ with respect
to the Gaussian density $f_{s}$ and $\bc_{i,s}$ is the corresponding center of mass of $\vor_{i}$. If we restrict the density $f_s$ onto the ball $\ball_{s}(\rad)$, the values of $m_{i,s}$ and $\bc_{i,s}$ do not change much; in particular, we can control the amount of change using the tail bound~\eqref{eq:concentration-bound}.
With this in mind, we proceed by considering two sub cases.
\begin{itemize}
    \item Case 2(a): $\Wint_i = \emptyset$, in which case $T_i = \Wext_i$. Following the same argument for deriving equation~\eqref{eq:T2-mass} and accounting for the tail probability on $\ball_{s}(\rad)^{\complement}$, we obtain that
    \begin{equation}
    m_{i,s}\leq k\lambda\rad + \mP_{s}\big(\ball_{s}(\rad)^{\complement}\big) \leq k\lambda\rad + \tail(t), \qquad \forall s\in \Wext_i.
    \label{ineq: upper-bound-volume-Gaussian}
    \end{equation}
    It follows that 
    \[
    \mP(\vor_{i})=\frac{1}{k}\sum_{s\in T_i}m_{i,s}\leq k\lambda\rad + \tail(t).
    \]
    This proves Part~2(a) of Theorem~\ref{prop:gaussian-family-bounds}.
    
    \item Case 2(b): $\Wint_i \neq \emptyset$.  
    By Lemma$\ $\ref{lem:necessary}, $\bbeta$ must satisfy the expression
    \[
    \bbeta_{i}=\frac{\sum_{[s]}\bc_{i,s}m_{i,s}}{\sum_{s\in[k]}m_{i,s}}.
    \]
    Using this expression, we have the decomposition $\bbeta_i - \bb_i = \bmu + \bnu$ for some vectors $\bmu$ and $\bnu$ as in equation~\eqref{eq:A and B}. To bound $\bmu$ and $\bnu$, we follow our general strategy to  decompose the Gaussian density $f_s$ into two parts, one supported on the ball $\ball_{s}(\rad)$ and the other the tail, where the tail probability is bounded by $\tail(t)$ as in equation~\eqref{eq:concentration-bound}. 
    By doing so, we can establish analogous versions of the bounds~\eqref{eq:mis_bound} and~\eqref{eq:cis_bound} as given below:
    \begin{align*}\label{eq:mis_bound}
        \begin{cases}
        m_{i,s} \le k \lambda \rad + \tail(t), & \text{if } s\in \Wext_i, \\
        m_{i,s} \ge 1-k^2\lambda\rad - \tail(t), & \text{if } s \in \Wint_i,
        \end{cases}
    \end{align*}
    and
    \begin{align*}
        \|\bbeta_{s}^{*}-\bc_{i,s}\| \le \begin{cases}
        \frac{2\std }{m_{i,s}}, & \text{if }s\in\Wext_i,\\
        \frac{ 2\std \sqrt{k^{2}\lambda\rad + \tail(t)}}{m_{i,s}}, & \text{if } s\in\Wint_i,
        \end{cases}
    \end{align*}
    where the bound on $\|\bbeta_{s}^{*}-\bc_{i,s}\|$ follows  from Lemma \ref{lem:bound-dist-2-center-of-mass-gaussian}. Using the above two bounds, we can further establish analogous versions of Lemmas~\ref{lem:bound-A} and~\ref{lem:bound-B} as given below:
    \begin{align*}
    \| \bmu\| & \leq\frac{2k\sigma}{1-k^{2}\lambda\rad-\tail(t)}+\frac{2k^{2}(k\lambda\rad+\tail(t))\sigma\sqrt{k^{2}\lambda\rad+\tail(t)}}{(1-k^{2}\lambda\rad-\tail(t)))^{2}}
      +\frac{2k(k\lambda\rad+\tail(t))}{1-k^{2}\lambda\rad-\tail(t)}\deltamax,\\
    \| \bnu\| & \leq\frac{2k\sigma\sqrt{k^{2}\lambda\rad+\tail(t)}}{1-k^{2}\lambda\rad-\tail(t)}+\frac{k^{2}\lambda\rad+\tail(t)}{1-k^{2}\lambda\rad-\tail(t)}\deltamax.
    \end{align*}
    It follows that  an analogue of inequality~(\ref{eq:bound-dist-weighted-center}) holds:
    \begin{align*}
    \|\bbeta_{i}-\bb_{i}\| 
    \leq \| \bmu\|+ \| \bnu\| &\leq\frac{4k\sigma}{1-k^{2}\lambda\rad-\tail(t)}+\frac{2k^{2}\sigma\sqrt{k^{2}\lambda\rad+\tail(t)}(k\lambda\rad+\tail(t))}{(1-k^{2}\lambda\rad-\tail(t))^{2}} \\
      &\qquad+\frac{3k^{2}\lambda\rad+3k\tail(t)}{1-k^{2}\lambda\rad-\tail(t)}\deltamax. 
    \end{align*}
    This proves Part~2(b) of Theorem~\ref{prop:gaussian-family-bounds}.
\end{itemize}

\section{Additional Examples\label{sec:Additional Examples}}

In this section, we provide two concrete examples that give additional insights on the behaviors of the local minima of the $k$-means objective and corroborate the results in our main theorems.

Our main theorem assumes certain separation conditions in terms of the SNRs $\snrmin$ and $\snrmax$. The first example shows that if the SNR is too small, then a local minimum may fail to have the structures described in Theorem~\ref{thm:main_ball}. Therefore, a separation condition on the true clusters is in general necessary.

\begin{example}[Small Separation]
\label{ex:example1} Consider the Stochastic Ball Model with $k=3$ in dimension $d=1$, where the ground truth cluster centers are
$\beta_{1}^{*}=-1$, $\beta_{2}^{*}=0$ and $\beta_{3}^{*}=1$, and the radius $\rad$ of the balls  satisfies $(\frac{9\sqrt{2}}{2}-\frac{1}{4})\rad>1$. Let $\bbeta = (\beta_1, \beta_2, \beta_3)$ be a candidate solution with $\beta_{1}=-\frac{2}{3}-\frac{1}{6}\rad$, $\beta_{2}=\frac{2}{3}+\frac{1}{6}\rad$ and $\beta_3>0$ sufficiently large.
\end{example}

When $\beta_{3}$ large, the minimization $\min_{i\in[k]} \|x - \beta_i\|^2$ in the objective $G$ is never attained by $i=3$. In this case, the only effective variables for $G$ are the first two centers $\beta_1$ and $\beta_2$.  The Voronoi boundary $\partial_{1,2}(\bbeta)$ (which is $0$) intersects the second ground truth cluster. Note that these properties continue to hold under small perturbation of $\bbeta$. 
Consequently, for any solution $\bb$ in a small neighborhood 
of $\bbeta$, its objective value has the following expression:
\[
G(\bb)=\frac{1}{6\rad}\left[\int_{-1-\rad}^{-1+\rad}(x-b_{1})^{2}\ddup x+\int_{-\rad}^{\frac{b_{1}+b_{2}}{2}}(x-b_{1})^{2}\ddup x+\int_{\frac{b_{1}+b_{2}}{2}}^{\rad}(x-b_{2})^{2}\ddup x+\int_{1-\rad}^{1+\rad}(x-b_{2})^{2}\ddup x\right].
\]
We compute the gradient and Hessian of $G$ at $\bb$ (recall that only the first two coordinate of $\bb$ are effective):
\begin{align*}
\nabla_{\bb} G & =\frac{1}{6\rad}\left[\begin{array}{c}
-2\int_{-1-\rad}^{-1+\rad}(x-b_{1})\ddup x-2\int_{-\rad}^{\frac{b_{1}+b_{2}}{2}}(x-b_{1})\ddup x\\
-2\int_{\frac{b_{1}+b_{2}}{2}}^{\rad}(x-b_{2})\ddup x-2\int_{1-\rad}^{1+\rad}(x-b_{2})\ddup x
\end{array}\right],\\
\nabla_{\bb}^{2} G & =\frac{1}{6\rad}\left[\begin{array}{cc}
6\rad+(b_{1}+b_{2})-\frac{b_{2}-b_{1}}{2} & -\frac{b_{2}-b_{1}}{2}\\
-\frac{b_{2}-b_{1}}{2} & 6\rad-(b_{1}+b_{2})+\frac{b_{1}-b_{2}}{2}
\end{array}\right].
\end{align*}
Evaluating these expressions at $\beta_{1}=-\frac{2}{3}-\frac{1}{6}\rad$ and $\beta_{2}=\frac{2}{3}+\frac{1}{6}\rad$, we find that 
the gradient vanishes $\nabla_{\bb}G\mid_{\bb=\bbeta}=0$ and the Hessian  is
\[
\nabla_{\bb}^{2}G\mid_{\bb=\bbeta}=\frac{1}{6\rad}\left[\begin{array}{cc}
\frac{35}{6}r-\frac{2}{3} & -\frac{2}{3}-\frac{1}{6}\rad\\
-\frac{2}{3}-\frac{1}{6}\rad & \frac{37}{6}r+\frac{2}{3}
\end{array}\right].
\]
When $(\frac{9\sqrt{2}}{2}-\frac{1}{4}) \rad >1$ or equivalently $\snrmin<\frac{9\sqrt{2}}{2}-\frac{1}{4}$,  the Hessian is positive definite, so $\bbeta$ is a local minimum of $G$. Moreover, one can verify that $G(\bbeta) < G(\bbeta^*)$, so $\bbeta$ is not a global minimum. We  see that the spurious local minimum $\bbeta$ does not have the structures described in Theorem~\ref{thm:main_ball}, as $\bbeta$ involves a $2$-fit-$3$ association.\\

The second example shows that in higher dimensions, there exists a local minimum $\bbeta = (\bbeta_1, \bbeta_2, \bbeta_3)$ such that $\bbeta_1$ \emph{approximately} equals $\bbeta^*_1$, and  $\bbeta_2$ \emph{approximately} equals $(\bbeta^*_2+\bbeta^*_3)/2$ --- a structure guaranteed by Theorem~\ref{thm:main_ball} --- but neither approximation is exact. Therefore, the non-zero approximation errors that appear in Theorem~\ref{thm:main_ball}, is necessary in general.

\begin{example}[Approximation Errors]\label{ex:example2}
Consider the Stochastic Ball model with $k=3$ in dimension $d=2$, where the true cluster centers are $\bbeta_{1}^{*}=(-1,0)$, $\beta_{2}^{*}=(0,0)$
and $\beta_{3}^{*}=(1,0)$, where the radius $r$ of the balls satisfies $r\geq \frac{1}{4}$. Let $\bbeta $ be a candidate solution with  $\bbeta_{1}=(-1,0)$, $\bbeta_{2}=(\frac{1}{2},0)$ and
$\bbeta_{3}$ sufficiently far away from the origin. See 
the left panel of Figure~\ref{fig:3-Ball-clusters} for an illustration.%
\end{example}

\begin{figure}
\centering
\includegraphics[scale=0.3,clip, trim=0 250 0 100]{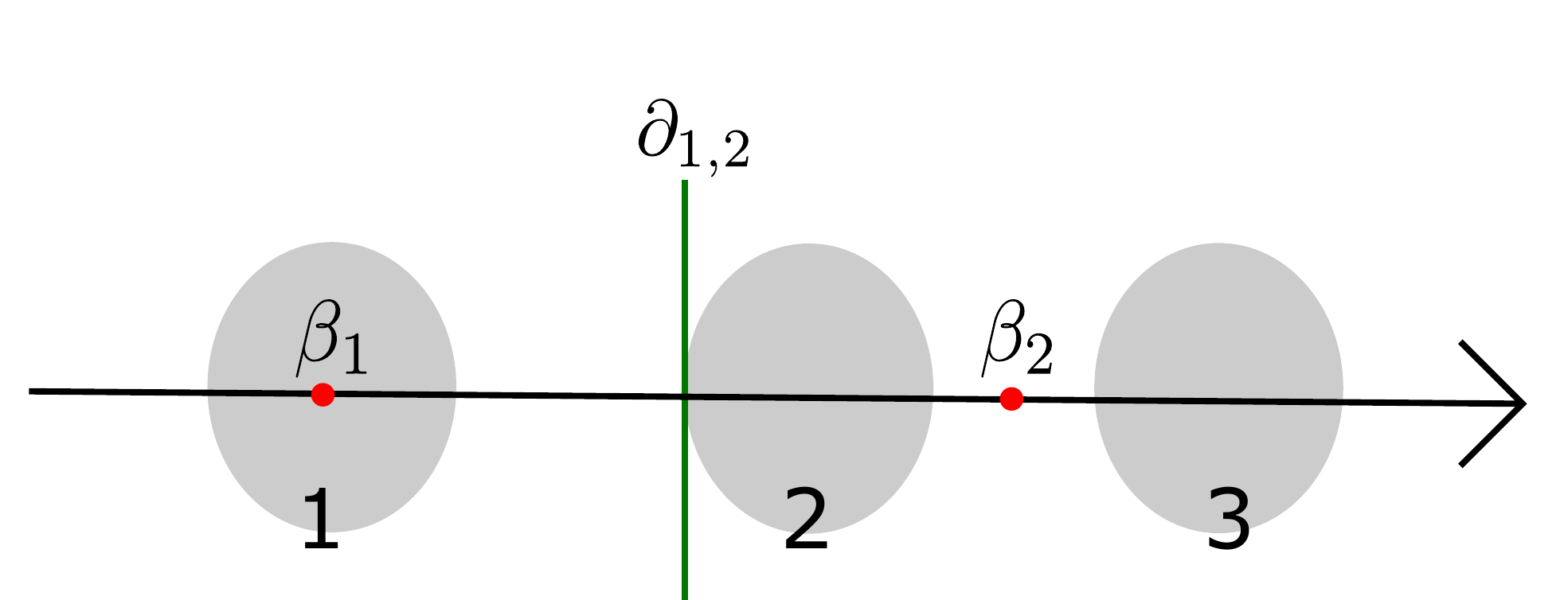}~~~~\includegraphics[scale=0.3, trim=0 250 0 100]{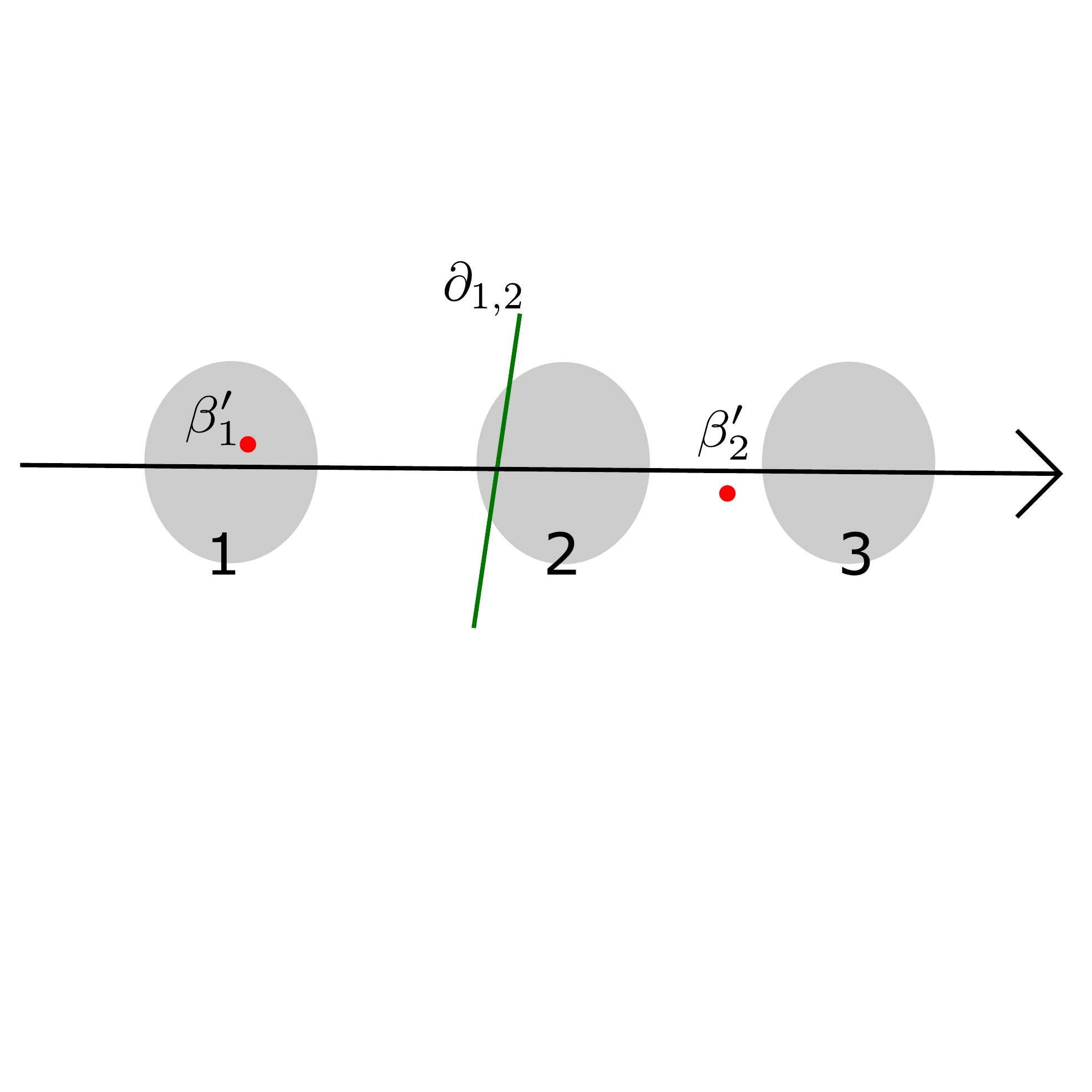}
\caption{\label{fig:3-Ball-clusters}
Two-dimensional Stochastic Ball Model, where the true centers are $ \bbeta_{1}^{*} = (-1,0)$, $ \bbeta_{2}^{*} = (0,0)$ and $ \bbeta_{3}^{*}=(1,0)$. The third fitted center $\bbeta_3$ is far away from the origin, so the only effective variables are $ \bbeta_{1}$ and $\bbeta_{2}$ . Left panel: $ \bbeta_{1}= \bbeta_{1}^{*}$ and $ \bbeta_{2}=\frac{1}{2}( \bbeta_{2}^{*}+ \bbeta_{3}^{*})$. The green line is the Voronoi boundary $\partial_{1,2}(\bbeta)$. Right panel: A perturbed solution $(\bbeta_1',\bbeta_2')$ and the corresponding Voronoi boundary. }
\end{figure}

As in Example~\ref{ex:example1}, here the only effective variables are $\bbeta_1$ and $\bbeta_2$. Assume first that $\rad=\frac{1}{4}$. In this case, the Voronoi boundary $\partial_{1,2}$ is at $x_1=-\frac{1}{4}=\bbeta_2^*-r$, the left boundary of the second true cluster. We claim that $\bbeta$ is a local minimum of $G$; the proof is deferred to the end of this section.
Now, let us increase the radius $\rad$ by a sufficient small amount, in which case the objective function becomes $\widetilde{G}$. By the continuity, there exists a local minimum $\widetilde{\bbeta}$ of $\widetilde{G}$ near the original local minima $\bbeta$. Recall that by Lemma~\ref{lem:necessary}, $\widetilde{\bbeta}_1$ and $\widetilde{\bbeta}_2$ must lie at the center of mass of their Voronoi sets $\vor_1(\widetilde{\bbeta})$ and $\vor_2(\widetilde{\bbeta})$, respectively. It is then not hard to see that the new Voronoi boundary  $\widetilde{\partial}_{1,2}$ corresponding to $\widetilde{\bbeta}$ necessarily intersects the interior of the second true cluster $\ball_2$. It follows that $\vor_1(\widetilde{\bbeta}) = \ball_1 \cup D$  and $\vor_2(\widetilde{\bbeta}) = (\ball_2 \cup \ball_3)\setminus D$  for some subset $D \subset \ball_2$ with a positive measure. Applying Lemma~\ref{lem:necessary} again, we conclude that $\widetilde{\bbeta}_1$ is close but not equal to $\bbeta_1^*$, and that $\widetilde{\bbeta}_2$ is close but not equal to $(\bbeta^*_1+\bbeta^*_2)/2$.

\begin{proof}[Proof of the claim]
Let $t \in (0,1/8)$ be a sufficiently small number, and $\bv_1,\bv_2\in \real^2$ be two arbitrary vectors satisfying $\|\bv_1\|^2+\|\bv_2\|^2=1$. Consider perturbing $\bbeta_1$ and $\bbeta_2$ to  $\bbeta_1'=\bbeta_1+t\bv_1$ and $\bbeta_2' = \bbeta_2+t\bv_2$, respectively. Since Voronoi sets only change by a small amount when the perturbation $t$ is small, we find that $\Delta_{2\to 1}^{\bv}(t)\subseteq \ball_2$ is the only set of points that change their association from one Voronoi set to another; see the right panel of Figure~\ref{fig:3-Ball-clusters}. 
Using the expression~\eqref{eq:H_decompose} for the directional $k$-means objective, we can write  $G(\bbeta')$ as
\begin{align*}
    G(\bbeta')= H^{\bv}(t) 
    = &\frac{1}{3}\left[\int_{\vor_1(\bbeta)\cap \ball_1} \|\x-\bbeta_1-t\bv_1\|^2 \ddup \x + \int_{\vor_2(\bbeta)\cap (\ball_2\cup \ball_3)} \|\x-\bbeta_2-t\bv_2\|^2 \ddup \x\right]\\
    &\qquad +\frac{1}{3}\int_{\Delta_{2\to 1}^{\bv}(t)}\big(\|\x -\bbeta_1-t\bv_1\|^2 -\|\x -\bbeta_2-t\bv_2\|^2\big)\ddup \x. 
\end{align*}
A quick calculation shows that 
\begin{align*}
    G(\bbeta')-G(\bbeta) = H^{\bv}(t) -H^{\bv}(0) =\frac{1}{3}\bigg[t^2 - \underbrace{\int_{\Delta_{2\to 1}^{\bv}(t)}\Big( \|\x -\bbeta_2-t\bv_1\|^2 -\|\x -\bbeta_1-t\bv_2\|^2 \Big)\ddup \x}_{K(t)} \bigg].
\end{align*}
We decompose the term $K(t)$ as follows:
\begin{align*}
    K(t)
    = & \int_{\Delta_{2\to 1}^{\bv}(t)}\underbrace{\Big(\|\x-\bbeta_2\|^2-\|\x-\bbeta_1\|^2\Big)}_{\kappa_1} \ddup \x + \int_{\Delta_{2\to 1}^{\bv}(t)} \underbrace{t^2 \Big(\|\bv_2\|^2-\|\bv_1\|^2\Big)}_{\kappa_2}\ddup \x \nonumber\\
    &\qquad +\int_{\Delta_{2\to 1}^{\bv}(t)} \underbrace{2t \Big( \langle \bv_1-\bv_2,\x\rangle + \langle \bv_2,\bbeta_2\rangle -\langle \bv_1,\bbeta_1\rangle \Big)}_{\kappa_3} \ddup \x 
\end{align*}
For all $\x\in \Delta_{2\to 1}^{\bv}(t) \subseteq \vor_2(\bbeta)$, we have $\|\x-\bbeta_2\|\leq \|\x-\bbeta_1\| $, hence $\kappa_1 \le 0$. We also have $\kappa_2 \le  t^2$ since $\|\bv_2\|^2-\|\bv_1\|^2\leq\|\bv_1\|^2\leq 1$. To bound $\kappa_3$ we observe that $\langle \bv_1-\bv_2,\x\rangle \leq \|\bv_1-\bv_2\|\|\x\|\leq 4\rad = 1$ for all $\x\in \Delta_{2\to 1}^{\bv}(t)\subseteq \ball_2$, $\langle \bbeta_2,\bv_2\rangle \leq \|\bbeta_2\|\|\bv_2\|\leq \frac{1}{2}$, and $\langle \bbeta_1,\bv_1\rangle \leq \|\bbeta_1\|\|\bv_1\|\leq 1$; it follows that $\kappa_3\le 5t.$ Combining pieces, we obtain that 
\begin{align}
    G(\bbeta')-G(\bbeta) \ge \big[t^2 - 6t \cdot \vol(\Delta_{2\to 1}^{\bv}(t)) \big] /3. \label{eq:example2_bound1}
\end{align} 

\begin{figure}[t]
    \centering
    \includegraphics[scale=0.3, clip, trim = 0 150 0 20]{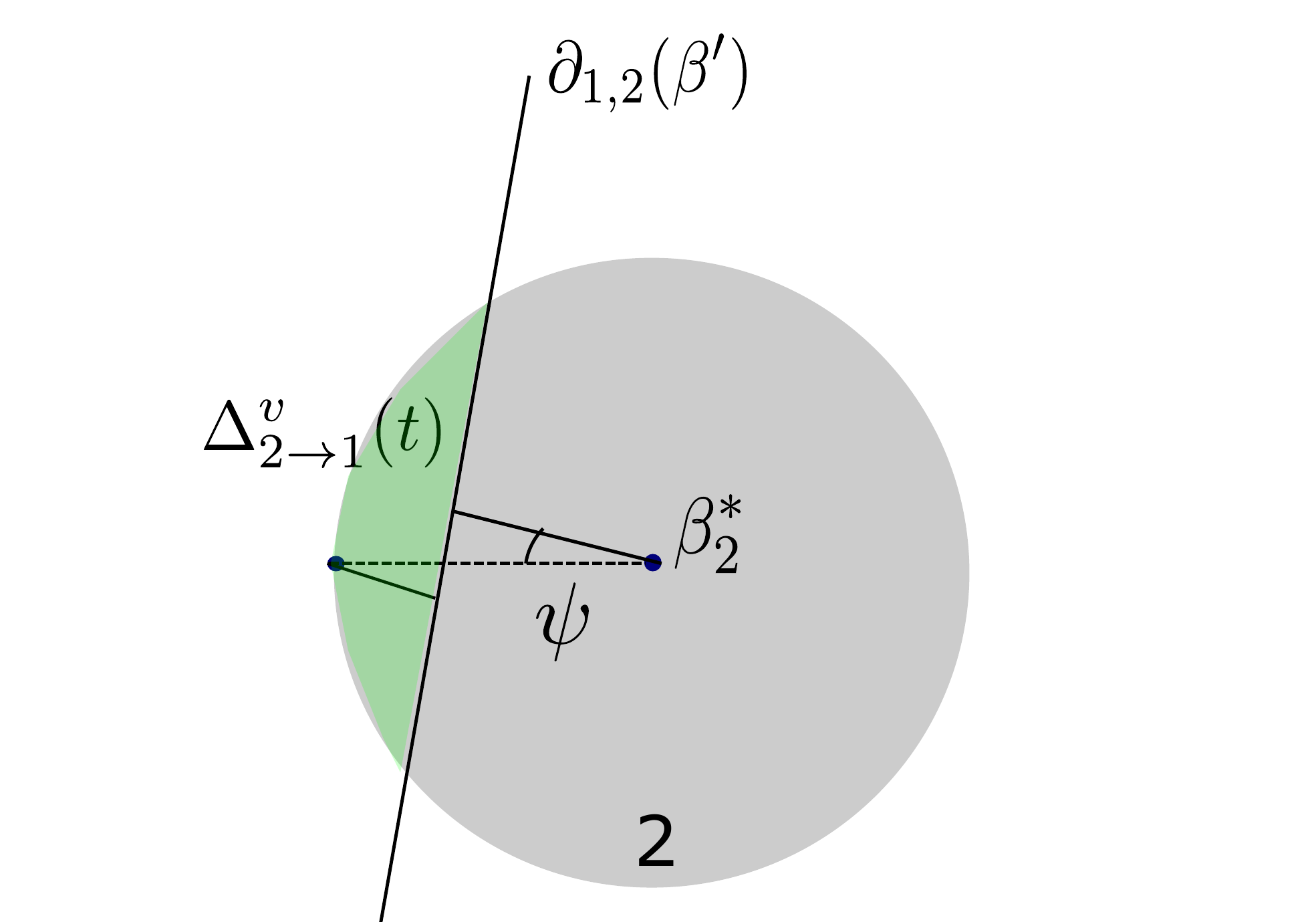}
    \caption{Illustration of $\Delta_{2\to 1}^{\bv}(t)$. The new boundary $\partial_{1,2}(\bbeta')$ has an angle $\psi$ with the original Voronoi boundary $\partial_{1,2}(\bbeta)$.
    }
    \label{fig:maximal-delta-area}
\end{figure}

It remains to control the volume of the set $\Delta_{2\to 1}^{\bv}(t)$, which is illustrated in Figure~\ref{fig:maximal-delta-area}.
Observe that the distance between the old mid point $(\bbeta_1+\bbeta_2)/2$ and the new Voronoi boundary $\partial_{1,2}(\bbeta')$ can be bounded as
\begin{align*}
    d_1 := \dist\left(\frac{\bbeta_1+\bbeta_2}{2}, \partial_{1,2}(\bbeta')\right) \le  \left\| \frac{\bbeta_1+\bbeta_2}{2} - \frac{\bbeta'_1+\bbeta'_2}{2} \right\| \le t.
\end{align*}
Moreover, the (unsigned) angle $\psi$ between the old and new Voronoi boundaries $\partial_{1,2}(\bbeta)$ and $\partial_{1,2}(\bbeta')$ satisfies
\begin{align*}
    \tan\psi = \left|\frac{t (v_{2,2}- v_{1,2})}{2+t(v_{2,1}-v_{1,1})} \right |\leq \frac{t}{1-t} \le 2t.
\end{align*}
From these two observations and the fact that $\rad=1/4$, elementary geometry shows that the distance $d_2$ between $\bbeta^*_2$ and $\partial_{1,2}(\bbeta')$ satisfies
\begin{align*}
    d_2 = \rad \cos \psi - d_1 = \frac{r}{\sqrt{1+\tan^2 \psi}} - d_1 \ge \frac{\rad }{1+2t}-4\rad \ge \rad(1-6t),
\end{align*}
whence
\begin{align*}
    \vol(\Delta_{2\to 1}^{\bv}(t)) \le 2 \cdot \sqrt{\rad^2 - d_2^2} \cdot (\rad-d_2) \le 12t\sqrt{12t}.
\end{align*}
Combining with equation~\eqref{eq:example2_bound1} shows that $G(\bbeta') > G(\bbeta)$ when $t$ is sufficiently small. As this inequality holds for arbitrary perturbation direction $(\bv_1,\bv_2)$, we conclude that $\bbeta$ is a local minimum of $G$. 
\end{proof}

\end{document}